\documentclass{article}
\usepackage{ijcai19}

\usepackage{times}
\usepackage{soul}
\usepackage{url}
\usepackage[hidelinks]{hyperref}
\usepackage[utf8]{inputenc}
\usepackage[small]{caption}
\usepackage{graphicx}
\usepackage{amsmath}
\usepackage{booktabs}
\usepackage{algorithm,algorithmic}
\usepackage{subfigure}
\usepackage{amsfonts}
\usepackage{bm}
\usepackage{amssymb}
\usepackage{bbm}
\usepackage{amsthm}

\usepackage{enumerate}
\usepackage{xcolor}

\newtheorem{theorem}{Theorem}
\newtheorem{lemma}{Lemma}

\newtheorem{remark}{Remark}

\title{AdaLinUCB: Opportunistic Learning for Contextual Bandits 
}

\author{Xueying Guo\and
Xiaoxiao Wang\and
Xin Liu\\
\affiliations
University of California, Davis\\
\emails
guoxueying@outlook.com,
\{xxwa, xinliu\}@ucdavis.edu}

\begin{document}

\maketitle

\begin{abstract}
In this paper, we  propose and study opportunistic contextual bandits - a special case of contextual bandits where the exploration cost varies under different environmental conditions, such as network load or return variation in recommendations. 
When the exploration cost is low, so is the actual regret of pulling a sub-optimal arm (e.g., trying a suboptimal recommendation). Therefore, intuitively, we could explore more when the exploration cost is relatively low and exploit more when the exploration cost is relatively high. Inspired by this intuition, for opportunistic contextual bandits with Linear payoffs, we propose an Adaptive Upper-Confidence-Bound algorithm (AdaLinUCB) to adaptively balance the exploration-exploitation trade-off for opportunistic learning. 
We prove that AdaLinUCB achieves $O\left((\log T)^2\right)$ problem-dependent regret upper bound, which has a smaller coefficient than that of the traditional LinUCB algorithm. Moreover, based on both synthetic and real-world dataset, we show that AdaLinUCB significantly outperforms other contextual bandit algorithms, under large exploration cost fluctuations.
\end{abstract}

\section{Introduction}

In sequential decision making problems such as contextual bandits \cite{Auer2002_2003,Chu2011_LinUCB_analysis,Abbasi2011,Langford2007}, there exists an intrinsic trade-off between exploration (of unknown environment) and exploitation (of current knowledge). Existing algorithm design focuses on how to balance such a trade-off appropriately under the implicit assumption that the exploration cost remains the same over time.
However, in a variety of application scenarios, the exploration cost is time varying  and situation-dependent. Such scenarios present an opportunity to explore more when the exploration cost is relatively low and exploit more when that cost is high, thus adaptively balancing the exploration-exploitation trade-off to reduce the overall regret.
Consider the following motivating examples.

\paragraph{Motivating scenario 1: return variation in recommendations.} 
Contextual bandits have been widely used in recommendation systems \cite{Li2010_LinUCB}. In such scenarios, the candidate articles/products to be recommended are considered as the arms,  the features of users as the context, and   the click-through rate as the reward (i.e., the probability that a user accepts the recommendation). 
However,  note that  the monetary return of a recommendation (if accepted) can differ depending on  1) timing (e.g., holiday vs. non-holiday season) and 2) users with different levels of purchasing power or loyalty (e.g., diamond vs. silver status). Because the ultimate goal is to maximize the overall monetary reward, intuitively, when the monetary return of a  recommendation (if accepted) is low, the monetary regret of pulling a suboptimal arm is low, leading to a low exploration cost, and correspondingly, high returns lead to high regret and high exploration cost. 


\paragraph{Motivating scenario 2: load variation for network configuration.} 
In computer networks, there are a number of parameters that can be configured and have a large impact on overall network performance. For example, in cellular networks, a cell tower can configure transmission power, radio spectrum, antenna, etc., that can affect network performance such as coverage, throughput, and quality of service.
Contextual bandit can be applied in network configuration \cite{Chuai2019_Infocom}.
In such problems, the goal of network configuration can be improving network performance for peak load scenario.
In such a scenario, a possible configuration of a cellular base station can be considered as an arm,   the characteristics of the cell station such as coverage area as the  context,  and network performance such as throughput as reward.
However, network traffic load fluctuates over time, and thus the actual regret of using a suboptimal configuration varies accordingly.

Specifically, when the network load is low, dummy traffic can be injected into the network so that the total load (real plus dummy load) is the same as the peak load. In this manner, we can seek the optimal  configuration under the peak load even in off-peak hours. Meanwhile, the regret of using a suboptimal configuration is low since the real load affected is low.
In practice, the priority of the dummy traffic can be set to be lower than that of the real traffic. Because the network handles high priority traffic first, low priority traffic has little or no impact on the high priority traffic\cite{Walraevens2003}. 
Thus, the regret on the actual load can be further reduced, leading to a low or even negligible exploration cost.

\paragraph{Opportunistic Contextual Bandits.}
Motivated by these application scenarios, we study opportunistic contextual bandits in this paper, focusing on the contextual bandit setting with linear payoffs. Specifically, we define \textit{opportunistic contextual bandit} as a contextual bandit problem with the following characteristic:
1) The exploration cost (regret) of selecting a suboptimal arm varies depending on a time-varying external factor that we called the variation factor. 
2) The variation factor is revealed first so that the learning agent can decide which arm to pull depending on this variation factor.
As suggested by its name, in opportunistic contextual bandits, the variation of this external variation factor can be leveraged  to reduce the actual regret. Further, besides the previous two examples, opportunistic contextual bandit algorithms can be applied to other scenarios that share these characteristics.

We also note that this can be considered as a special case of contextual bandits, by regarding the variation factor as part of context. However, the general contextual bandit algorithms do not take advantage of the opportunistic nature of the problem, and can lead to a less competitive performance.

\paragraph{Contributions.}
In this paper, we propose an Adaptive Upper-Confidence-Bound algorithm for opportunistic contextual bandits with Linear payoffs  (AdaLinUCB). The algorithm is designed to dynamically balance the exploration-exploitation trade-off in opportunistic contextual bandits. To be best of our knowledge, this is the first work to study opportunistic learning for contextual bandits.
We focus on the problem-dependent bound analysis here, which is a setting that allows  a better bound to be achieved under stronger assumptions. 
To the best of our knowledge, such a bound does not exist for LinUCB in the existing literature. In this paper, we prove  problem-dependent bounds for both the proposed AdaLinUCB and the traditional LinUCB algorithms. 
Both algorithms have a regret upper bound of $O\left((\log T)^2\right)$, and the coefficient of the AdaLinUCB bound is  smaller than that of  LinUCB. Furthermore, using both synthetic and real-world large-scale dataset, we show that AdaLinUCB significantly outperforms other contextual bandit algorithms, under large exploration cost fluctuations.

\section{Related Work}
Contextual bandit algorithms have been applied to many real applications, such as display advertising \cite{Li2011_unbiased_LinUCB_simu} and content recommendation \cite{Li2010_LinUCB,Bouneffouf2012}.
In contrast to the classic $K$-arm bandit problem \cite{Auer2002a,Chapelle2011_TS,Agrawal1995}, side information called context is provided in contextual bandit problem before arm selection \cite{Auer2002_2003,Chu2011_LinUCB_analysis,Abbasi2011,Langford2007}. 
The contextual bandits with linear payoffs was first introduced in \cite{Auer2002_2003}. 
In \cite{Li2010_LinUCB}, LinUCB algorithm is introduced based on the ``optimism in the face of Uncertainty" principal for linear bandits. The LinUCB algorithm and its variances are reported to be effective in real application scenarios \cite{Li2010_LinUCB,Wu2016,Wang2016,Wang2017FactorizationBF}. Compared to the classic $K$-armed bandits, the contextual bandits achieves superior performance in various application scenarios \cite{Filippi2010}.

Although LinUCB is effective and widely applied, its analysis is  challenging.
In the initial analysis effort \cite{Chu2011_LinUCB_analysis}, instead of analyzing LinUCB, it  presents an $O(\sqrt{T \ln^3(T)})$ regret bound for a modified version of LinUCB. The modification is needed to satisfy the independent requirement by applying Azuma/Hoeffding inequality. 
In another line of analysis effort, the authors in \cite{Abbasi2011} design another algorithm for  contextual bandits with linear payoffs and provide its regret analysis without independent requirement. Although the algorithm proposed in \cite{Abbasi2011} is different from LinUCB and suffers from a higher computational complexity, the analysis techniques are helpful.

The opportunistic learning has been introduced in \cite{Wu2018_AdaUCB} for classic $K$-armed bandits. 
However, we note that opportunistic learning exists for any sequential decision making problem.
In \cite{Bouneffouf2012}, the authors study into contextual bandits with HLCS (High-Level Critical Situations) set, and proposes a contextual-$\epsilon$-greedy policy, a policy that has an opportunistic nature since the $\epsilon$ (exploration level) is adaptively adjusted based on the similarity to HLCSs (importance level).
However, it only introduces a heuristic algorithm, and does not present a clearly formulation of opportunistic learning. Furthermore, the policy design in  \cite{Bouneffouf2012} implicitly makes the assumption that the contexts in HLCS have already been explored sufficiently beforehand, which is not a cold-start problem.
To the best of our knowledge,
no prior work has made formal mathematical formulation and rigorous performance analysis for opportunistic contextual bandits.

The opportunistic linear contextual bandits can be regarded as a special case of non-linear contextual bandits. However, general contextual bandit algorithms 
such as KernelUCB \cite{valko2013finite} do not take advantage of the opportunistic nature of the problem, and thus can lead to a less competitive performance, as shown in Appendix~\ref{se:ap_simu_kernel} for more details. 
Moreover, KernelUCB suffers from the sensitivity to hyper-parameter tuning, and the extremely high computational complexity for even moderately large dataset, which limits its application in real problems. 

\section{System Model}\label{se:system model}
We use the following notation conventions.
We use $\lVert x \rVert_2$ to denote the 2-norm of a vector $x \in \mathbb{R}^d$. For a positive-definite matrix $A \in \mathbb{R}^{d\times d}$, the weighted 2-norm of vector $x \in \mathbb{R}^d$ is defined by $\lVert x \rVert_A = \sqrt{x^\top A x}$. The inner product of vectors is denoted by $\langle \cdot,\cdot \rangle$, that is, $\langle x,y \rangle =x^\top y$.
Denote by $\lambda_{\min}(A)$ the minimum eigenvalue of a positive-definite matrix $A$. Denote by $\det (A)$ the determinant of matrix $A$. Denote by $\mathrm{trace} (A)$ the trace of matrix $A$.

Now, we present system model. We first introduce the setting of a standard linear contextual bandit problem. The time is slotted. In each time slot $t$, there exists a set of possible arms, denoted by set $\mathcal{D}_t$. 
For each arm $a \in \mathcal{D}_t$, there is an associated context vector $x_{t,a}\in \mathbb{R}^d$, and a nominal reward $r_{t,a}$. 
In each slot $t$, the learner can observe context vectors of all possible arms, and then choose an arm $a_t$ and receive the corresponding nominal reward $r_{t,a_t}$. Note that only the nominal reward of the chosen arm is revealed for the learner in each time slot $t$. 
Further, the nominal rewards of arms are assumed to be a noisy version of an unknown linear function of the context vectors. Specifically, $r_{t,a} = \langle x_{t,a}, \theta_\star \rangle +\eta_t$, where $\theta_\star \in \mathbb{R}^d$ is an unknown parameter, and $\eta_t$ is a random noise with zero mean, i.e., $\mathbb{E}[\eta_t | x_{t,a_t},\mathcal{H}_{t-1}] = 0$, with $\mathcal{H}_{t-1}= (x_{1,a_1}, \eta_1, \cdots, x_{t-1,a_{t-1}}, \eta_{t-1})$ representing  historical observations. 

The goal of a standard contextual bandit problem is to minimize the total regret in $T$ slots, in terms of the nominal rewards. Particularly, the accumulated $T$-slot regret regarding nominal reward is defined as,
\begin{align}
\mathbf{R}_\text{total} (T) = \sum_{t=1}^{T} R_t = \sum_{t=1}^{T} \mathbb{E}[
r_{t,a_t^\star} - r_{t,a_t}],
\end{align}
where $R_t$ is the one-slot regret regarding nominal reward for time slot $t$, $a_t^\star$ is the optimal arm at time slot $t$. 
Here, the optimal arm is the one with the largest expected reward, i.e., 
$a_t^\star = \arg\max_{a\in \mathcal{D}_t} \mathbb{E}[r_{t,a}]$.
To simplify the notation, we denote $r_{t,\star}= r_{t,a_t^\star}$ in the following. That is, $r_{t,\star}$ is the optimal nominal reward at slot $t$.

In the \textbf{opportunistic learning environment}, let $L_t$ be an external variation factor for time slot $t$. The \textbf{actual reward $\tilde{r}_{t,a}$} that the agent receives has the following relationship with the \textbf{nominal reward}:
\begin{align*}
\tilde{r}_{t,a} = L_t r_{t,a}, \forall t, \forall a\in \mathcal{D}_t.
\end{align*}
At each time slot, the learner first observes the context vectors associated with all possible arms, i.e., $x_{t,a}, \forall a\in \mathcal{D}_t$, as well as the current value of $L_t$. 
Based on which the learner selects current arm $a_t$,  observes a nominal reward $r_{t,a_t}$, and receives the actual reward $\tilde{r}_{t,a}= L_t r_{t,a} $. 

This model captures the essence of the opportunistic contextual bandits. For example, in the recommendation scenario, 
and $L_t$ can be a seasonality factor, which captures the general purchase rate in current season.
Or $L_t$ can be purchasing power (based on historical information) or loyalty level of users (e.g., diamond vs. silver status).
In the network configuration example, when the nominal reward $r_{t,a}$ captures the impact of a configuration at the peak load, 
the total load (the dummy load plus the real load) resembles the peak load. 
Then, $L_t$ can be the amount of real load, and thus the actual reward is modulated by $L_t$ as $L_t r_{t,a}$.

The goal of the learner is to minimize the total regret in $T$ slots, in terms of the actual rewards. Particularly, the accumulated $T$-slot regret regarding actual reward is defined as,
\begin{align}\label{eq:total_regret_actual}
\tilde{\mathbf{R}}_\text{total} (T) = \sum_{t=1}^{T} \mathbb{E}[ R_t L_t ]= \sum_{t=1}^{T} \mathbb{E}[
L_t r_{t,\star} - L_t r_{t,a_t}].
\end{align}
In a special case, equation \eqref{eq:total_regret_actual} has an equivalent form: when $L_t$ is i.i.d. over time with mean value $\bar{L}$ and $r_{t,a_t}$ is independent of $L_t$ conditioned on $a_t$, the total regret regarding actual reward is $\tilde{\mathbf{R}}_\text{total} (T) =\bar{L} \sum_{t=1}^{T} \mathbb{E}[ r_{t,\star}] - \sum_{t=1}^{T} \mathbb{E}[ L_t r_{t,a_t}]$. 
Note that in general, it is likely that $\mathbb{E}[ L_t r_{t,a_t}] \neq \bar{L} \mathbb{E}[r_{t,a_t}]$, because the action $a_t$ can depend on $L_t$.

\section{Adaptive LinUCB}
We note that the conventional LinUCB algorithm assumes that the exploration cost factor does not change over time, i.e., $L_t=1$. Therefore, to minimize the  the nominal reward is equivalent to that of the actual reward.
When $L_t$ is time-varying and situation dependent as discussed earlier, we need to maximize the total actual reward, which is affected by the variation factor $L_t$. 
Motivated by this distinction, we  design the adaptive LinUCB algorithm (AdaLinUCB) as in Algo.~\ref{alg:AdaLinUCB}.

\begin{algorithm}[tb]
	\caption{AdaLinUCB}
	\label{alg:AdaLinUCB} 
	\begin{algorithmic}[1]
		\STATE{Inputs:} {$\alpha \in \mathbb{R}_+$, $d \in \mathbb{N}$, $l^{(+)}$, $l^{(-)}$.}
		\STATE{$A \leftarrow \bm{I}_{d}$ \{The $d$-by-$d$ identity matrix\}
		}
		\STATE{$b \leftarrow \bm{0}_d$}
		\FOR{$t = 1, 2, 3,\cdots, T$}
		\STATE{ $\theta_{t-1} = A^{-1} b$}\label{line:AdaLinUCB_theta}
		\STATE{Observe possible arm set $\mathcal{D}_t$, and observe associated context vectors $x_{t,a}, \forall a \in \mathcal{D}_t$.
		}
		\STATE{Observe $L_t$ and calculate $\tilde{L}_t$ by \eqref{eq:L_tilde}.}
		\FOR{$a \in \mathcal{D}_t$}
		\STATE{ $p_{t,a} \leftarrow \theta_{t-1}^\top x_{t,a} + \alpha \sqrt{ (1-\tilde{L}_t)  x_{t,a}^\top A^{-1} x_{t,a} }$ 
		}\label{line:AdaLinUCB_index}
		\ENDFOR
		\STATE{Choose action $a_t = \arg\max_{a\in \mathcal{D}_t} p_{t,a}$ with ties broken arbitrarily.}\label{line:AdaLinUCB_selection}
		\STATE{Observe nominal reward $r_{t,a_t}$.}
		\STATE{$A \leftarrow A + x_{t,a_t} x_{t,a_t}^\top $}\label{line:AdaLinUCB_A_update}
		\STATE{ $b \leftarrow b + x_{t, a_t} r_{t,a_t} $
		}\label{line:AdaLinUCB_b_update}
		\ENDFOR
	\end{algorithmic}
\end{algorithm} 

In Algo.~\ref{alg:AdaLinUCB}, $\alpha$ is a hyper-parameter, which is an input of the algorithm, and 
$\tilde{L}_t$ is the normalized variation factor, defined as, 
\begin{align}\label{eq:L_tilde}
\tilde{L}_t = 
\left({[L_t]_{l^{(-)}}^{l^{(+)}} -l^{(-)} }\right)
/
\left({l^{(+)}-l^{(-)}} \right),
\end{align}
where $l^{(-)}$ and $l^{(+)}$ are the lower and upper thresholds for truncating the variation factor, and
$
    [L_t]_{l^{(-)}}^{l^{(+)}}  =\max\{l^{(-)}, \min\{L_t, l^{(+)} \}
    \}
    $.
 That is, $\tilde{L}_t$ normalizes $L_t$ into $[0,1]$ to capture different ranges of $L_t$. 
 To achieve good performance, the truncation thresholds should be appropriately chosen to achieve sufficient exploration. Empirical results show that a wide range of threshold values can lead to good performance of AdaLinUCB.
 Furthermore, these thresholds can  be learned online in practice without prior knowledge on the distribution of $L_t$, as discussed in Sec.~\ref{se:numerical_results} and Appendix~\ref{se:ap_simulations}. 
 Note that $\tilde{L}_t$ is only used in AdaLinUCB algorithm. The actual rewards and regrets are based on $L_t$, not $\tilde{L}_t$.
 
 In Algo.~\ref{alg:AdaLinUCB}, for each time slot, the algorithm updates a matrix $A$ and a vector $b$. 
 The $A$ is updated in step~\ref{line:AdaLinUCB_A_update}, which is denoted as $A_t=I_d + \sum_{\tau=1}^{t}x_{\tau, a_\tau} x_{\tau, a_\tau}^\top$ in the following analysis.
Note that $A_t$ is a positive-definite matrix for any $t$, and that $A_0=I_d$.
The $b$ is updated in step~\ref{line:AdaLinUCB_b_update}, which is denoted as $b_t = \sum_{\tau=1}^{t} 
x_{\tau, a_\tau} r_{\tau,a_\tau}$ in the following analysis.
Then, we have $\theta_t  = A_t^{-1} b_t $ (see step~\ref{line:AdaLinUCB_theta}), which is the estimation of the unknown parameter $\theta_\star$ based on historical observations.
Specifically, 
$\theta_t$ is the result of a ridge regression for estimating $\theta_\star$, which minimizes a penalized residual sum of squares, i.e.,
$\theta_t = \arg\min_\theta \left\{
\sum_{\tau=1}^{t} \left(
r_{\tau,a_\tau} - \langle \theta, x_{\tau,a_\tau} \rangle
\right)^2
+ \lVert \theta \rVert_2^2
\right\}$.

In general, the AdaLinUCB algorithm explores more when the variation factor is relatively low, and exploits more when the variation factor is relatively high. 
To see this, note that the first term of the index $p_{t,a}$ in step~\ref{line:AdaLinUCB_index},
i.e., $\theta_{t\!-\!1}^\top x_{t,a}$, is the estimation of the corresponding reward; while the second part is an adaptive upper confidence bound modulated by $\tilde{L}_t$, which determines the level of exploration.
At one example, when $L_t$ is at its lowest level with $L_t \leq l^{(-)}$,  $\tilde{L}_t=0$, and the index $p_{t,a}$ 
is the same as that of the LinUCB algorithm, and then the  algorithm selects arm in the same way as the conventional LinUCB.
At the other extreme, when $\tilde{L}_t = 1$, i.e., $L_t \geq l^{(+)}$, the index  $p_{t,a}=\theta_{t-1}^\top x_{t,a}$, which is the estimation of the corresponding reward. 
That is, when the variation factor is at its highest level, the AdaLinUCB algorithm purely exploits the existing knowledge and selects the current best arm.

\section{Performance Analysis}\label{se:Performance Analysis}

We first summarize the technical assumptions needed for performance analysis:
i. Noise satisfies $C_\text{noise}$-sub-Gaussian condition, as explained later in \eqref{as:sub_Gaussian}; ii. The unknown parameter $\theta_\star$ satisfies $||\theta_\star||_2 \leq C_\text{theta}$; iii. For $\forall t, \forall a\in \mathcal{D}_t$, $\lVert x_{t,a} \rVert_2 \leq C_\text{context} $ holds; iv. $\lambda_{\min}(I_d)\geq \max\{1, C_\text{context}^2 \}$; v. the nominal reward $r_{t,a_t}$ is independent of the variation factor $L_t$, conditioned on $a_t$.

We note that assumptions i.-iv. are widely used in contextual bandit analysis \cite{Auer2002_2003,Chu2011_LinUCB_analysis,Abbasi2011,Wu2016,Wang2016}.

Specifically, the sub-Gaussian condition in assumption i. is a constraint on the tail property of the noise distribution, as that in \cite{Abbasi2011}.
That is, for the noise $\eta_t$, we assume that,
\begin{align}\label{as:sub_Gaussian}
\forall \zeta \in \mathbb{R}, \quad \mathbb{E}[e^{\zeta \eta_t} | 
x_{t,a_t}, \mathcal{H}_{t-1}
] 
\leq \exp\left(\frac{\zeta^2 C_\text{noise}^2}{2} \right),
\end{align}
with $\mathcal{H}_{t-1}= (x_{1,a_1}, \eta_1, \cdots, x_{t-1,a_{t-1}}, \eta_{t-1})$ and $C_\text{noise}>0$. 
Note that the 
sub-Gaussian condition requires both  \eqref{as:sub_Gaussian} and $\mathbb{E}[\eta_t | x_{t,a_t}, \mathcal{H}_{t-1}] = 0$.
Further, this condition indicates that $\mathrm{Var}[\eta_t | F_{t-1}]\leq C_\text{noise}^2 $, where $\{F_t\}_{t=0}^\infty$ is the filtration of $\sigma$-algebras for selected context vectors and noises, i.e., $F_t = \sigma( x_{1,a_1}, x_{2,a_2}, \cdots, x_{t+1,a_{t+1}}, \eta_1, \eta_2, \cdots, \eta_t )$. Thus, $C_\text{noise}^2$ can be viewed as the (conditional) variance of the noise.

Examples for the distributions that satisfies the sub-Gaussian condition are:
1) A  zero-mean Gaussian noise with variance at most  $C_\text{noise}^2$; 2) A bounded noise with zero-mean and lying in an interval of length at most $2 C_\text{noise}$.



Assumption iv.~can be relaxed by changing the value of $A_0$ in Algo.~\ref{alg:AdaLinUCB} from the current identity matrix $I_d$ to a positive-definite matrix with a higher minimum eigenvalue (see Appendix~\ref{se:LemmaProofs} for more details).

Assumption v. is valid in many application scenarios. 
For example, 
in the network configuration scenario, since the total load resembles the peak load, the network performance, i.e., the nominal reward $r_{t,a_t}$, is independent of the real load $L_t$, conditioned on configuration $a_t$.
Also, in the recommendation scenario, the click-through rate (i.e., reward $r_{t,a}$) can be independent of the user influence (i.e., variation factor $L_t$).



\subsection{Problem-Dependent Bounds}\label{se:problem_dependent_setting}

We focus on problem-dependent performance analysis here because it can lead to a tighter bound albeit under stronger   assumptions.
To derive the problem-dependent bound, we assume that there are a finite number of possible context values, and denote this number as $N$. 
Then, let $\Delta_{\min}$ denote the minimum nominal reward difference between the best and the ``second best'' arms. That is,
$
\Delta_{\min} =
\min_t \left\{ 
r_{t,\star} - \max_{a\in \mathcal{D}_t , r_{t,a}\neq r_{t,\star} } r_{t,a} \right\} $. 
Similarly, let $\Delta_{\max}$ denote the maximum nominal reward difference between arms. That is, 
$
\Delta_{\max} = \max_t \left\{ 
r_{t,\star} - \min_{a\in \mathcal{D}_t  } r_{t,a} \right\} $. 

As in existing literature for problem-dependent analysis of linear bandits\cite{Abbasi2011}, we assume that single optimal context condition holds here. Specifically, for different time slot $t=1,2,\cdots$, there is a single optimal context value. That is, there exists $x_\star \in \mathbb{R}^d$, such that,
$
x_\star = x_{t,a_t^\star}, \forall t$.

\subsection{AdaLinUCB under  Binary-Valued Variation}

We first introduce the result under a random binary-valued variation factor. We assume that the variation factor $L_t$ is i.i.d. over time, with $L_t \in\{\epsilon_0, 1- \epsilon_1 \}$, where $\epsilon_0, \epsilon_1\geq 0$ and $\epsilon_0<1-\epsilon_1$. Let $\rho$ denote the probability that the  variation factor is low, i.e., $\mathbb{P}\{L_t = \epsilon_0 \}=\rho$.

Firstly, we note that, for a $\tilde{\delta} \in (0,1) $, there exists a positive integer $C_\text{slots}$ such that,  
\begin{align}\label{eq:Cslots_definition}
\nonumber
& \forall t\geq C_\text{slots},~
\rho t - \sqrt{\frac{t}{2} \log \frac{\tilde{\delta}}{2} }
- \frac{16 C_\text{noise}^2 C_\text{theta}^2 }{\Delta_{\min}^2} \bigg[\log(C_\text{context} t) \\ \nonumber
& + 2 (d-1) \log \Big(
d\log\frac{d+ t C_\text{context}^2}{d} + 2 \log\frac{2}{\tilde{\delta}}
\Big) + 2 \log\frac{2}{\tilde{\delta}}
\\ 
& +(d-1) \log\frac{64 C_\text{noise}^2 C_\text{theta}^2 C_\text{context} }{\Delta_{\min}^2} 
\bigg]^2
\geq 
\frac{4 d}{\Delta_{\min}^2}.
\end{align}
To see such an integer $C_\text{slots}$  exists, note that for large enough $t$,  in the left-hand side of the inequality \eqref{eq:Cslots_definition}, the dominant positive term is $O(t)$ while the dominant negative term is $O(\sqrt{t})$.

To interpret $C_\text{slots}$,
it is an integer that is large enough so that during $C_\text{slots}$-slot period, 
enough exploration is done in the time slots when variation factor is relatively low, such that to have a relatively tight bound for the estimation of the optimal reward.

Then, we have the following results. 

\begin{theorem}\label{th:AdaLinUCB_problem_dependent}
Consider the opportunistic contextual bandits with linear payoffs and binary-valued variation factor. 
With probability at least $1-\tilde{\delta}$, the accumulated regret (regarding actual reward) of AdaLinUCB algorithm satisfies,
\begin{align*}
\tilde{\mathbf{R}}_\mathrm{total}(T)
& \leq \epsilon_0 \cdot \frac{16 C_\mathrm{noise}^2 C_\mathrm{theta}^2 }{\Delta_{\min}}
\bigg[
\log(C_\mathrm{context} T)
+ 2 \log\frac{2}{\tilde{\delta}} \\
& + 2 (d-1) \log \Big(
d\log\frac{d+ T C_\mathrm{context}^2}{d} + 2 \log\frac{2}{\tilde{\delta}}
\Big) \\
& +(d-1) \log\frac{64 C_\mathrm{noise}^2 C_\mathrm{theta}^2 C_\mathrm{context} }{\Delta_{\min}^2} 
\bigg]^2\\
&
+ (1-\epsilon_1)   \Big[
\big(
\Delta_{\max} C_\mathrm{slots} 
+
4  d \frac{N-1}{\Delta_{\min}} \big) \\
&~~~ \cdot
\Big(
C_\mathrm{noise}\sqrt{d\log
	\frac{2+2TC_\mathrm{context}^2}{\tilde{\delta}} } + C_\mathrm{theta}
\Big)^2 \bigg]
,
\end{align*}
where $C_\mathrm{slots}$ is a constant satisfying \eqref{eq:Cslots_definition}.
\end{theorem}

\textit{Proof Sketch:} Although the proof for Theorem~\ref{th:AdaLinUCB_problem_dependent}
 is complicated, the key  is to treat the slots with low variation factor and the slots with high variation factor separately. 
For slots with low variation factor,  the one-step regret is upper bounded by the weighted $2$-norm of the selected context vectors, i.e., $R_t \mathbbm{1} \{L_t = \epsilon_0 \}\leq 2 \alpha \lVert x_{t,a_t} \rVert_{A_{t-1}^{-1}}$, and then the accumulated regret can be analyzed accordingly.
For the slots with high variation factor, by matrix analysis, we can show that when a particular context value has been selected enough times, its estimated reward is accurate enough in an appropriate sense. 
Further, it can benefit from regret bound for low variation factor slots that the optimal context has been selected enough time with high probability. Then, we combine these to prove the result. 
More details are shown in Appendix~\ref{se:Proof_AdaLinUCB}.

\begin{remark}\label{rm:AdaLinUCB_binary}
For the regret bound in Theorem~\ref{th:AdaLinUCB_problem_dependent}, the first three lines cover the accumulated regret that is incurred during time slots when the variation factor is relatively low, i.e., during slots $t$ with $L_t=\epsilon_0$, while the last two lines cover the accumulated regret that is incurred during time slots when the variation factor is relatively high, i.e., during slots $t$ with $L_t= 1- \epsilon_1$.
Further, when $T$ is large enough, the dominant term for the first three lines is $O\left((\log T)^2\right)$, while the dominant term for the last two lines is $O\left(\log T\right)$. 
That is, the bound for the accumulated regret during slots when the variation factor is relatively high actually increases slower than the bound for the accumulated regret during slots when the variation factor is relatively low.
This is in consistent with the motivation of AdaLinUCB design: explore more when the variation factor is relatively low, and exploit more when the variation factor is relatively high.

Furthermore, beside parameter $T$, which is the time horizon, the regret bound in Theorem~\ref{th:AdaLinUCB_problem_dependent} is also affected by problem-dependent parameters: 
it is affected by $N$, which is the number of possible context values, $\Delta_{\min}$, which is the minimum nominal reward difference between the best and the ``second best" arms, and $\Delta_{\max}$, which is the maximum nominal reward difference between arms.
In general, a larger number of possible context values, i.e., a larger $N$, may lead to a larger $\Delta_{\max}$ and a smaller $\Delta_{\min}$, and in this way, results in a larger regret bound.
\end{remark}


\subsection{AdaLinUCB under  Continuous Variation}
We now study AdaLinUCB in opportunistic contextual bandits under continuous variation factor. Under continuous variation factor, it is difficult to obtain regret bound for general values of $l^{(-)}$ and $l^{(+)}$ because exploration and exploitation mix in a complex fashion when $l^{(-)}\!<\! L_t \!<\!l^{(+)}$. Instead, inspired by the insights obtained from the binary-valued variation factor case, we illustrate the advantages of AdaLinUCB for special case with $l^{(-)}=l^{(+)}$.


In the special case of $l^{(-)}=l^{(+)}$, the normalized variation factor $\tilde{L}_t$ in \eqref{eq:L_tilde} is redefined as $\tilde{L}_t=0$ when $L_t \leq l^{(-)}$ and as $\tilde{L}_t=1$ when $L_t > l^{(+)}=l^{(-)}$.

\begin{theorem}\label{th:continuous_AdaLinUCB_problem_dependent}
In the opportunistic contextual bandits with linear payoffs and continuous variation factor that is i.i.d. over time, under AdaLinUCB with  $\mathbb{P}\{ L_t \leq l^{(-)} \}= \rho >0$ and $l^{(-)}=l^{(+)}$, with probability at least $1-\tilde{\delta}$, the accumulated regret (regarding actual reward) 
satisfies,
\begin{align*}
\tilde{\mathbf{R}}_\mathrm{total}(T)
& \leq \!
\mathbb{E}\big[\!L_t \!| L_t \! \leq \! l^{(-)} \!
\big]\!
\frac{16 C_\mathrm{noise}^2\! C_\mathrm{theta}^2 }{\Delta_{\min}}\!
\bigg[\! \log(C_\mathrm{context} T) \\
& + 2 (d-1) \log \Big(
d\log\frac{d+ T C_\mathrm{context}^2}{d} + 2 \log\frac{2}{\tilde{\delta}}
\Big) \\
& +\!(d\!-\!1) \log\frac{64 C_\mathrm{noise}^2 C_\mathrm{theta}^2 C_\mathrm{context} }{\Delta_{\min}^2} \!+\! 2 \log\frac{2}{\tilde{\delta}}
\bigg]^2 \\
& + \mathbb{E}[L_t | L_t > l^{(-)}]
\cdot \bigg[
\Big(
\Delta_{\max} C_\mathrm{slots} 
+
4  d \frac{N-1}{\Delta_{\min}} \Big) \\
& \cdot
\Big(
C_\mathrm{noise}\sqrt{d\log
	\frac{2+2TC_\mathrm{context}^2}{\tilde{\delta}} } + C_\mathrm{theta}
\Big)^2 \bigg]
,
\end{align*}
where $C_\mathrm{slots}$ is a constant satisfying \eqref{eq:Cslots_definition}.
\end{theorem}
\begin{proof}
Recall that for the special case with $l^{(+)}=l^{(-)}$, we have  $\tilde{L}_t=0$ when $L_t \leq l^{(-)}$ and as $\tilde{L}_t=1$ when $L_t > l^{(+)}$. 
Thus, this theorem can be proved analogically to the proof of Theorem~\ref{th:AdaLinUCB_problem_dependent}, by noting the following: When $L_t \leq l^{(-)}$, we have $\tilde{L}_t=0$ which corresponds to the case of $L_t = \epsilon_0$ ($\tilde{L}_t=0$) in the binary-valued variation factor case; while when  $L_t > l^{(+)}$ ($\tilde{L}_t=1$) corresponds to the case of $L_t=1-\epsilon_1$ under binary-valued variation factor case.
The conclusion of the theorem then follows by using the fact that all variation factor below $l^{(-)}$ are treated same by AdaLinUCB, i.e.,  $\tilde{L}_t=0$ for $L_t \leq l^{(-)}$; while all variation factor above $l^{(-)}$ are treated same by AdaLinUCB, i.e.,  $\tilde{L}_t=1$ for $L_t \leq l^{(+)}$.
\end{proof}

\begin{remark}
Similar to Remark~\ref{rm:AdaLinUCB_binary} for
Theorem~\ref{th:AdaLinUCB_problem_dependent}, the regret bound in Theorem~\ref{th:continuous_AdaLinUCB_problem_dependent} can be divided into two parts:
the first three lines cover the accumulated regret that is incurred during time slots when $L_t\leq l^
{(-)}$  and is $O\left((\log T)^2\right)$, while the last two lines cover the accumulated regret for time slots when $L_t> l^
{(-)}$ and is $O\left((\log T)\right)$.
Furthermore, a larger $N$, i.e., a larger number of possible context values, can lead to a larger regret bound.
\end{remark}

\subsection{Regret Bound of  LinUCB}\label{se:LinUCB_in_mainPpaer}
To the best of our knowledge, there exists no problem-dependent bound on LinUCB. (The initial analysis of LinUCB presents a more general and looser performance bound for a modified version of LinUCB. The modification is needed to satisfy the independent requirement by applying Azuma/Hoeffding inequality \cite{Chu2011_LinUCB_analysis}.) Furthermore,
we note that one can directly apply LinUCB to opportunistic contextual bandits using the linear relationship $
\mathbb{E}[r_{t,a}|x_{t,a}] = \langle x_{t,a}, \theta_\star \rangle$, which is called LinUCBExtracted in numerical results. Therefore, we derive the regret upper bound for LinUCB here, both as an individual contribution as well as for comparison purpose.


\begin{theorem}\label{th:LinUCB_problem_dependent_bound}
In the opportunistic contextual bandits with linear payoffs and continuous variation factor that is i.i.d. over time with mean $\bar{L}$, with probability at least $1-\delta$, the accumulated $T$-slot regret (regarding actual reward) of LinUCB satisfies,
\begin{align*}
\tilde{\mathbf{R}}_\mathrm{total}(T)
& \leq \frac{16 \bar{L} C_\mathrm{noise}^2 
	C_\mathrm{theta}^2}{\Delta_{\min}} 
\bigg[
\log(C_\mathrm{context} T )
+ 2\log \frac{1}{\delta} \\
& + 2(d-1) \log \Big(
d \log\frac{d+TC_\mathrm{context}^2}{d}
+2\log\frac{1}{\delta}
\Big) \\
&
+ (d-1) \log \frac{64 C_\mathrm{noise}^2 
	C_\mathrm{theta}^2 C_\mathrm{context}}{\Delta_{\min}^2}
\bigg]^2 .
\end{align*}
\end{theorem}

\begin{figure*}[bt]
\begin{center}
\begin{minipage}{0.66\textwidth}
\begin{center}
\captionsetup{justification=centering,margin=0.2cm}
\subfigure[Binary-valued variation factor]{\includegraphics[angle = 0,width = 0.47\linewidth]{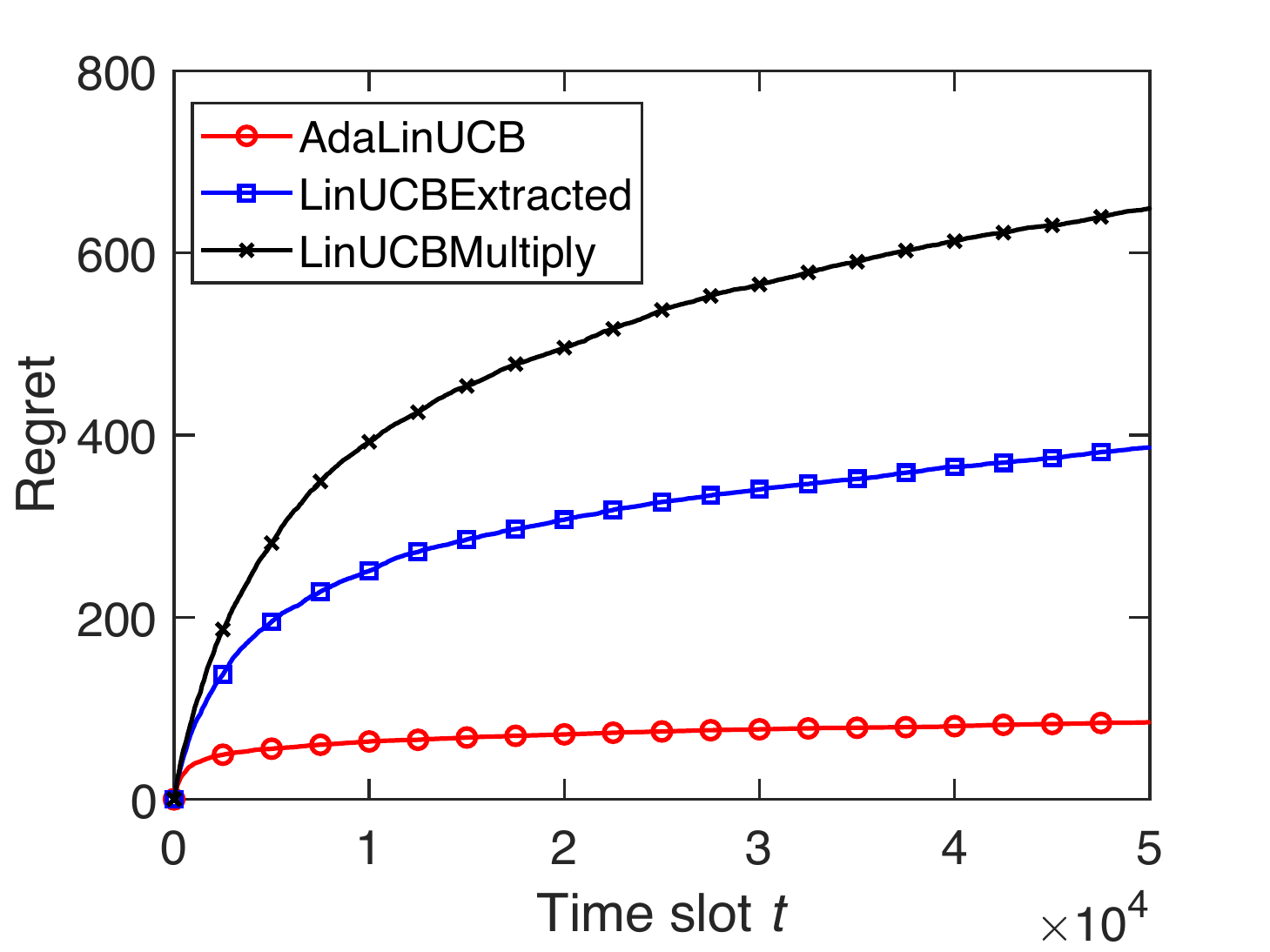}
\label{fig:BinaryValue_L}}
\subfigure[Beta distributed variation factor]
{\includegraphics[angle = 0,width = 0.47\linewidth]{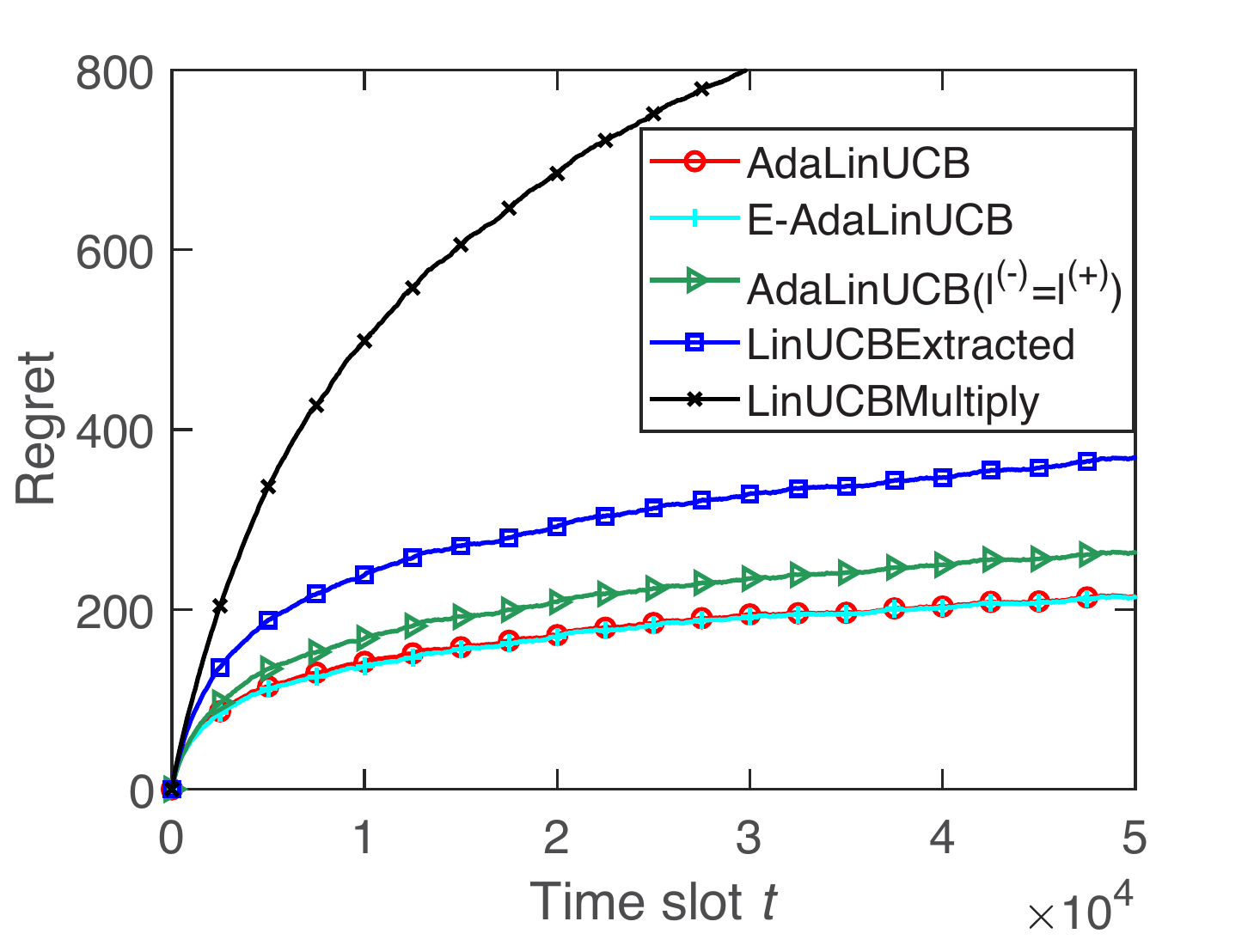}
\label{fig:beta_L}}
\vspace{-.25cm}
\caption{Regret under Synthetic Scenarios. In (a), $\epsilon_0 = \epsilon_1 = 0,  \rho = 0.5$.  In (b),  AdaLinUCB: $l^{(-)} \!\! =\! l^{(-)}_0$,  $l^{(+)}\!\! = l^{(+)}_0$; AdaLinUCB: $(l^{(-)} \! \!=l^{(+)})$, $l^{(-)}\!\!=l^{(+)} \!\!=l^{(-)}_{0.5}$}
\label{fig:synthetic_scenarios}
\end{center}
\end{minipage}
\begin{minipage}{0.33\textwidth}
\begin{center}
\captionsetup{justification=centering}
{\includegraphics[angle = 0,width = 0.94\linewidth]{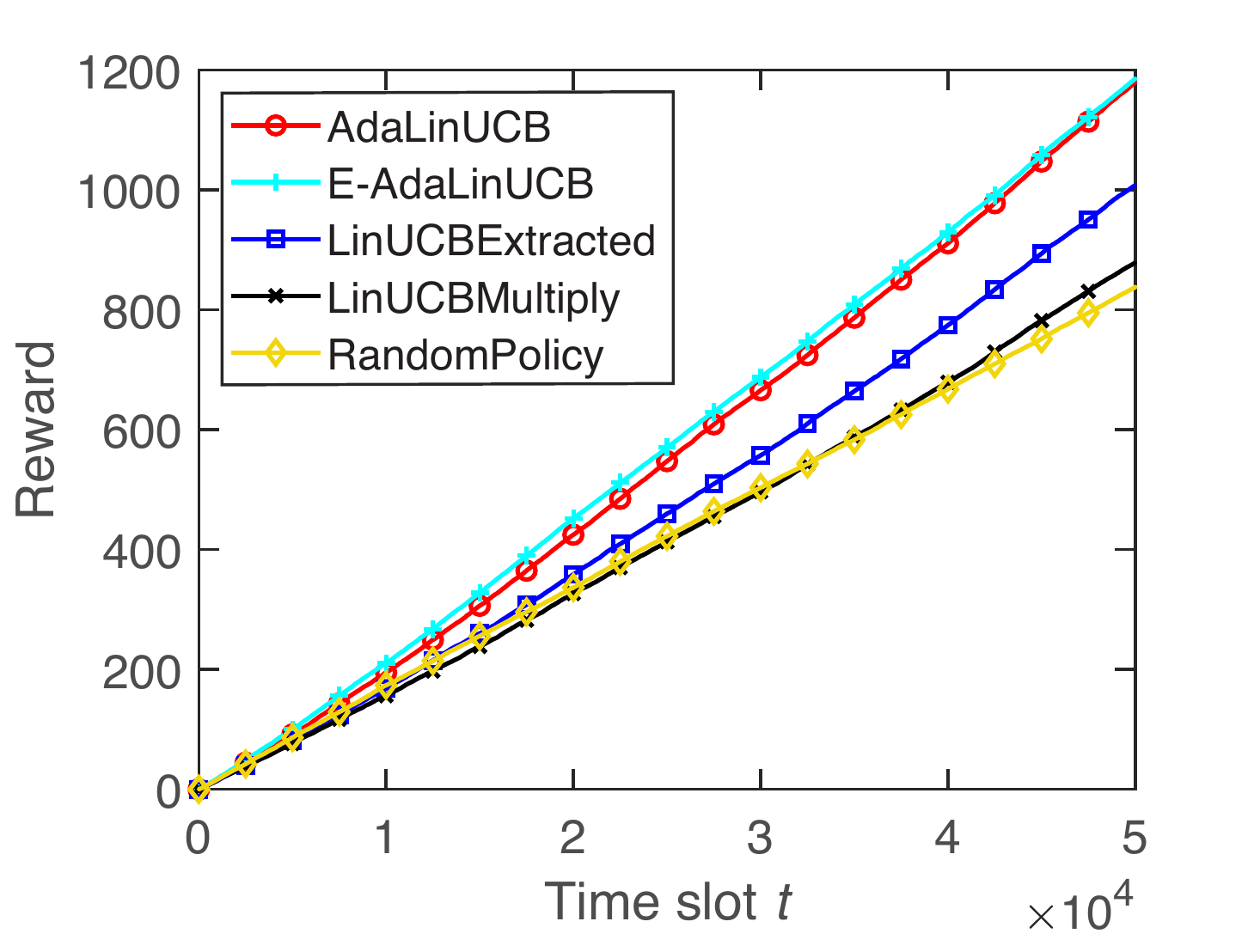}
\vspace{0.25cm}
\caption{Rewards for Yahoo! Today Module 
$l^{(-)}=l_0^{(-)}$,     $l^{(+)}=l_{0.3}^{(+)}$}
\vspace{0.25cm}
\label{fig:Yahoo_Module}}
\end{center}
\end{minipage}
\end{center}
\vspace{-0.45cm}
\end{figure*}

The regret bound for LinUCB under non-opportunistic case can be shown by simply having $\bar{L}=1$ in the above result. 
Here, note that problem-dependent bound analysis is a setting that allows a better bound to be achieved with stronger assumptions. Recall that the assumptions are discussed in Sec.~\ref{se:problem_dependent_setting}. As a result, the problem-dependent bound of LinUCB is much better than its general bound, such as the bound for a modified version of LinUCB in  \cite{Chu2011_LinUCB_analysis}.
More results for LinUCB and the proof of Theorem~\ref{th:LinUCB_problem_dependent_bound} can be found in Appendix~\ref{se:Appendix_LinUCB}.

\begin{remark}
Theorem~\ref{th:LinUCB_problem_dependent_bound} and Theorem~\ref{th:AdaLinUCB_problem_dependent} show  that the problem-dependent regret bounds (regarding actual reward) for LinUCB and AdaLinUCB are both $O\left( (\log T)^2\right)$.
Further, for binary-valued variation factor, the asymptotically dominant term for the bound of LinUCB is
$
\frac{1-\epsilon_1 +\epsilon_0}{2}  \cdot
\frac{16  C_\mathrm{noise}^2 
	C_\mathrm{theta}^2}{\Delta_{\min}} \left(\log T\right)^2$.
In comparison, for AdaLinUCB, it is
$
\epsilon_0  \cdot
\frac{16  C_\mathrm{noise}^2 
	C_\mathrm{theta}^2}{\Delta_{\min}} \left(\log T\right)^2$.
Because $\epsilon_0 < 1- \epsilon_1$, in the scenario of binary-valued variation factor, the AdaLinUCB algorithm has a better asymptotic problem-dependent upper bound than that of the LinUCB algorithm.
Similarly, in scenario with continuous variation factor, 
the AdaLinUCB algorithm with $l^{(+)}=l^{(-)}$ has a better problem-dependent bound than LinUCB algorithm as long as $\mathbb{E}[L_t | L_t \leq l^{(-)}]< \bar{L}$, which holds in most cases.
\end{remark}

\subsection{Discussions on the Disjoint Model}
The seminal paper on LinUCB  \cite{Li2010_LinUCB} introduces different models for 
contextual bandits.
The opportunistic learning applies to these different models.
One of them is the joint model discussed above. 
Another model is the disjoint model, which assumes that,
$
\mathbb{E}[r_{t,a}|x_{t,a}] = \langle x_{t,a}, \theta_\star^{(a)} \rangle
$,
where $x_{t,a}$ is a  context vector and $\theta_\star^{(a)}$ is the unknown coefficient vector for arm $a$. This model is called disjoint since the parameters are not shared among different arms.
There is also a hybrid model that combines the joint model and the disjoint model.

In this paper, we focus on the design and analysis of opportunistic contextual bandits using the joint model.
However, it should be noted that, the AdaLinUCB algorithm in Algo.~\ref{alg:AdaLinUCB}
can be modified slightly and applied to the disjoint model, see Appendix~\ref{se:Appendix_Disjoint} for more details.
Also, the analysis of the joint model can be extended to the disjoint one. Note that the disjoint model can be converted to a joint model when the number of possible arms is finite.
Specifically, for an arbitrary disjoint-model opportunistic contextual bandit problem  with $\theta_\star^{(a)}~\forall a$, an equivalent joint-model problem exists with the joint unknown parameter as  $\theta_\star=([\theta_\star^{(1)}]^\top, [\theta_\star^{(2)}]^\top, \cdots)^\top$ and the context vectors modified accordingly.
Thus, the previous analytical results are valued for the disjoint model with appropriate modifications.

\section{Numerical Results}\label{se:numerical_results}

We present numerical results to demonstrate the performance of the AdaLinUCB algorithm using both synthetic scenario and real-world datasets.
We have implemented the following algorithms: 
1) AdaLinUCB in Algo.~\ref{alg:AdaLinUCB};
2) LinUCB(Extracted) in Sec.~\ref{se:LinUCB_in_mainPpaer};
3) \textbf{LinUCBMultiply}, another way to directly apply LinUCB in opportunistic case,
where we use $ L_t \cdot x_{t,a}$ as context vector;
4) \textbf{E-AdaLinUCB}, an algorithm that 
adjusts the threshold $l^{(+)}$ and  $l^{(-)}$ based on the empirical distribution of $L_t$. 
In all the algorithms, we set $\alpha=1.5$ to make a fair comparison.

We have also experimented  \textbf{LinUCBCombine} algorithm, where we use $\tilde{x}_{t,a}= [L_t, x_{t,a}^\top]^\top$ as context vector to directlty apply LinUCB, and find that LinUCBCombine has a much worse performance compared to other algorithms. 

Meanwhile, we also notice that the opportunistic linear contextual bandits can be regarded as a special case of non-linear contextual bandits by viewing the variation factor $L_t$ as a part of context vector. Along this line of thinking, we have also experimented \textbf{KernelUCB} algorithm \cite{valko2013finite}, which is a general algorithm for non-linear contextual bandits. However, we find that 
KernelUCB is less competitive in performance and suffers from extremely high computational complexity (see Appendix~\ref{se:ap_simu_kernel} for more details).
One reason is that a general contextual bandit algorithm such as KernelUCB does not take advantage of the opportunistic nature of the problem, and can, therefore, have a worse performance than AdaLinUCB.

\subsection{Experiments on Synthetic Scenarios}

The synthetic scenario has a total of $20$ possible arms, each associated with a disjoint unknown coefficient   $\theta_\star^{(a)}$.
The simulator generates $5$ possible groups of context vectors, and each group has context vectors associated with all the possible arms.
At each time slot, 
a context group is presented before the decision.
Further, each unknown coefficient $\theta_\star^{(a)}$ and each context $x_{t,a}$ is a $6$-dimension vector, with elements in each dimension generated randomly, and is normalized such that the L2-norm of $\theta_\star^{(a)}$ or $x_{t,a}$ is $1$. 

Fig.~\ref{fig:BinaryValue_L} shows the regret for different algorithms under random binary-value variation factor with $\epsilon_0=\epsilon_1=0$ and $\rho=0.5$. AdaLinUCB significantly reduces the regret in this scenario. Specifically, at time slots $t=5\times 10^4$, AdaLinUCB achieves a regret that is only  $10.3\%$ of that of LinUCBMultiply, and 17.6\% of that of LinUCBExtracted. 

For continuous variation factor, Fig.~\ref{fig:beta_L} compares the regrets for the algorithms under a beta distributed variation factor. 
Here, we define $l_\rho^{(-)}$ as the lower threshold such that $\mathbb{P}\{L_t\! \leq \! l_\rho^{(-)} \!=\!\rho \}$, and $l_\rho^{(+)}$ as the higher threshold such that $\mathbb{P}\{L_t \! \geq\! l_\rho^{(+)}\! =\! \rho \}$.
It is shown that AdaLinUCB still outperforms other algorithms, 
and AdaLinUCB has a regret $41.8\%$ lower than that of LinUCBExtracted.
Furthermore, its empirical version, E-AdaLinUCB has a similar performance to that of AdaLinUCB. 
Even in the special case with a single threshold $l^{(-)}\!\!=l^{(+)}\!\!=l^{(-)}_{0.5}$, AdaLinUCB still outperforms LinUCBExtracted, reducing the regret by $28.6\%$.

We have conducted more simulations to evaluate the impact of environment and algorithm parameters such as variation factor fluctuation and the thresholds for variation factor truncation, and find that AdaLinUCB works well in different scenarios (see Appendix~\ref{se:ap_simu_binary} and \ref{se:ap_simu_beta} ).

\subsection{Experiments on Yahoo! Today Module}

We also test the performance of the algorithms using the data from Yahoo! Today Module. 
This dataset contains over $4$ million user visits to the Today module in a ten-day period in May 2009
\cite{Li2010_LinUCB}.
To evaluate contextual bandits using offline data, the experiment uses the unbiased offline evaluation protocol proposed in \cite{Li2011_unbiased_LinUCB_simu}.
For the variation factor, we use a real trace - the sales of a popular store. It includes everyday turnover in two years \cite{Ross}.

In this real recommendation scenario, because we do not know the ground truth; i.e., which article is best for a specific user, we cannot calculate the regret. Therefore, all the results are measured using the reward, as  shown in Fig. \ref{fig:Yahoo_Module}. 
We note that AdaLinUCB increases the reward by  $17.0\%$, compared to LinUCBExtracted, and by $40.8\%$ compared to the random policy.  We note that an   increase in accumulated reward is typically much more substantial than the same decrease in regret. 
We also note that E-AdaLinUCB, where one does not assume prior knowledge on the variation factor distribution, achieves a similar performance.  
This experiment demonstrates the
effectiveness of AdaLinUCB and E-AdaLinUCB in practical situations,
where the variation factor are continuous and
are possibly non-stationary, and the candidate arms are time-varying. More details on the datasets and evaluation under different parameters can be found in Appendix~\ref{se:ap_simu_real}. 


\section{Conclusions}
In this paper, we study opportunistic contextual bandits where the exploration cost is time-varying depending on external conditions such as network load or return variation in recommendations. 
We propose AdaLinUCB that opportunistically chooses between exploration and exploitation based on that external variation factor, i.e., taking the slots with low variation factor as opportunities for more explorations. 
We prove that AdaLinUCB achieves $O\left((\log T)^2 \right)$ problem-dependent regret upper bound, which has a smaller coefficient thatn that of the traditional LinUCB algorithm.
Extensive experiment results based on both synthetic and real-world database demonstrate the significant benefits of opportunistic exploration under large exploration cost fluctuations.
\section*{Acknowledgment}
The work is partially supported supported by NSF
through grants CNS-1547461, CNS-1718901, and IIS-1838207.

\bibliographystyle{named}
\bibliography{OppContextBandit}

\begin{thebibliography}{}

\bibitem[\protect\citeauthoryear{Abbasi-Yadkori \bgroup \em et al.\egroup
  }{2011}]{Abbasi2011}
Yasin Abbasi-Yadkori, D\'{a}vid P\'{a}l, and Csaba Szepesv\'{a}ri.
\newblock Improved algorithms for linear stochastic bandits.
\newblock In {\em Proceedings of the 24th International Conference on Neural
  Information Processing Systems}, NIPS'11, pages 2312--2320, 2011.

\bibitem[\protect\citeauthoryear{Agrawal}{1995}]{Agrawal1995}
Rajeev Agrawal.
\newblock Sample mean based index policies with o(log n) regret for the
  multi-armed bandit problem.
\newblock {\em Advances in Applied Probability}, 27(4):1054--1078, 1995.

\bibitem[\protect\citeauthoryear{Auer \bgroup \em et al.\egroup
  }{2002}]{Auer2002a}
Peter Auer, Nicol\`{o} Cesa-Bianchi, and Paul Fischer.
\newblock Finite-time analysis of the multiarmed bandit problem.
\newblock {\em Mach. Learn.}, 47(2-3):235--256, 2002.

\bibitem[\protect\citeauthoryear{Auer}{2002}]{Auer2002_2003}
Peter Auer.
\newblock Using confidence bounds for exploitation-exploration trade-offs.
\newblock {\em J. Mach. Learn. Res.}, 3:397--422, 2002.

\bibitem[\protect\citeauthoryear{Bouneffouf \bgroup \em et al.\egroup
  }{2012}]{Bouneffouf2012}
Djallel Bouneffouf, Amel Bouzeghoub, and Alda~Lopes Gan{\c{c}}arski.
\newblock A contextual-bandit algorithm for mobile context-aware recommender
  system.
\newblock In {\em International conference on Neural Information Processing
  (ICONIP)}, pages 324--331, Berlin, Heidelberg, 2012. Springer Berlin
  Heidelberg.

\bibitem[\protect\citeauthoryear{Chapelle and Li}{2011}]{Chapelle2011_TS}
Olivier Chapelle and Lihong Li.
\newblock An empirical evaluation of thompson sampling.
\newblock In J.~Shawe-Taylor, R.~S. Zemel, P.~L. Bartlett, F.~Pereira, and
  K.~Q. Weinberger, editors, {\em Advances in Neural Information Processing
  Systems 24}, pages 2249--2257. Curran Associates, Inc., 2011.

\bibitem[\protect\citeauthoryear{Chu \bgroup \em et al.\egroup
  }{2009}]{Chu2009_LinUCB_feature}
Wei Chu, Seung-Taek Park, Todd Beaupre, Nitin Motgi, Amit Phadke, Seinjuti
  Chakraborty, and Joe Zachariah.
\newblock A case study of behavior-driven conjoint analysis on yahoo!: Front
  page today module.
\newblock In {\em Proceedings of the 15th ACM SIGKDD International Conference
  on Knowledge Discovery and Data Mining}, KDD '09, pages 1097--1104, 2009.

\bibitem[\protect\citeauthoryear{Chu \bgroup \em et al.\egroup
  }{2011}]{Chu2011_LinUCB_analysis}
Wei Chu, Lihong Li, Lev Reyzin, and Robert Schapire.
\newblock Contextual bandits with linear payoff functions.
\newblock In {\em Proceedings of the Fourteenth International Conference on
  Artificial Intelligence and Statistics}, volume~15 of {\em AISTATS}, pages
  208--214, 2011.

\bibitem[\protect\citeauthoryear{Chuai \bgroup \em et al.\egroup
  }{2019}]{Chuai2019_Infocom}
Jie Chuai, Zhitang Chen, Guochen Liu, Xueying Guo, Xiaoxiao Wang, Xin Liu,
  Chongming Zhu, and Feiyi Shen.
\newblock A collaborative learning based approach for parameter configuration
  of cellular networks.
\newblock In {\em IEEE International Conference on Computer Communications
  (INFOCOM)}, 2019.

\bibitem[\protect\citeauthoryear{Filippi \bgroup \em et al.\egroup
  }{2010}]{Filippi2010}
Sarah Filippi, Olivier Cappe, Aur\'{e}lien Garivier, and Csaba Szepesv\'{a}ri.
\newblock Parametric bandits: The generalized linear case.
\newblock In J.~D. Lafferty, C.~K.~I. Williams, J.~Shawe-Taylor, R.~S. Zemel,
  and A.~Culotta, editors, {\em Advances in Neural Information Processing
  Systems 23}, pages 586--594. Curran Associates, Inc., 2010.

\bibitem[\protect\citeauthoryear{Langford and Zhang}{2008}]{Langford2007}
John Langford and Tong Zhang.
\newblock The epoch-greedy algorithm for multi-armed bandits with side
  information.
\newblock In J.~C. Platt, D.~Koller, Y.~Singer, and S.~T. Roweis, editors, {\em
  Advances in Neural Information Processing Systems 20}, pages 817--824. Curran
  Associates, Inc., 2008.

\bibitem[\protect\citeauthoryear{Li \bgroup \em et al.\egroup
  }{2010}]{Li2010_LinUCB}
Lihong Li, Wei Chu, John Langford, and Robert~E. Schapire.
\newblock A contextual-bandit approach to personalized news article
  recommendation.
\newblock In {\em the 19th International Conference on World Wide Web (WWW)},
  2010.

\bibitem[\protect\citeauthoryear{Li \bgroup \em et al.\egroup
  }{2011}]{Li2011_unbiased_LinUCB_simu}
Lihong Li, Wei Chu, John Langford, and Xuanhui Wang.
\newblock Unbiased offline evaluation of contextual-bandit-based news article
  recommendation algorithms.
\newblock In {\em Proceedings of the Fourth ACM International Conference on Web
  Search and Data Mining}, WSDM '11, pages 297--306. ACM, 2011.

\bibitem[\protect\citeauthoryear{Rasmussen}{2004}]{rasmussen2004gaussian}
Carl~Edward Rasmussen.
\newblock Gaussian processes in machine learning.
\newblock In {\em Advanced lectures on machine learning}, pages 63--71.
  Springer, 2004.

\bibitem[\protect\citeauthoryear{Rossman}{2015}]{Ross}
Rossman.
\newblock {\em Rossmann Store sales data}, 2015.
\newblock \url{https://www.kaggle.com/c/rossmann-store-sales/data}.

\bibitem[\protect\citeauthoryear{Valko \bgroup \em et al.\egroup
  }{2013}]{valko2013finite}
Michal Valko, Nathaniel Korda, R{\'e}mi Munos, Ilias Flaounas, and Nelo
  Cristianini.
\newblock Finite-time analysis of kernelised contextual bandits.
\newblock {\em arXiv preprint arXiv:1309.6869}, 2013.

\bibitem[\protect\citeauthoryear{Walraevens \bgroup \em et al.\egroup
  }{2003}]{Walraevens2003}
Joris Walraevens, Bart Steyaert, and Herwig Bruneel.
\newblock Performance analysis of a single-server atm queue with a priority
  scheduling.
\newblock {\em Computers \& Operations Research}, 30(12):1807 -- 1829, 2003.

\bibitem[\protect\citeauthoryear{Wang \bgroup \em et al.\egroup
  }{2016}]{Wang2016}
Huazheng Wang, Qingyun Wu, and Hongning Wang.
\newblock Learning hidden features for contextual bandits.
\newblock In {\em Proceedings of the 25th ACM International on Conference on
  Information and Knowledge Management}, CIKM '16, pages 1633--1642. ACM, 2016.

\bibitem[\protect\citeauthoryear{Wang \bgroup \em et al.\egroup
  }{2017}]{Wang2017FactorizationBF}
Huazheng Wang, Qingyun Wu, and Hongning Wang.
\newblock Factorization bandits for interactive recommendation.
\newblock In {\em Proceedings of the Thirty-First AAAI Conference on Artificial
  Intelligence}, AAAI'17, 2017.

\bibitem[\protect\citeauthoryear{Wu \bgroup \em et al.\egroup }{2016}]{Wu2016}
Qingyun Wu, Huazheng Wang, Quanquan Gu, and Hongning Wang.
\newblock Contextual bandits in a collaborative environment.
\newblock In {\em Proceedings of the 39th International ACM SIGIR Conference on
  Research and Development in Information Retrieval}, SIGIR '16, pages
  529--538. ACM, 2016.

\bibitem[\protect\citeauthoryear{Wu \bgroup \em et al.\egroup
  }{2018}]{Wu2018_AdaUCB}
Huasen Wu, Xueying Guo, and Xin Liu.
\newblock Adaptive exploration-exploitation tradeoff for opportunistic bandits.
\newblock In {\em ICML}, 2018.

\bibitem[\protect\citeauthoryear{Zhang}{2006}]{zhang2006schur}
Fuzhen Zhang.
\newblock {\em The Schur complement and its applications}, volume~4.
\newblock Springer Science \& Business Media, 2006.

\end{thebibliography}

\newpage
\appendix

\section{Preparatory Results for Analysis}\label{se:LemmaProofs}

We begin with some preparatory analysis results. 

\begin{lemma}\label{th:lemma_confidence_set}
Assume that the noise satisfies the $C_\mathrm{noise}$-sub-Gaussian condition in \eqref{as:sub_Gaussian}, and assume that $||\theta_\star||_2 \leq C_\mathrm{theta}$. Then, the following results hold:
\begin{enumerate}[1)]
\item For any $\delta \! \in\! (0,1)$, with probability at least $1\!-\!\delta$, for all $t\geq 0$,  $\theta_\star \in  \big\{ 
\theta \in \mathbb{R}^d :
\left \|  \theta_t - \theta_\star \right \|_{A_t} \leq \det(I_d)^{\frac{1}{2}}  C_\mathrm{theta} $
$+ C_\mathrm{noise} 
\sqrt{
2 \log \bigg(
\frac{\det(A_t)^{\frac{1}{2}} \det(I_d)^{-\frac{1}{2}} }{\delta}
\bigg)
} 
\bigg\}$.
\item Further, if $\lVert x_{t,a} \rVert_2 \leq C_\mathrm{context} $ holds for $\forall t, \forall a\in \mathcal{D}_t$, then, with probability at least $1-\delta$,
$\theta_\star \in   \big\{
\theta \in \mathbb{R}^d :$
$\left \lVert  \theta_t - \theta_\star \right \rVert_{A_t} \leq C_\mathrm{theta} +
C_\mathrm{noise} \sqrt{
	d \log \bigg(
	\frac{1 + t C_\mathrm{context}^2 }{\delta}
	\bigg)
} 
\bigg\}$.
\end{enumerate}
\end{lemma}
\begin{proof}
it simply follows from the fact that $\theta_t$ is the result of a ridge regression, and that the sub-Gaussian condition is assumed. The technique is as in  \cite{Abbasi2011} (specifically, Theorem 2 in \cite{Abbasi2011}). 
\end{proof}

Lemma \ref{th:lemma_confidence_set} shows that the estimation $\theta_t$ is close to the unknown parameter $\theta_\star$ in an appropriate sense. 

\begin{lemma}\label{th:lemma_inequality_reward_estimate}
For $\forall t \geq 1$, $\forall a \in \mathcal{D}_t$, the following result holds,
\begin{align*}
\lvert 
\langle \theta_\star, x_{t,a} \rangle -
\langle \theta_{t-1}, x_{t,a}
\rangle
\rvert
\leq 
\lVert  \theta_{t-1} - \theta_\star  \rVert_{A_{t-1}}
\cdot 
\lVert x_{t,a} \rVert_{A_{t-1}^{-1}}.
\end{align*}
\end{lemma}
\begin{proof}
We have the following,
\begin{align*}
 &\lvert 
\langle \theta_\star, x_{t,a} \rangle -
\langle \theta_{t-1}, x_{t,a}
\rangle
\rvert \\
 =&
\lvert \left( \theta_\star - \theta_{t-1}  \right)^\top x_{t,a}  \rvert \\
 = & \lvert \left( \theta_\star - \theta_{t-1}  \right)^\top 
A_{t-1}^{\frac{1}{2}} A_{t-1}^{-\frac{1}{2}}
x_{t,a}  \rvert \\
 =& \big \lvert ~\langle 
A_{t-1}^{\frac{1}{2}} \left( \theta_\star - \theta_{t-1}  \right),~
A_{t-1}^{-\frac{1}{2}} x_{t,a}
\rangle~ \big \rvert \\
\leq & \lVert A_{t-1}^{\frac{1}{2}} ( \theta_\star - \theta_{t-1}  ) \rVert_2 \cdot
\lVert A_{t-1}^{-\frac{1}{2}} x_{t,a} \rVert_2 \\
 =& 
 \resizebox{.91\linewidth}{!}{$
 \sqrt{
( \theta_\star - \theta_{t-1}  )^\top
A_{t-1}^{\frac{1}{2}} A_{t-1}^{\frac{1}{2}}
( \theta_\star - \theta_{t-1}  ) 
}
\cdot
\sqrt{
x_{t,a}^\top
A_{t-1}^{-\frac{1}{2}} A_{t-1}^{-\frac{1}{2}}
x_{t,a}
} $
}\\
 =& \lVert  \theta_{t-1} - \theta_\star  \rVert_{A_{t-1}}
\cdot 
\lVert x_{t,a} \rVert_{A_{t-1}^{-1}},
\end{align*}
where the third equality holds by noting that a positive-definite matrix $A_{t-1}$ is symmetric; the inequality holds by Cauchy–Schwarz inequality.
\end{proof}

 Lemma~\ref{th:lemma_inequality_reward_estimate} presents an upper bound on the estimation error of the reward corresponding to a given context vector. 
Combining Lemma~\ref{th:lemma_confidence_set}, the first term of this upper bound, i.e., $\lVert  \theta_{t-1} - \theta_\star  \rVert_{A_{t-1}}$, can be upper bounded with a high probability. To further consider the property of the second term, i.e., $\lVert x_{t,a} \rVert_{A_{t-1}^{-1}}$, we bound the summation by the following result.

\begin{lemma}\label{th:lemma_summation_xt}
Assume that $\lVert x_{t,a} \rVert_2 \leq C_\mathrm{context} $ holds for $\forall t, \forall a\in \mathcal{D}_t$, and assume that $\lambda_{\min}(I_d)\geq \max\{1, C_\mathrm{context}^2 \}$, then the following results hold,
\begin{align}\label{eq:lemma_summation_xt}
& \sum_{t=1}^{T} \lVert x_{t,a_t} \rVert_{A_{t-1}^{-1}}^2
\leq 2 \log \Big( \det (A_T) /\det (I_d) \Big) \\ \nonumber
& \leq  2 \left[
d \log\Big( \frac{
\mathrm{trace} (I_d) + T C_\mathrm{context}^2
}{d} \Big)
-\log\det(I_d)
\right].
\end{align}
\end{lemma}

Lemma~\ref{th:lemma_summation_xt}  directly follows from \cite{Abbasi2011} (specifically, Lemma 11 of \cite{Abbasi2011}).

\begin{remark} (\textbf{Relax assumption iv. in Sec.~\ref{se:Performance Analysis}.})
Now, we briefly discuss the way to 
relax the assumption $\lambda_{\min}(I_d)\geq \max\{1, C_\text{context}^2 \}$ in Lemma~\ref{th:lemma_summation_xt} and the theorems using this Lemma. 
For this assumption, it can be relaxed by changing the initial value of matrix $A$ in Algo.~\ref{alg:AdaLinUCB}. Currently, we have the initial matrix value $A_0 = I_d$, leading to the current  results in Lemma~\ref{th:lemma_summation_xt}. 
If the initial value $A_0$ is changed to a positive-definite matrix with a higher minimum eigenvalue, then with a modified assumption $\lambda_{\min}(A_0)\geq \max\{1, C_\text{context}^2 \}$, Lemma~\ref{th:lemma_summation_xt} still holds after substituting the matrix $I_d$ by the new $A_0$. One example of such a new $A_0$ can be $C_\text{context} \cdot I_d$ with $C_\text{context} > 1$. Also, we note that Lemma~\ref{th:lemma_inequality_reward_estimate} still holds after changing $A_0$ to another positive-definite matrix. Further, when changing $A_0$ to another positive-definite matrix, the first statement of Lemma~\ref{th:lemma_confidence_set} still holds after substituting the matrix $I_d$ by the new matrix $A_0$, while the second statement following from the first one and can be changed accordingly. Thus, for assumption $\lambda_{\min}(I_d)\geq \max\{1, C_\text{context}^2 \}$, we can modify the choice of $A_0$ to relax this assumption.
\end{remark}

\section{Proof for Theorem~\ref{th:AdaLinUCB_problem_dependent}}\label{se:Proof_AdaLinUCB}
Firstly, note that we have some preparatory analysis results, as shown in Appendix~\ref{se:LemmaProofs}.

We begin with some notations. 
Let $\mathbf{R}_\text{total}^\text{(low)}(T)$ denote the accumulated regret regarding nominal rewards when the variation factor is low, i.e.,
$\mathbf{R}_\text{total}^\text{(low)}(T) 
=\sum_{t=1}^{T} R_t \mathbbm{1}\{
L_t = \epsilon_0\}$, where $ \mathbbm{1}\{ \cdot \}$ is the indicator function.
The $\mathbf{R}_\text{total}^\text{(high)}(T)$ is defined similarly as  $\mathbf{R}_\text{total}^\text{(high)}(T) 
=\sum_{t=1}^{T} R_t \mathbbm{1}\{
L_t = 1-\epsilon_1\}$.

Then, we have that, 
\begin{align*}
\tilde{\mathbf{R}}_\text{total}(T)
= \epsilon_0 \cdot \mathbf{R}_\text{total}^\text{(low)}(T) 
+ (1-\epsilon_1) \cdot
\mathbf{R}_\text{total}^\text{(high)}(T).
\end{align*}
As a result, to prove this theorem, it is sufficient to prove that,
with probability at least $1-\tilde{\delta}$, both of the following results hold:
\begin{align}
\label{eq:AdaProof_low level total}
\nonumber
\mathbf{R}_\text{total}^\text{(low)}(T) &\leq 
\frac{16 C_\text{noise}^2 C_\text{theta}^2 }{\Delta_{\min}}
\Bigg[\log(C_\text{context} T) + 2 \log\frac{2}{\tilde{\delta}}
\\ \nonumber & + 2 (d-1) \log \left(
d\log\frac{d+ T C_\text{context}^2}{d} + 2 \log\frac{2}{\tilde{\delta}}
\right) 
\\  
&+(d-1) \log\frac{64 C_\text{noise}^2 C_\text{theta}^2 C_\text{context} }{\Delta_{\min}^2} 
\Bigg]^2, 
\end{align}

\begin{align}
\label{eq:AdaProof_high level total}
\nonumber
& \mathbf{R}_\text{total}^\text{(high)}(T)  \leq \!\!
\Bigg[
4  d \frac{N\!-\!1}{\Delta_{\min}}\! 
\bigg(\!\!
C_\text{noise}\sqrt{d\log
\!	\frac{2\!+\!2TC_\text{context}^2}{\tilde{\delta}} }\! +\! C_\text{theta} \!
\bigg)^2 
\\ 
& \quad +  \Delta_{\max} C_\text{slots} 
\bigg( \!\!
C_\text{noise}\sqrt{d\log
	\frac{2\!+\!2TC_\text{context}^2}{\tilde{\delta}} } \!+\! C_\text{theta}\!\!
\bigg)^2 
\Bigg].
\end{align}

\subsection{Regret for slots with low variation factor:}
Firstly, we focus on $\mathbf{R}_\text{total}^\text{(low)}(T)$. 
For the binary-valued variation factor,
let the lower threshold $l^{(-)}=\epsilon_0$. 
Then the variation factor $L_t=\epsilon_0$, while the normalized variation factor $\tilde{L}_t = 0$. 
As a result, the index $p_{t,a}$ in step~\ref{line:AdaLinUCB_index} of Algo.~\ref{alg:AdaLinUCB} becomes,
\begin{align*}
p_{t,a} =\theta_{t-1}^\top x_{t,a} + \alpha \sqrt{   x_{t,a}^\top A_{t-1}^{-1} x_{t,a} },
\end{align*}
Then, for $\forall t \geq 1$ with $L_t= \epsilon_0$, if $\alpha \geq \lVert
\theta_{t-1} - \theta_\star
\rVert_{A_{t-1}}$, we have,
\begin{align*}
R_t & = \langle x_{t,a_t^\star}, \theta_\star \rangle - \langle x_{t,a_t}, \theta_\star \rangle \\
& \leq 
\langle x_{t,a_t^\star}, \theta_{t-1} \rangle 
+ \alpha \lVert x_{t,a_t^\star} \rVert_{A_{t-1}^{-1}}
- \langle x_{t,a_t}, \theta_\star \rangle \\
& \leq 
\langle x_{t,a_t}, \theta_{t-1} \rangle 
+ \alpha \lVert x_{t,a_t} \rVert_{A_{t-1}^{-1}}
- \langle x_{t,a_t}, \theta_\star \rangle
\\
& =
\langle x_{t,a_t}, \theta_{t-1} \rangle 
- \langle x_{t,a_t}, \theta_\star \rangle
+ \alpha \lVert x_{t,a_t} \rVert_{A_{t-1}^{-1}}
\\
& \leq 
\lVert
\theta_{t-1} - \theta_\star
\rVert_{A_{t-1}}  \lVert x_{t,a_t} \rVert_{A_{t-1}^{-1}} + \alpha \lVert x_{t,a_t} \rVert_{A_{t-1}^{-1}} \\
& \leq 2 \alpha \lVert x_{t,a_t} \rVert_{A_{t-1}^{-1}},
\end{align*}
where the inequality in the second line holds by Lemma~\ref{th:lemma_inequality_reward_estimate} and $\alpha \geq \lVert
\theta_{t-1} - \theta_\star
\rVert_{A_{t-1}}$; the inequality in the third line holds by the design of the AdaLinUCB algorithm, specifically, by step~\ref{line:AdaLinUCB_selection} of Algo.~\ref{alg:AdaLinUCB}; the inequality in the fifth line holds by Lemma~\ref{th:lemma_inequality_reward_estimate}, and the last inequality holds by $\alpha \geq \lVert
\theta_{t-1} - \theta_\star
\rVert_{A_{t-1}}$.
As a result, we have,
\begin{align*}
R_t \mathbbm{1} \{L_t = \epsilon_0 \}
\leq 2 \alpha \lVert x_{t,a_t} \rVert_{A_{t-1}^{-1}},
\end{align*}
with $\alpha \geq \lVert
\theta_{t-1} - \theta_\star
\rVert_{A_{t-1}}$.
Then, we have,
\begin{align}\label{eq:AdaProof_t1}
\mathbf{R}_\text{total}^\text{(low)}(T) 
& =\sum_{t=1}^{T} R_t \mathbbm{1}\{
L_t = \epsilon_0\} \leq \sum_{t=1}^{T} 2 \alpha  \lVert x_{t,a_t} \rVert_{A_{t-1}^{-1}}, 
\end{align}
with $\alpha \geq \lVert
\theta_{T-1} - \theta_\star
\rVert_{A_{T-1}}$.

We also note that, by Lemma~\ref{th:lemma_confidence_set}, with probability at least $1-\frac{\tilde{\delta}}{2}$, for all $t$, 
\begin{align}\label{eq:AdaProof_t2}
\nonumber
\theta_\star \in \Big\{ 
&\theta \in \mathbb{R}^d : 
\left \lVert  \theta_t - \theta_\star \right \rVert_{A_t}\\
& \leq C_\text{theta} +
C_\text{noise} \sqrt{
	d \log \left(
	\frac{2 + 2t C_\text{context}^2 }{\tilde{\delta}}
	\right)
} 
\Bigg\},
\end{align}
which substitutes the $\delta$ in Lemma~\ref{th:lemma_confidence_set} by $\frac{\tilde{\delta}}{2}$.

Further, we note that,
\begin{align}\label{eq:AdaProof_t2p1}
\mathbf{R}_\text{total}(T) = \sum_{t=1}^{T} R_t \leq \sum_{t=1}^{T} \frac{R_t^2}{\Delta_{\min}},
\end{align}
where the inequality holds since either $R_t=0$ or $\Delta_{\min}<=R_t$.

Then, by combining  \eqref{eq:AdaProof_t1}, \eqref{eq:AdaProof_t2}, and \eqref{eq:AdaProof_t2p1}, it follows from a similar argument as  \cite{Abbasi2011} (specifically, the proof of Theorem 5 in \cite{Abbasi2011}) that \eqref{eq:AdaProof_low level total} holds with probability at least $1-\frac{\tilde{\delta}}{2}$.
Note that the proof procedure uses Lemma~\ref{th:lemma_summation_xt} and the single optimal context condition.

\subsection{Regret for slots with high variation factor}

Now, we focus on $\mathbf{R}_\text{total}^\text{(high)}(T) $. 
We begin with some notations. 
The $N$ possible values of context vectors are denoted by $x_{(1)}, x_{(2)}, \cdots, x_{(N)}$ respectively. 
Without loss of generality, we assume that $x_{(1)}$ is the optimal context value, i.e., $x_{(1)}=x_\star$.
Let $m_{t,\star}$ be the number of times that the arm with the optimal context value has been pulled before time slot $t$, i.e., 
$m_{t,\star} = \sum_{\tau=1}^t \mathbbm{1}\{x_{\tau,a_\tau}= x_\star \} $.
Similarly, let $m_{t,(n)}$ be the number of times that the arm with context value $x_{(n)}$ has been pulled before time slot $t$, i.e., 
$m_{t,(n)} = \sum_{\tau=1}^t \mathbbm{1}\{x_{\tau,a_\tau}= x_{(n)} \} $.
In addition, let $m_{t,\star}^\text{(low)}$ be the number of times when the variation factor is low and the arm with the optimal context value $x_\star$ has been pulled during $t$-slot period, i.e., 
$m_{t,\star}^\text{(low)} = \sum_{\tau=1}^t \mathbbm{1}\{x_{\tau,a_\tau}= x_\star \} \cdot \mathbbm{1}\{L_\tau=\epsilon_0 \} $.
Let $m_{t,\text{all}}^\text{(low)}$ be the number of times when the variation factor is low during $t$-slot period, i.e., 
$m_{t,\text{all}}^\text{(low)} = \sum_{\tau=1}^t \mathbbm{1}\{L_\tau=\epsilon_0 \} $.
Further, let $m_{t,\text{subopt}}^\text{(low)}$ be the number of times when the variation factor is low and the arm with a suboptimal context value has been pulled during $t$-slots, i.e., 
$m_{t,\text{subopt}}^\text{(low)} = m_{t,\text{all}}^\text{(low)} -  m_{t,\star}^\text{(low)}$.

\begin{lemma}\label{th:lemma_UCBwidth}
For the AdaLinUCB algorithm, the following inequality holds, for any $n = 1,2, \cdots, N$,
\begin{align*}
\lVert x_{(n)} \rVert_{A_{t-1}^{-1}} \leq \sqrt{\frac{d}{m_{t-1, (n)}}}.
\end{align*}
\end{lemma}
\begin{proof}
We note that,
\begin{align*}
d & =\mathrm{trace}(I_d) 
= \mathrm{trace}\left(
A_{t-1}^{-1} A_{t-1}
\right) \\
& = \mathrm{trace}\left(
A_{t-1}^{-1}
\left[ \sum_{i=1}^{N}
m_{t-1,(i)} \cdot x_{(i)}x_{(i)}^\top +I_d
\right]
\right) \\
& = \mathrm{trace}\left(
 \sum_{i=1}^{N}
m_{t-1,(i)} \cdot A_{t-1}^{-1} x_{(i)}x_{(i)}^\top + A_{t-1}^{-1}
\right) \\
& = 
 \sum_{i=1}^{N}
m_{t-1,(i)} \cdot
\mathrm{trace}\left( A_{t-1}^{-1} x_{(i)}x_{(i)}^\top \right) +
\mathrm{trace}\left(
A_{t-1}^{-1}
\right) \\
& = 
 \sum_{i=1}^{N}
m_{t-1,(i)} \cdot
\mathrm{trace}\left(x_{(i)}^\top A_{t-1}^{-1} x_{(i)} \right) +
\mathrm{trace}\left(
A_{t-1}^{-1}
\right) \\
& = 
 \sum_{i=1}^{N}
m_{t-1,(i)} \cdot
x_{(i)}^\top A_{t-1}^{-1} x_{(i)} +
\mathrm{trace}\left(
A_{t-1}^{-1}
\right) \\
& \geq
m_{t-1,(n)} \cdot
x_{(n)}^\top A_{t-1}^{-1} x_{(n)}, \quad \forall n=1,2,\cdots,N,
\end{align*}
where the last inequality holds by noting that $A_{t-1}^{-1}$ is a positive-definite matrix.
Then, the results follow.
\end{proof}

In the following, we let,
\begin{align}\label{eq:AdaProof_alphaT}
\alpha_T = C_\text{noise} \sqrt{
	d \log \left(
	\frac{2 + 2 T C_\text{context}^2 }{\tilde{ \delta}}
	\right)
} + C_\text{theta}.
\end{align}
Thus, by \eqref{eq:AdaProof_t2}, with probability at least $1-\frac{\tilde{\delta}}{2}$, for all $t\leq T$, the following inequality holds.
\begin{align}\label{eq:AdaProof_t3}
\lVert  \theta_t - \theta_\star  \rVert_{A_t}\leq \alpha_T.
\end{align}
Let the $\alpha$ in AdaLinUCB algorithm be $\alpha =\alpha_T$.

\begin{lemma}\label{th:lemma_optSelect_condition}
When inequality \eqref{eq:AdaProof_t3} holds, for any slot $t$ with variation factor $L_t = 1-\epsilon_1$,
if,
\begin{align}\label{eq:AdaProof_t4}
\langle x_\star, \theta_\star \rangle -
\langle
x_{(n)}, \theta_\star
\rangle
>
\alpha_T \lVert x_\star \rVert_{A_{t-1}^{-1}}
\!+\alpha_T \lVert x_{(n)} \rVert_{A_{t-1}^{-1}}, ~~\forall n,
\end{align}
then the arm with the optimal context is pulled in slot $t$, i.e., $R_t=0$.
\end{lemma}
\begin{proof}
For the binary-valued variation factor, let the higher threshold of variation factor  $l^{(+)}=1-\epsilon_1$. Thus, when the variation factor is high, i.e., $L_t = 1-\epsilon_1$, the truncated variation factor becomes $\tilde{L}_t =1$. As a result, the index in step~\ref{line:AdaLinUCB_index} of Algo.~\ref{alg:AdaLinUCB} becomes,
\begin{align*}
p_{t,a}  = \theta_{t-1}^\top x_{t,a} = \langle  x_{t,a},\theta_{t-1} \rangle.
\end{align*}
As a result, to prove that the arm with optimal context value is selected, it is sufficient to prove that,
\begin{align*}
\langle x_\star, \theta_{t-1}
\rangle
-
\langle x_{(n)}, \theta_{t-1}
\rangle>0,~~\forall n=2,\cdots, N.
\end{align*}

When inequality \eqref{eq:AdaProof_t3} holds, by Lemma~\ref{th:lemma_inequality_reward_estimate} , we have that,
\begin{align*}
\langle x_\star, \theta_{t-1}
\rangle \geq 
\langle x_\star, \theta_\star
\rangle
- \alpha_T \lVert x_\star \rVert_{A_{t-1}^{-1}},
\end{align*}
and,
\begin{align*}
\langle x_{(n)}, \theta_{t-1}
\rangle \leq
\langle x_{(n)}, \theta_\star
\rangle
- \alpha_T \lVert x_{(n)} \rVert_{A_{t-1}^{-1}}.
\end{align*}
As a result, for any $n=2,3,\cdots, N$,
\begin{align*}
& \langle x_\star, \theta_{t-1}
\rangle
-
\langle x_{(n)}, \theta_{t-1}
\rangle \\
\geq &
\langle x_\star, \theta_\star
\rangle
-
\langle x_{(n)}, \theta_\star
\rangle
-\alpha_T \lVert x_\star \rVert_{A_{t-1}^{-1}}
-\alpha_T \lVert x_{(n)}  \rVert_{A_{t-1}^{-1}} \\>&0,
\end{align*}
where the last inequality holds by the condition \eqref{eq:AdaProof_t4}  of this Lemma, which completes the proof.
\end{proof} 

By Lemma~\ref{th:lemma_optSelect_condition}, when inequality \eqref{eq:AdaProof_t3} holds, for any slot $t$ with variation factor $L_t = 1-\epsilon_1$, if both of the following inequalities holds,
\begin{align}\label{eq:AdaProof_t5}
\alpha_T \lVert x_\star \rVert_{A_{t-1}^{-1}} \leq  \frac{\Delta_{\min}}{2},
\end{align}
\begin{align}\label{eq:AdaProof_t6}
\alpha_T \lVert x_{(n)} \rVert_{A_{t-1}^{-1}} 
<
\frac{\langle x_\star, \theta_\star
\rangle
-
\langle x_{(n)}, \theta_\star
\rangle}{2},
~~ \forall n=2,\cdots, N,
\end{align}
then the arm with the optimal context is selected with probability at least $1-\tilde{\delta}$.

Now, we analyze when \eqref{eq:AdaProof_t6} holds.
For any suboptimal context value $x_{(n)}$ with $ n\neq 1$, by Lemma~\ref{th:lemma_UCBwidth}, \eqref{eq:AdaProof_t6} holds when,
\begin{align*}
m_{t-1, (n)} > \frac{4 d \alpha_T^2}{
\left[\langle x_\star, \theta_\star
\rangle
-
\langle x_{(n)}, \theta_\star
\rangle \right]^2
}.
\end{align*}
As a result, before \eqref{eq:AdaProof_t6} is satisfied, pulling the arms with the suboptimal context values increases $\mathbf{R}_\text{total}^\text{(high)}(T) $ by at most, 
\begin{align}\label{eq:AdaProof_r1}
\sum_{n=2}^{N}
\frac{4 d \alpha_T^2}{
\langle x_\star, \theta_\star
\rangle
-
\langle x_{(n)}, \theta_\star
\rangle 
}
\leq 
(N-1)
\frac{4 d }{ \Delta_{\min}} \alpha_T^2.
\end{align}
Note that the r.h.s. of \eqref{eq:AdaProof_r1} is the first term of \eqref{eq:AdaProof_high level total} by recalling $\alpha_T$ definition in \eqref{eq:AdaProof_alphaT}.

Now, we focus on analyzing when \eqref{eq:AdaProof_t5} holds.
For optimal context value $x_{(1)}=x_\star$,
by Lemma~\ref{th:lemma_UCBwidth}, \eqref{eq:AdaProof_t5} holds when,
\begin{align}\label{eq:AdaProof_t7}
m_{t-1, \star} > \frac{4 d \alpha_T^2}{\Delta_{\min}^2
}.
\end{align}
To analyze when \eqref{eq:AdaProof_t7} holds,
we can take advantage of \eqref{eq:AdaProof_low level total}, and note that, 
\begin{align*}
m_{t,\text{subopt}}^\text{(low)}    
\leq 
\frac{\mathbf{R}_\text{total}^\text{(low)}(t)}{\Delta_{\min}} .
\end{align*}
Thus, by \eqref{eq:AdaProof_low level total}, with probability at least $1-\frac{\tilde{\delta}}{2}$, for all $t$,
\begin{align}\label{eq:AdaProof_t9}
\nonumber
& m_{t,\text{subopt}}^\text{(low)}    
 \leq   
\frac{16 C_\text{noise}^2 C_\text{theta}^2 }{\Delta_{\min}^2}
\Bigg[\log(C_\text{context} t) \\ \nonumber
& + 2 (d-1) \log \left(
d\log\frac{d+ t C_\text{context}^2}{d} + 2 \log\frac{2}{\tilde{\delta}}
\right) \\
& + (d-1) \log\frac{64 C_\text{noise}^2 C_\text{theta}^2 C_\text{context} }{\Delta_{\min}^2} + 2 \log\frac{2}{\tilde{\delta}}
\Bigg]^2, 
\end{align}
where the  probability is introduced by \eqref{eq:AdaProof_t2} when proving \eqref{eq:AdaProof_low level total}. 

Further, for the binary-valued variation factor, by Hoeffding's inequality, we have that, with probability at least $1-\frac{\tilde{\delta}}{2}$,
\begin{align*}
m_{t,\text{all}}^\text{(low)}
\geq \rho t - \sqrt{
\frac{t}{2} \log\frac{2}{\tilde{\delta}},
}
\end{align*}
which also holds when $t=\alpha_T^2 C_\text{slots}$.
Thus, with probability at least $1-\frac{\tilde{\delta}}{2}$,
\begin{align}\label{eq:AdaProof_t8}
m_{\alpha_T^2 C_\text{slots},\text{all}}^\text{(low)}
\geq \rho \cdot \alpha_T^2 C_\text{slots} - \sqrt{
\frac{\alpha_T^2 C_\text{slots}}{2} \log\frac{2}{\tilde{\delta}}.
}
\end{align}

Then, by combining \eqref{eq:AdaProof_t9} and \eqref{eq:AdaProof_t8}, and by recalling $C_\text{slots}$ definition in \eqref{eq:Cslots_definition}, 
with probability at least $1-\tilde{\delta}$,
\begin{align*}
m_{\alpha_T^2 C_\text{slots}, \star} & \geq m_{\alpha_T^2 C_\text{slots}, \star}^\text{(low)} \\
& = m_{\alpha_T^2 C_\text{slots}, \text{all}}^\text{(low)}
- m_{\alpha_T^2 C_\text{slots}, \text{subopt}}^\text{(low)} \\
& \geq 
\frac{4 d \alpha_T^2}{\Delta_{\min}^2
}.
\end{align*}
Thus, with probability at least $1-\tilde{\delta}$, for $\forall t \geq \alpha_T^2 C_\text{slots}$, the inequality  \eqref{eq:AdaProof_t5} holds.
As a result, with probability at least $1-\tilde{\delta}$,
before \eqref{eq:AdaProof_t5} is satisfied, pulling the arms with the suboptimal context values increases $\mathbf{R}_\text{total}^\text{(high)}(T) $ by at most 
$\alpha_T^2 C_\text{slots} \Delta_{\max} $.
By combining \eqref{eq:AdaProof_r1}, we have that, with probability at least $1-\tilde{\delta}$, the inequality \eqref{eq:AdaProof_high level total} for $\mathbf{R}_\text{total}^\text{(high)}(T)$ holds.

\subsection{Combine Results and Finish Proof}
Further by noting that probabilities introduced in this proof procedure only comes from two events: i) confidence set for $\theta_t$ in \eqref{eq:AdaProof_t2}; ii) lower bound for number of total slots with low variation factor as in \eqref{eq:AdaProof_t8}.
Note that each of these two events with probability at least $1-\frac{\tilde{\delta}}{2}$ and that they are independent.
Thus, with probability at least $1-\tilde{\delta}$, both inequalities \eqref{eq:AdaProof_low level total} and \eqref{eq:AdaProof_high level total} hold, which completes the proof.

\section{Performance Analysis of LinUCB}\label{se:Appendix_LinUCB}

\subsection{LinUCB Algorithm Notation}
In opportunistic contextual bandit problem, one way to select bandits is to ignore the variation factor, i.e., $L_t$, and just employ the LinUCB algorithm, as shown in Algorithm \ref{alg:LinUCBExtracted}.
This algorithm is denoted as LinUCBExtracted in numerical restuls.

\begin{algorithm}[tb]
	\caption{LinUCB(Extracted)}
	\label{alg:LinUCBExtracted} 
	\begin{algorithmic}[1]
		\STATE{Inputs:} {$\alpha \in \mathbb{R}_+$, $d \in \mathbb{N}$.}
		\STATE{$A \leftarrow \bm{I}_{d}$ \{The $d$-by-$d$ identity matrix\}
		}
		\STATE{$b \leftarrow \bm{0}_d$}
		\FOR{$t = 1, 2, 3,\cdots, T$}
		\STATE{ $\theta_{t\!-\!1} = A^{-1} b$}\label{line:LinUCBExtracted_theta}
		\STATE{Observe possible arm set $\mathcal{D}_t$, and observe associated context vectors $x_{t,a}, \forall a \in \mathcal{D}_t$.
		}
		\FOR{$a \in \mathcal{D}_t$}
		\STATE{ $p_{t,a} \leftarrow \theta_{t\!-\!1}^\top x_{t,a} + \alpha \sqrt{x_{t,a}^\top A^{-1} x_{t,a} }$  \{Computes upper confidence bound\}
		}\label{line:LinUCBExtracted_UCB_index}
		\ENDFOR
		\STATE{Choose action $a_t = \arg\max_{a\in \mathcal{D}_t} p_{t,a}$ with ties broken arbitrarily.}\label{line:LinUCBExtracted_at_choice}
		\STATE{Observe nominal reward $r_{t,a_t}$}
		\STATE{$A \leftarrow A + x_{t,a_t} x_{t,a_t}^\top $}\label{line:LinUCBExtracted_A_update}
		\STATE{ $b \leftarrow b + x_{t, a_t} r_{t,a_t} $
		}\label{line:LinUCBExtracted_b_update}
		\ENDFOR
	\end{algorithmic}
\end{algorithm} 

In Algo.~\ref{alg:LinUCBExtracted}, for each time slot, the algorithm updates an matrix $A$ and a vector $b$, so that to estimate the unknown parameter for the linear function of context vector. To make the notation clear, denote $A_t=I_d + \sum_{\tau=1}^{t}x_{\tau, a_\tau} x_{\tau, a_\tau}^\top$, which is the matrix $A$ updated in step \ref{line:LinUCBExtracted_A_update} for each time slot.
It directly follows that $A_t, \forall t\geq 0$ is a positive-definite matrix.
Denote $b_t = \sum_{\tau=1}^{t} 
x_{\tau, a_\tau} r_{\tau,a_\tau}$, which is the vector $b$ updated in step \ref{line:LinUCBExtracted_b_update} for each time slot $t$.

As a result, the estimation of the unknown parameter $\theta_\star$ is denoted by $\theta_t$, as shown in step \ref{line:LinUCBExtracted_theta},
which satisfies,
\begin{align}\label{eq:theta_ridge_1}
\theta_t & = A_t^{-1} b_t \\ \nonumber
& =
\left(
I_d + \sum_{\tau=1}^{t}x_{\tau, a_\tau} x_{\tau, a_\tau}^\top
\right)^{-1}
\sum_{\tau=1}^{t} 
x_{\tau, a_\tau} r_{\tau,a_\tau}.
\end{align}
Note that $\theta_t$ is the result of a ridge regression. That is, $\theta_t$ is the coefficient that minimize a penalized residual sum of squares, i.e.,
\begin{align}\label{eq:theta_ridge_2}
\theta_t = \arg\min_\theta \left\{
\sum_{\tau=1}^{t} \Big(
r_{\tau,a_\tau} - \langle \theta, x_{\tau,a_\tau} \rangle
\Big)^2
+ \lVert \theta \rVert_2^2
\right\}
\end{align}
Here, the complexity parameter that controls the amount of shrinkage is chosen as $1$.

Also, we note that the upper confidence index $p_{t,a}$, as shown in step \ref{line:LinUCBExtracted_UCB_index} of Algo.~\ref{alg:LinUCBExtracted} consists of two parts.
The first part $\theta_{t-1}^\top x_{t,a} = \langle
\theta_{t-1}, x_{t,a}
 \rangle$ is the estimation of the corresponding reward, using the up-to-date estimation of the unknown parameter, i.e., $\theta_{t-1}$. 
The second part, i.e., $\alpha \sqrt{x_{t,a}^\top A_{t-1}^{-1} x_{t,a} } = \alpha \lVert x_{t,a} \rVert_{A_{t-1}^{-1}}  $, is related to the uncertainty of reward estimation.

In the following, to analyze the performance of LinUCB algorithm, we assume the same assumptions as in Sec.~\ref{se:Performance Analysis}.


\subsection{General Performance Bound}

Now, we analyze the performance of LinUCB algorithm. We note that the initial analysis effort of LinUCB\cite{Chu2011_LinUCB_analysis} presents analysis result for a modified version of LinUCB to satisfied the independent requirement by applying Azuma/Hoeffding inequality \cite{Chu2011_LinUCB_analysis}. 
As a result, we firstly provide the general performance analysis of LinUCB.
We have used analysis technique as in \cite{Abbasi2011}.
(Note that \cite{Abbasi2011} provides analysis for another algorithm instead of LinUCB, but its technique is helpful.)

Firstly, we note that since $A_t, b_t, \theta_t$ has the same definition as that in AdaLinUCB Algorithm, the previous Lemma~\ref{th:lemma_confidence_set}, Lemma~\ref{th:lemma_inequality_reward_estimate}, and Lemma~\ref{th:lemma_summation_xt} also hold here for LinUCB algorithm. Then, we have the following results.

\begin{theorem}\label{th:regret_LinUCB}
(The general regret bound of LinUCB).
For the LinUCB algorithm in Algo.~\ref{alg:LinUCBExtracted}, consider traditional contextual bandits with linear payoffs, the following results hold.
\begin{enumerate}[1)]
\item 	$\forall t \geq 1$, if $\alpha \geq \lVert
\theta_{t-1} - \theta_\star
 \rVert_{A_{t-1}}$, then the one-step regret (regarding nominal reward) satisfies, 
 \begin{align*}
 R_t \leq 2 \alpha \lVert 
 x_{t,a_t} \rVert_{A_{t-1}^{-1}}.
\end{align*}
\item $\forall \delta \in (0,1)$, with probability at least $1-\delta$, the accumulated $T$-slot regret (regarding nominal reward) satisfies,
\begin{align}\label{eq:LinUCB_Raccumulate_nominal}
\nonumber
& \mathbf{R}_\mathrm{total} (T)\! \leq\! \!\sqrt{8 T }
\Bigg[\! C_\mathrm{noise} \sqrt{
	d \log \!\left(
	\frac{1 \!\!+\!\!T C_\mathrm{context}^2 }{\delta}
	\right)
} +\! C_\mathrm{theta} \Bigg] \\
& ~~ \cdot \sqrt{d \log\left[ \frac{
		\mathrm{trace} (I_d) + T C_\mathrm{context}^2
	}{d} \right]
	-\log\det(I_d)}. 
\end{align}
\end{enumerate}
\end{theorem}
\begin{proof}
We begin by analyzing the one-step regret (regarding nominal reward) of LinUCB algorithm in Algo.~\ref{alg:LinUCBExtracted}. For $\forall t \geq 1$, with $\alpha \geq \lVert
\theta_{t-1} - \theta_\star
\rVert_{A_{t-1}}$, we have,
\begin{align*}
R_t & = \langle x_{t,a_t^\star}, \theta_\star \rangle - \langle x_{t,a_t}, \theta_\star \rangle \\
& \leq 
\langle x_{t,a_t^\star}, \theta_{t-1} \rangle 
+ \alpha \lVert x_{t,a_t^\star} \rVert_{A_{t-1}^{-1}}
- \langle x_{t,a_t}, \theta_\star \rangle \\
& \leq 
\langle x_{t,a_t}, \theta_{t-1} \rangle 
+ \alpha \lVert x_{t,a_t} \rVert_{A_{t-1}^{-1}}
- \langle x_{t,a_t}, \theta_\star \rangle
\\
& =
\langle x_{t,a_t}, \theta_{t-1} \rangle 
- \langle x_{t,a_t}, \theta_\star \rangle
+ \alpha \lVert x_{t,a_t} \rVert_{A_{t-1}^{-1}}
\\
& \leq 
\lVert
\theta_{t-1} - \theta_\star
\rVert_{A_{t-1}}  \lVert x_{t,a_t} \rVert_{A_{t-1}^{-1}} + \alpha \lVert x_{t,a_t} \rVert_{A_{t-1}^{-1}} \\
& \leq 2 \alpha \lVert x_{t,a_t} \rVert_{A_{t-1}^{-1}},
\end{align*}
where the inequality in the second line holds by Lemma~\ref{th:lemma_inequality_reward_estimate} and $\alpha \geq \lVert
\theta_{t-1} - \theta_\star
\rVert_{A_{t-1}}$; the inequality in the third line holds by the design of the LinUCB algorithm, specifically, by step~\ref{line:LinUCBExtracted_at_choice} of Algo.~\ref{alg:LinUCBExtracted}; the inequality in the fifth line holds by Lemma~\ref{th:lemma_inequality_reward_estimate}, and the last inequality holds by $\alpha \geq \lVert
\theta_{t-1} - \theta_\star
\rVert_{A_{t-1}}$.
As a result, the first statement is proved.

Now, we analyze the accumulated regret. Let, 
\begin{align*}
\alpha_T = C_\text{noise} \sqrt{
	d \log \left(
	\frac{1 + T C_\text{context}^2 }{\delta}
	\right)
} + C_\text{theta}.
\end{align*}
Then, by Lemma~\ref{th:lemma_confidence_set}, with probability at least $1-\delta$, for $\forall t \in[1,T]$,
$\alpha_T \geq \lVert
\theta_{t-1} - \theta_\star
\rVert_{A_{t-1}}$. As a result, with probability at least $1-\delta$,
\begin{align*}
\mathbf{R}_\text{total} (T)& = \sum_{t=1}^{T} R_t  \leq \sqrt{T \sum_{t=1}^{T} R_t^2} \\
& \leq \sqrt{T \cdot 4 \alpha_T^2 \sum_{t=1}^{T} \lVert x_{t,a_t} \rVert_{A_{t-1}^{-1}}^2} \\
& \leq \sqrt{8 T \alpha_T^2 } \\
& \cdot \sqrt{d \log\left[ \frac{
		\mathrm{trace} (I_d)\! +\! T C_\text{context}^2
	}{d} \right]
	-\log\det(I_d)},
\end{align*}
where the first inequality holds by Jensen's inequality; the second inequality holds by statement 1); the third inequality holds by Lemma~\ref{th:lemma_summation_xt}.
Thus, by substituting the value of $\alpha_T$, the inequality \eqref{eq:LinUCB_Raccumulate_nominal} holds.
\end{proof}

\subsection{Problem-Dependent Bound}
Now, we study the problem-dependent bound of LinUCB, and have the following results.

\begin{theorem}\label{th:LinUCB_problem_dependent_bound_0}
For the LinUCB algorithm in Algo.~\ref{alg:LinUCBExtracted}, consider traditional contextual bandit setting with linear payoffs, the accumulated $T$-slot regret (regarding nominal reward) satisfies,
\begin{align*}
\mathbf{R}_\mathrm{total}(T)
 \leq &  \frac{16 C_\mathrm{noise}^2 
C_\mathrm{theta}^2}{\Delta_{\min}} 
\Bigg\{
\log(C_\mathrm{context} T ) + 2\log \frac{1}{\delta} \\
& + 2(d-1) \log \left[
d \log\frac{d+TC_\mathrm{context}^2}{d}
+2\log\frac{1}{\delta}
\right] \\
&
+ (d-1) \log \frac{64 C_\mathrm{noise}^2 
	C_\mathrm{theta}^2 C_\mathrm{context}}{\Delta_{\min}^2}
\Bigg\}^2 .
\end{align*}
\end{theorem}
\begin{proof}
We note that,
\begin{align*}
\mathbf{R}_\text{total}(T) = \sum_{t=1}^{T} R_t \leq \sum_{t=1}^{T} \frac{R_t^2}{\Delta_{\min}},
\end{align*}
where the inequality holds since either $R_t=0$ or $\Delta_{\min}<=R_t$.
Then, the results follows from same proof as in \cite{Abbasi2011} (see the proof of Theorem 5 in \cite{Abbasi2011}). Note that the proof procedure uses Lemma~\ref{th:lemma_summation_xt} and the single optimal context condition.
\end{proof}

\subsection{Performance for Opportunistic Case - Proof of Theorem~\ref{th:LinUCB_problem_dependent_bound}}
Note that the arm selection strategy in LinUCB in Algo.~\ref{alg:LinUCBExtracted} is independent of the value of $L_t$.
Thus, when $L_t$ is i.i.d. over time, we have that $\tilde{\mathbf{R}}_\text{total} (T) = \bar{L} \mathbf{R}_\text{total} (T)$.
As a result, Theorem~\ref{th:LinUCB_problem_dependent_bound} directly follows from  Theorem~\ref{th:LinUCB_problem_dependent_bound_0}.

\subsection{Another Way to Apply LinUCB in Opportunistic Linar Bandits}\label{se:LinUCBMultiply}
Beside the LinUCBExtracted algorithm in Algo.~\ref{alg:LinUCBExtracted}, we also note that there is another way to directly apply in LinUCB in opportunistic contextual bandit environment.
Recall that the LinUCBExtracted algorithm in  Algo.~\ref{alg:LinUCBExtracted} is based on the linear relationship, $\mathbb{E}[r_{t,a}|x_{t,a}] = \langle x_{t,a}, \theta_\star \rangle$.
We can also apply the LinUCBMultiply algorithm in Algo.~\ref{alg:LinUCBMul}, which is based on the linear relationship,
$\mathbb{E}[L_t \cdot r_{t,a}|x_{t,a}, L_t] = \langle L_t \cdot x_{t,a}, \theta_\star \rangle$, i.e., regarding $L_t \cdot x_{t,a}$ as context vector.

\begin{algorithm}[tb]
	\caption{LinUCB(Multiply)}
	\label{alg:LinUCBMul} 
	\begin{algorithmic}[1]
		\STATE{Inputs:} {$\alpha \in \mathbb{R}_+$, $d \in \mathbb{N}$.}
		\STATE{$A \leftarrow \bm{I}_{d}$ \{The $d$-by-$d$ identity matrix\}
		}
		\STATE{$b \leftarrow \bm{0}_d$}
		\FOR{$t = 1, 2, 3,\cdots, T$}
		\STATE{ $\theta_{t\!-\!1} = A^{-1} b$}
		\STATE{Observe possible arm set $\mathcal{D}_t$, and observe associated context vectors $x_{t,a}, \forall a \in \mathcal{D}_t$.
		}
		\STATE{
		Observe $L_t$, and get $\tilde{x}_{t,a} = L_t \cdot x_{t,a}, \forall a\in \mathcal{D}_t$.
		}
		\FOR{$a \in \mathcal{D}_t$}
		\STATE{ $p_{t,a} \leftarrow \theta_{t\!-\!1}^\top \tilde{x}_{t,a} + \alpha \sqrt{\tilde{x}_{t,a}^\top A^{-1} \tilde{x}_{t,a} }$  \{Computes upper confidence bound\}
		}
		\ENDFOR
		\STATE{Choose action $a_t = \arg\max_{a\in \mathcal{D}_t} p_{t,a}$ with ties broken arbitrarily.}
		\STATE{Observe nominal reward $r_{t,a_t}$ and get actual reward $\tilde{r}_{t,a_t}
		= L_t \cdot r_{t,a_t}
		$.}
		\STATE{$A \leftarrow A + \tilde{x}_{t,a_t} \tilde{x}_{t,a_t}^\top $}
		\STATE{ $b \leftarrow b + \tilde{x}_{t, a_t} \tilde{r}_{t,a_t} $
		}
		\ENDFOR
	\end{algorithmic}
\end{algorithm} 
Thus, we have also implemented LinUCBMultiply in the numerical results. However, from the experiment results, LinUCBExtracted algorithm has a better performance than LinUCBMultiply.

\section{AdaLinUCB for Disjoint Model}\label{se:Appendix_Disjoint}

In above, we focus on the design and analysis of opportunistic contextual bandit for the joint model.
However, it should be noted that, the AdaLinUCB algorithm in Algo.~\ref{alg:AdaLinUCB}
can be modified slightly and then be applied to the disjoint model, which is shown in the Algo.~\ref{alg:disjoint_AdaLinUCB}.

\begin{algorithm}[tb]
	\caption{AdaLinUCB - Disjoint Model}
	\label{alg:disjoint_AdaLinUCB} 
	\begin{algorithmic}[1]
		\STATE{Inputs:} {$\alpha \in \mathbb{R}_+$, $d \in \mathbb{N}$, $l^{(+)}$, $l^{(-)}$.}
		\STATE{$A^{(a)} \leftarrow \bm{I}_{d},~\forall a$
		}
		\STATE{$b^{(a)} \leftarrow \bm{0}_d,~\forall a$}
		\FOR{$t = 1, 2, 3,\cdots, T$}
		\STATE{Observe possible arm set $\mathcal{D}_t$, and observe associated context vectors $x_{t,a}, \forall a \in \mathcal{D}_t$.
		}
		\STATE{Observe $L_t$ and calculate $\tilde{L}_t$ by \eqref{eq:L_tilde}.}
		\FOR{$a \in \mathcal{D}_t$}
		\STATE{ $\theta_{t-1}^{(a)} = [A^{(a)}]^{-1} b^{(a)}$}
		\STATE{ $p_{t,a} \leftarrow [\theta_{t-1}^{(a)}]^\top x_{t,a} + \alpha \sqrt{ (1-\tilde{L}_t)  x_{t,a}^\top [A^{(a)}]^{-1} x_{t,a} }$ 
		}
		\ENDFOR
		\STATE{Choose action $a_t = \arg\max_{a\in\mathcal{D}_t
		} p_{t,a}$ with ties broken arbitrarily.}
		\STATE{Observe nominal reward $r_{t,a_t}$.}
		\STATE{$A^{(a)} \leftarrow A^{(a)} + x_{t,a_t} x_{t,a_t}^\top $}
		\STATE{ $b^{(a)} \leftarrow b^{(a)} + x_{t, a_t} r_{t,a_t} $
		}
		\ENDFOR
	\end{algorithmic}
\end{algorithm} 

Here, we note that the joint model is the model introduced in Sec.~\ref{se:system model}:, which assumes that,
\begin{align*}
\mathbb{E}[r_{t,a}|x_{t,a}] = \langle x_{t,a}, \theta_\star \rangle,
\end{align*}
where $x_{t,a}$ is a context vector and $\theta_\star$ is the unknown coefficient vector.
Another model is the disjoint model, which assumes that,
\begin{align*}
\mathbb{E}[r_{t,a}|x_{t,a}] = \langle x_{t,a}, \theta_\star^{(a)} \rangle,
\end{align*}
where $x_{t,a}$ is a  context vector and $\theta_\star^{(a)}$ is the unknown coefficient vector for arm $a$. This model is called disjoint since the parameters are not shared among different arms.

The joint and disjoint models correspond to different models for linear contextual bandit problems,
as introduced in the seminal paper on LinUCB  \cite{Li2010_LinUCB}.

\section{More Numerical Results} \label{se:ap_simulations}
We have implemented AdaLinUCB (as in Algo.~\ref{alg:AdaLinUCB}), LinUCBExtracted (as in Algo.~\ref{alg:LinUCBExtracted}), and LinUCBMultiply (as in Algo.~\ref{alg:LinUCBMul}).
We have also implemented \textbf{E-AdaLinUCB} algorithm, which is an algorithm that
adjusts the threshold $l^{(+)}$ and  $l^{(-)}$ based on the empirical distribution of $L_t$. 
Specifically, the E-AdaLinUCB algorithm maintains the empirical histogram for the variation factors (or its moving average version for non-stationary cases), and selects $l^{(+)}$ and  $l^{(-)}$ accordingly.
Furthermore, the results for \textbf{KernelUCB} is shown in Appendix~\ref{se:ap_simu_kernel}.

\subsection{Synthetic Scenario with Binary-Valued variation Factor}
\label{se:ap_simu_binary}

Fig.~\ref{fig:app_binary_L} shows the performance of different algorithms with binary-valued variation factor for different value of $\rho$. From the simulation result, the AdaLinUCB algorithm significantly outperforms other algorithms for different values of $\rho$.

\begin{figure*}[thbp]
\begin{center}
\begin{minipage}[t]{\textwidth}
\begin{center}
\subfigure[$\rho = 0.1$]{\includegraphics[angle = 0,height = 0.23\linewidth,width = 0.31\linewidth]{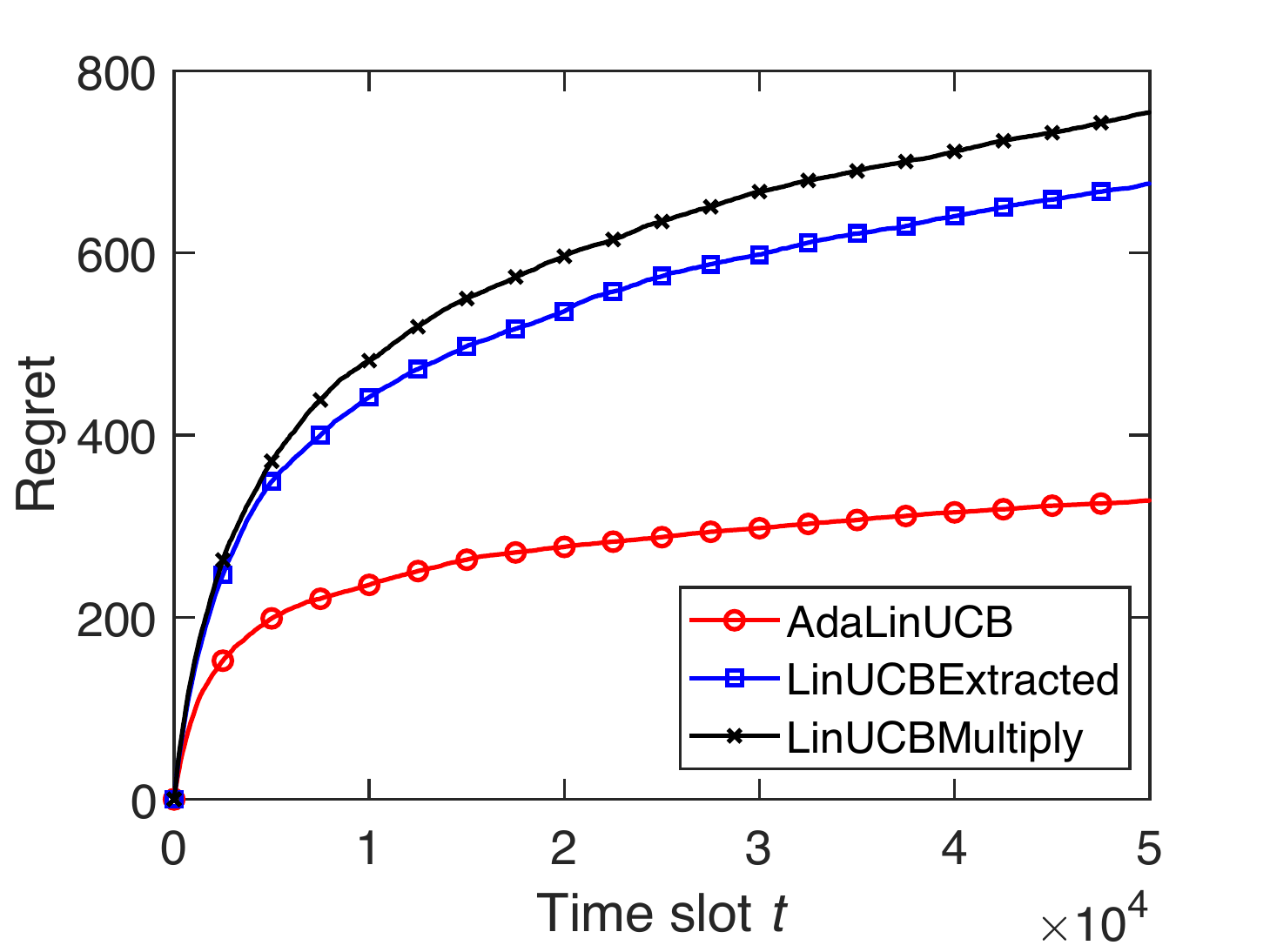}
\label{fig:app_binary_rhoBar_0p9}}
\subfigure[$\rho = 0.5$]{\includegraphics[angle = 0,height = 0.23\linewidth,width = 0.31\linewidth]{binary_rhoBar_0p5}
\label{fig:app_binary_rhoBar_0p5}}
\subfigure[$\rho = 0.9$]{\includegraphics[angle = 0,height = 0.23\linewidth,width = 0.31\linewidth]{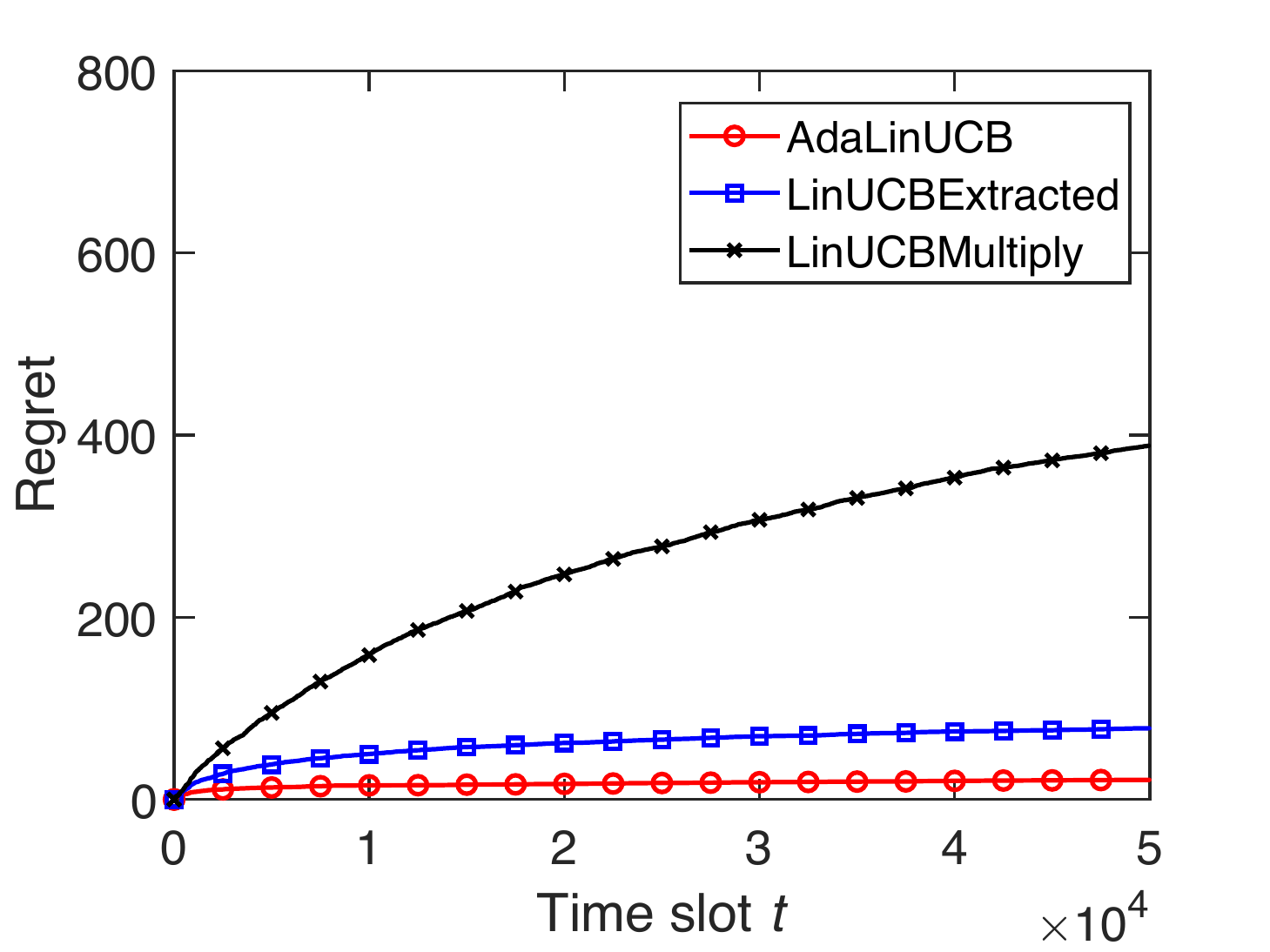}
\label{fig:app_binary_rhoBar_0p1}}
\vspace{-.25cm}
\caption{Regret under binary-valued variation factor.}
\label{fig:app_binary_L}
\end{center}
\end{minipage}
\end{center}
\end{figure*}

\subsection{Synthetic Scenario with Beta Distributed variation Factor}\label{se:ap_simu_beta}


Here, we define $l_\rho^{(-)}$ as the lower threshold such that $\mathbb{P}\{L_t\leq l_\rho^{(-)} =\rho \}$, and $l_\rho^{(+)}$ as the lower threshold such that $\mathbb{P}\{L_t\geq l_\rho^{(+)} =\rho \}$.
The simulation results demonstrate that, with appropriately chosen parameters, the proposed AdaLinUCB algorithm (and its empirical version E-AdaLinUCB) achieves good performance by leveraging the variation factor fluctuation in opportunistic contextual bandits. Furthermore, it turns out that, for a large range of $l^{(+)}$ and $l^{(-)}$ values, AdaLinUCB performs well. Meanwhile, E-AdaLinUCB has a similar performance as that of AdaLinUCB in different scenarios.

In Fig. \ref{fig:beta_load}, we implement both AdaLinUCB with a single threshold $l^{(-)}=l^{(+)}$, and
AdaLinUCB (and E-AdaLinUCB) with two different threshold values.
We find that AdaLinUCB and E-AdaLinUCB perform well for all these appropriate choices of $l^{(-)}$ and $l^{(+)}$. In addition, even in the special case with a single threshold $l^{(-)}=l^{(+)}$, AdaLinUCB has a better performance than other algorithms.

We evaluate the impact of $l^{(-)}$ and $l^{(+)}$ separately with the other one fixed in Fig. \ref{fig:beta_load_-} and Fig. \ref{fig:beta_load_+}, respectively.
Compared them, we can see that the impact of threshold values under continuous variation factor is insignificant (when $l^{(+)}$ and $l^{(-)}$ are changing in wide appropriate ranges),
and the regret of AdaLinUCB is significantly 
lower than that of  LinUCBExtracted and LinUCBMultiple.


\begin{figure*}[thbp]
\begin{center}

\begin{minipage}[t]{\textwidth}
\begin{center}

\subfigure[AdaLinUCB: $l^{(-)}=l_{0.05}^{(-)},l^{(+)}=l_{0.05}^{(+)}$; AdaLinUCB$(l^{(-)}=l^{(+)})$:$l^{(-)}=l^{(+)}=0.45$ ]{\includegraphics[angle = 0,height = 0.23\linewidth,width =0.31\linewidth]{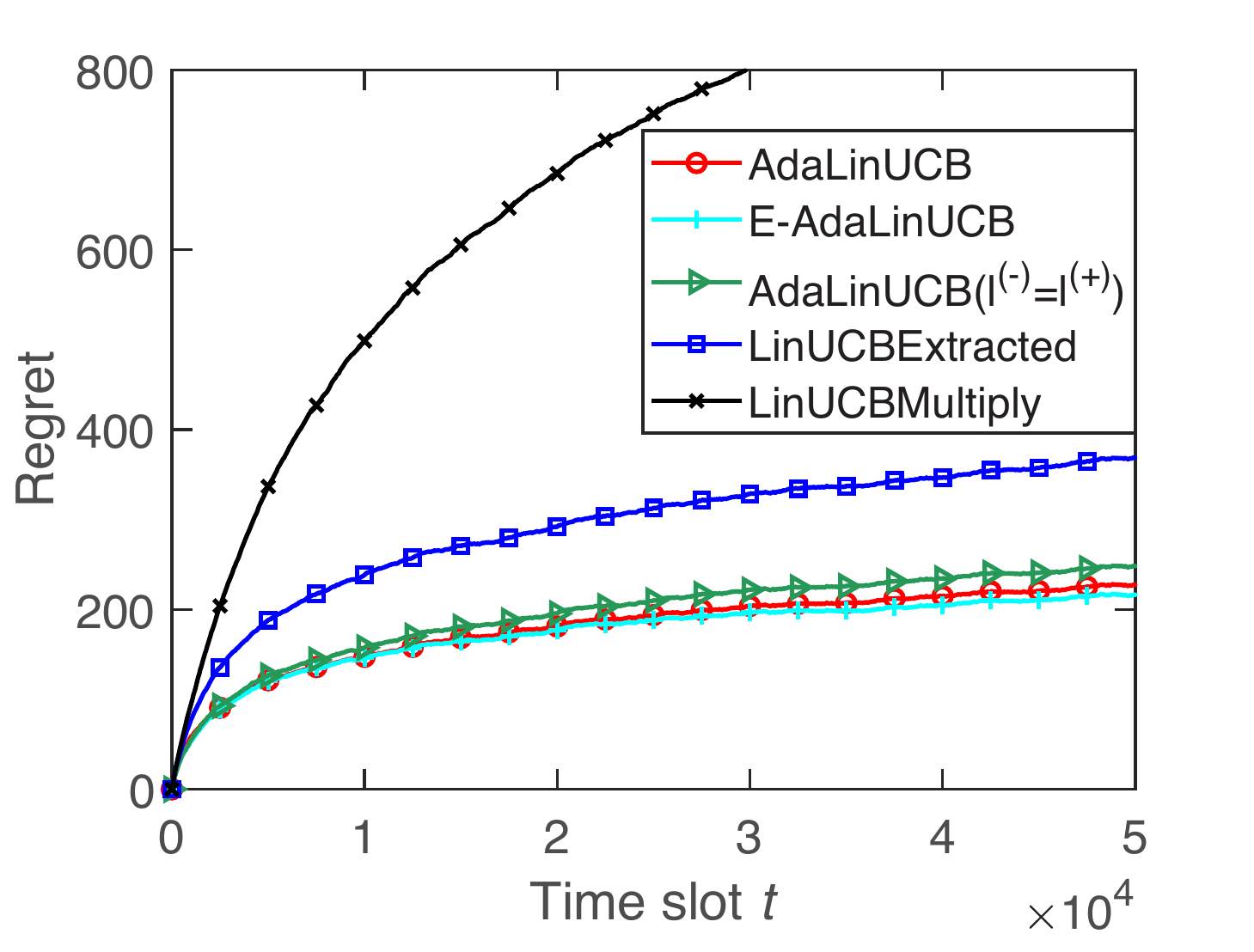}
\label{fig:beta_L05U05J45}}
\quad
\subfigure[AdaLinUCB: $l^{(-)}=l_{0}^{(-)},l^{(+)}=l_{0}^{(+)}$; AdaLinUCB$(l^{(-)}=l^{(+)})$: $l^{(-)}=l^{(+)}=0.5$ ]{\includegraphics[angle = 0,height = 0.23\linewidth,width = 0.31\linewidth]{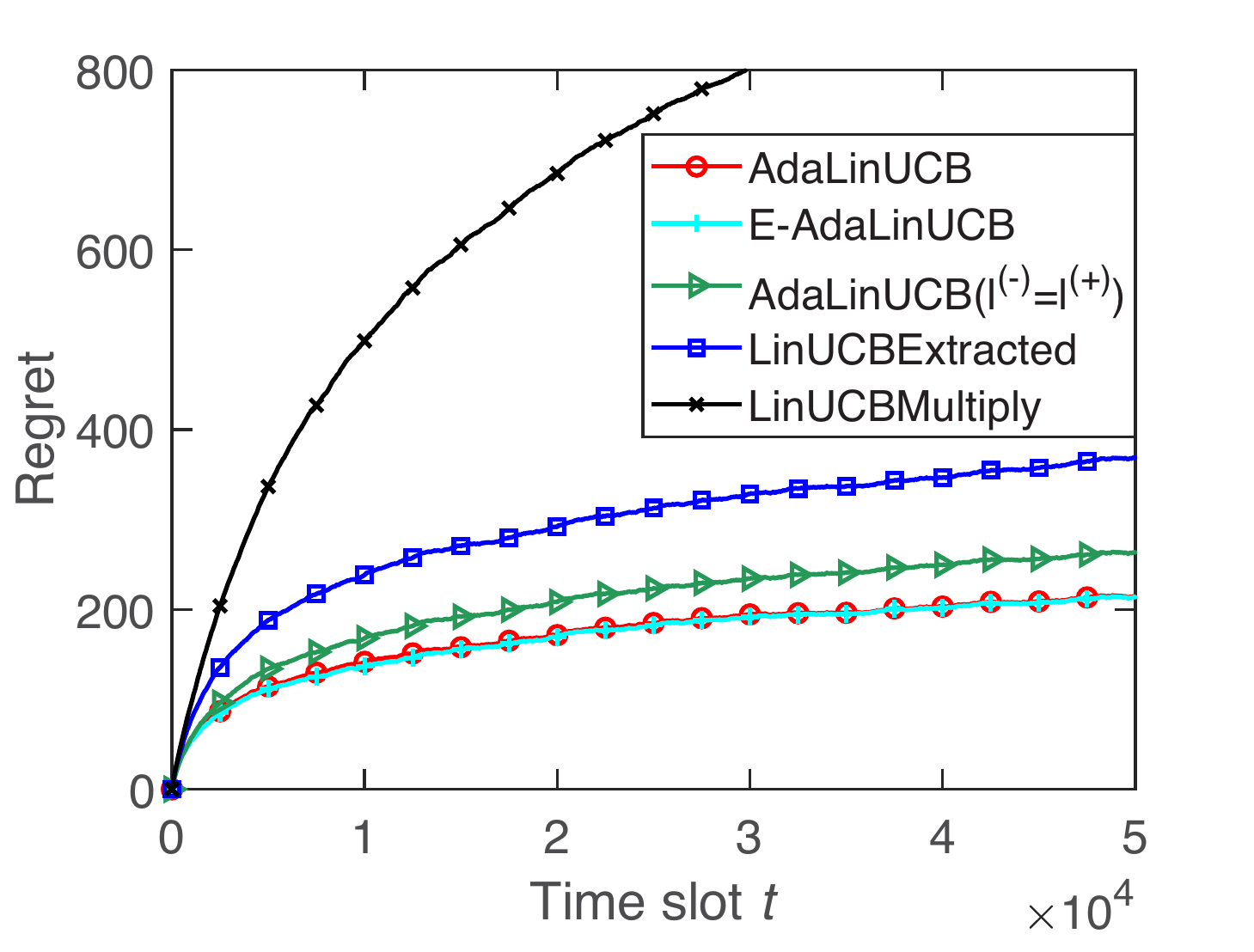}
\label{fig:beta_L00U00J50}}
\quad
\subfigure[AdaLinUCB: $l^{(-)}=l_{0.1}^{(-)},l^{(+)}=l_{0.1}^{(+)}$; AdaLinUCB$(l^{(-)}=l^{(+)})$: $l^{(-)}=l^{(+)}=0.55$]{\includegraphics[angle = 0,height = 0.23\linewidth,width = 0.31\linewidth]{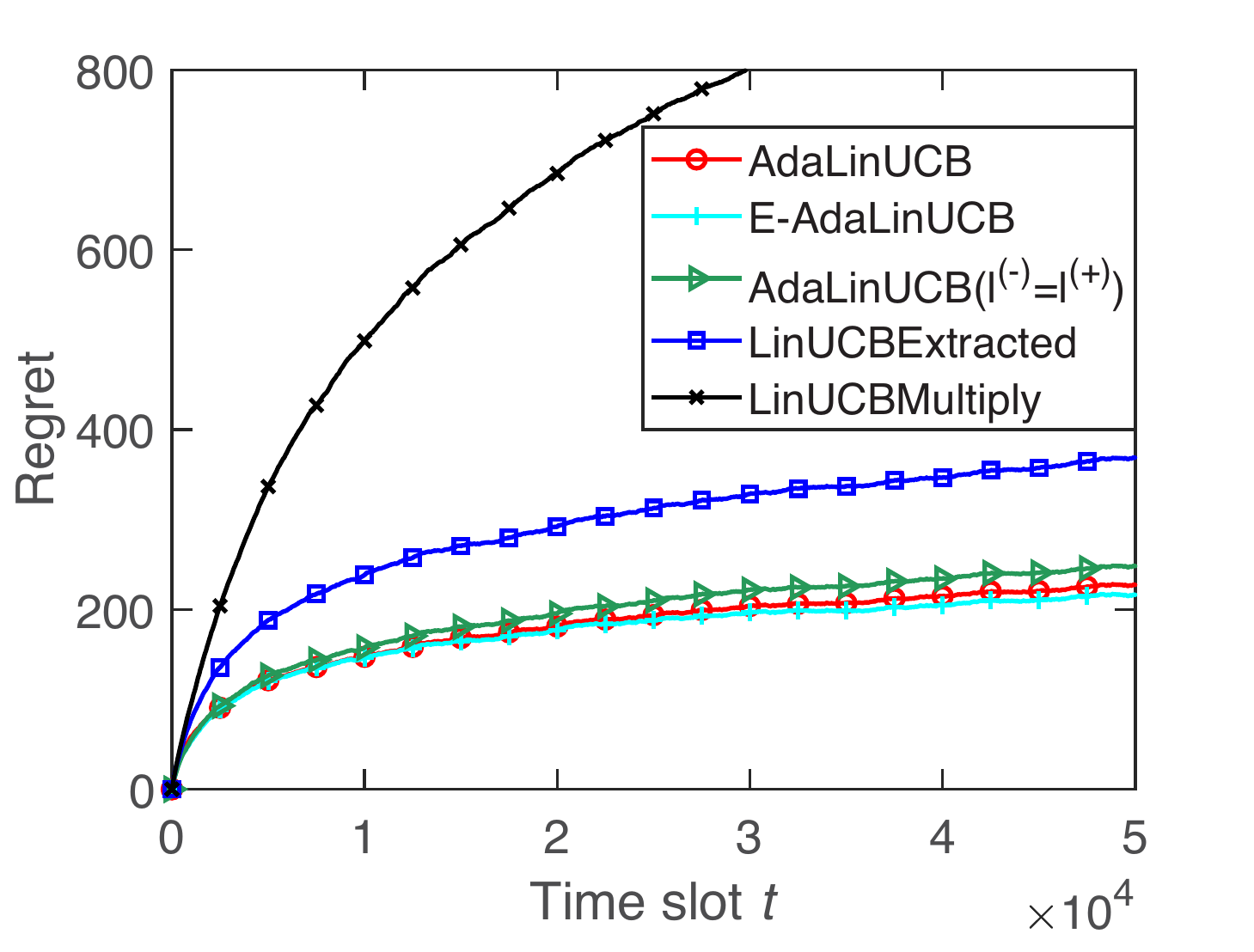}
\label{fig:beta_L10U10J55}}

\caption{Regret under beta distributed variation factor with a single threshold.}
\label{fig:beta_load}
\end{center}
\end{minipage}
\end{center}
\end{figure*}

\begin{figure*}[thbp]
\begin{center}

\begin{minipage}[t]{\textwidth}
\begin{center}
\subfigure[$l^{(-)}=l_{0.05}^{(-)},l^{(+)}=l_0^{(+)}$]{\includegraphics[angle = 0,height = 0.23\linewidth,width = 0.31\linewidth]{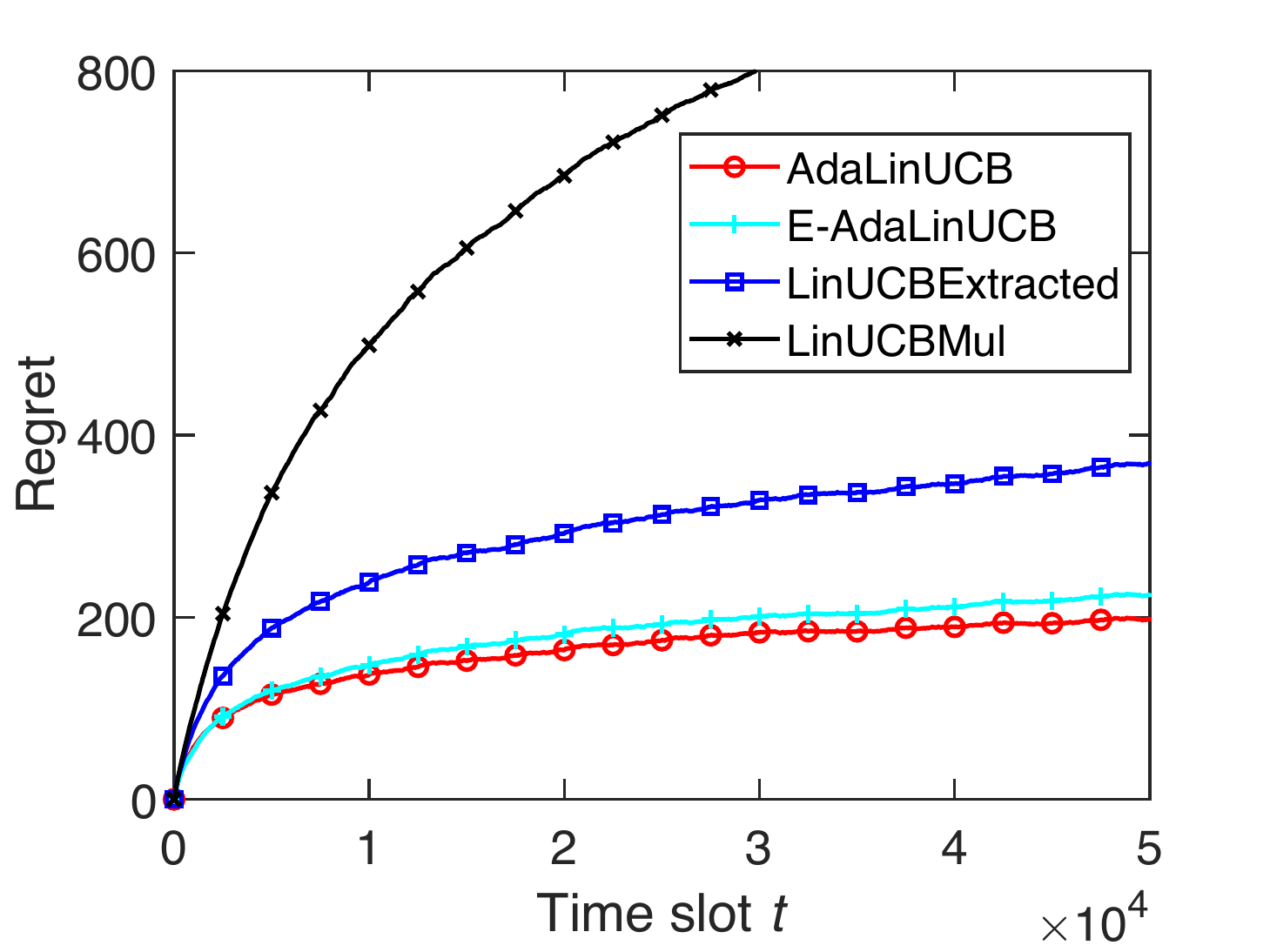}
\label{fig:beta_L05U00}}
\quad
\subfigure[$l^{(-)}=l_{0.1}^{(-)},l^{(+)}=l_0^{(+)}$]{\includegraphics[angle = 0,height = 0.23\linewidth,width = 0.31\linewidth]{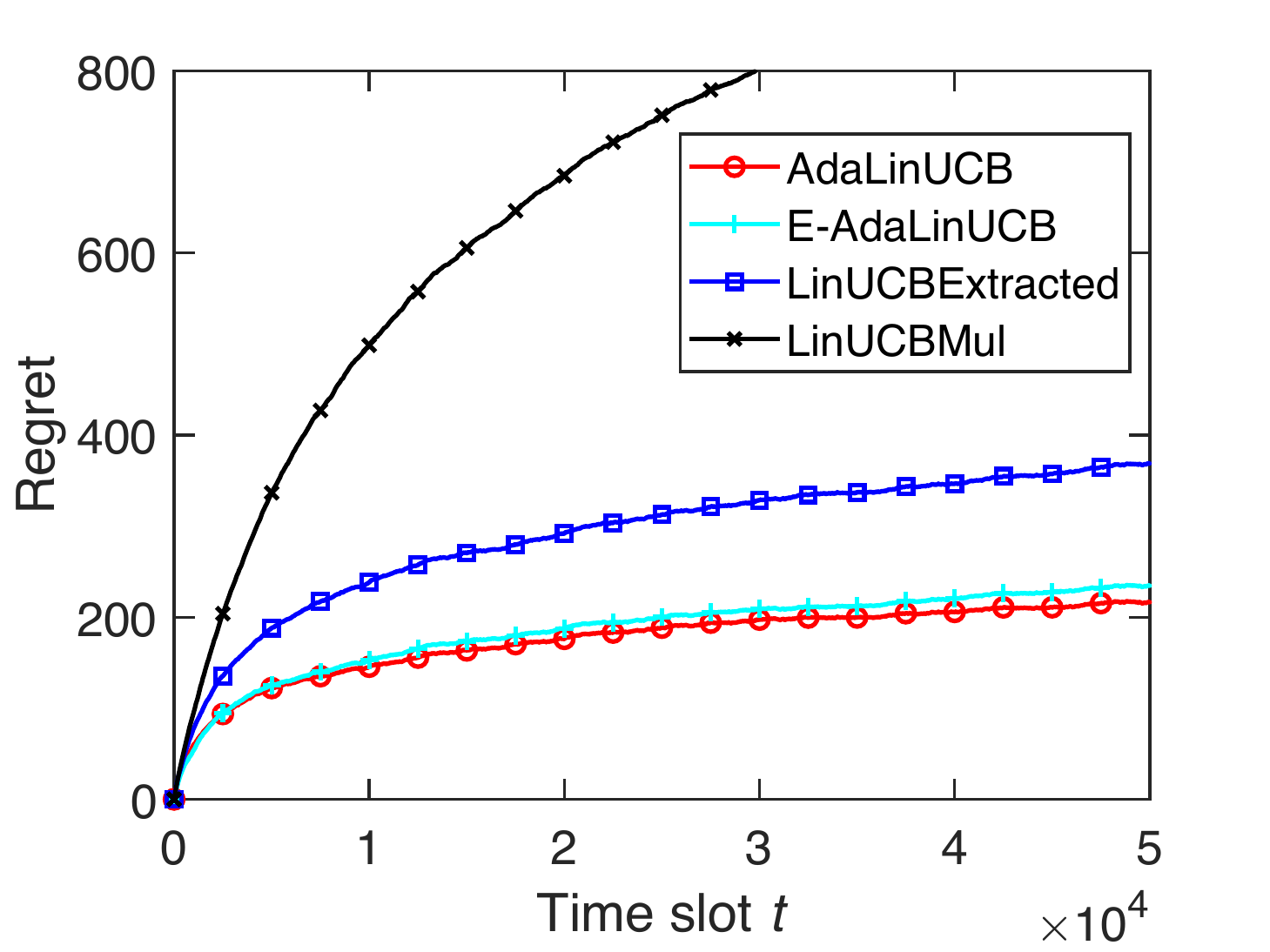}
\label{fig:beta_L10U00}}
\quad
\subfigure[$l^{(-)}=l_{0.15}^{(-)},l^{(+)}=l_0^{(+)}$]{\includegraphics[angle = 0,height = 0.23\linewidth,width = 0.31\linewidth]{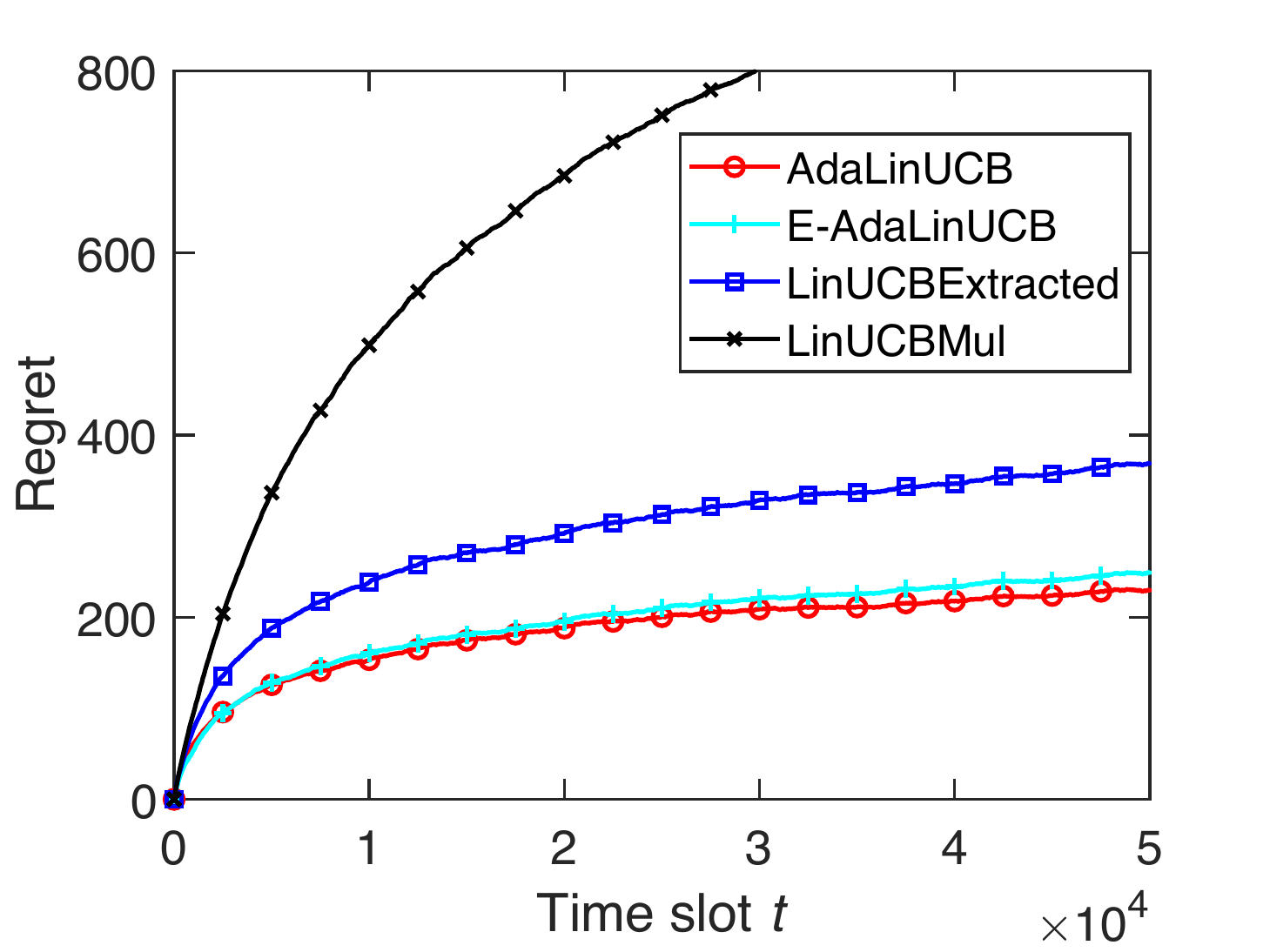}
\label{fig:beta_L15U00}}
\caption{ Regret under beta distributed variation factor with different values of $l^{(-)}$.}
\label{fig:beta_load_-}
\end{center}
\end{minipage}
\end{center}
\end{figure*}

\begin{figure*}[thbp]
\begin{center}
\begin{minipage}[t]{\textwidth}
\begin{center}
\subfigure[$l^{(-)}=l_{0}^{(-)},l^{(+)}=l_{0.05}^{(+)}$]{\includegraphics[angle = 0,height = 0.23\linewidth,width = 0.31\linewidth]{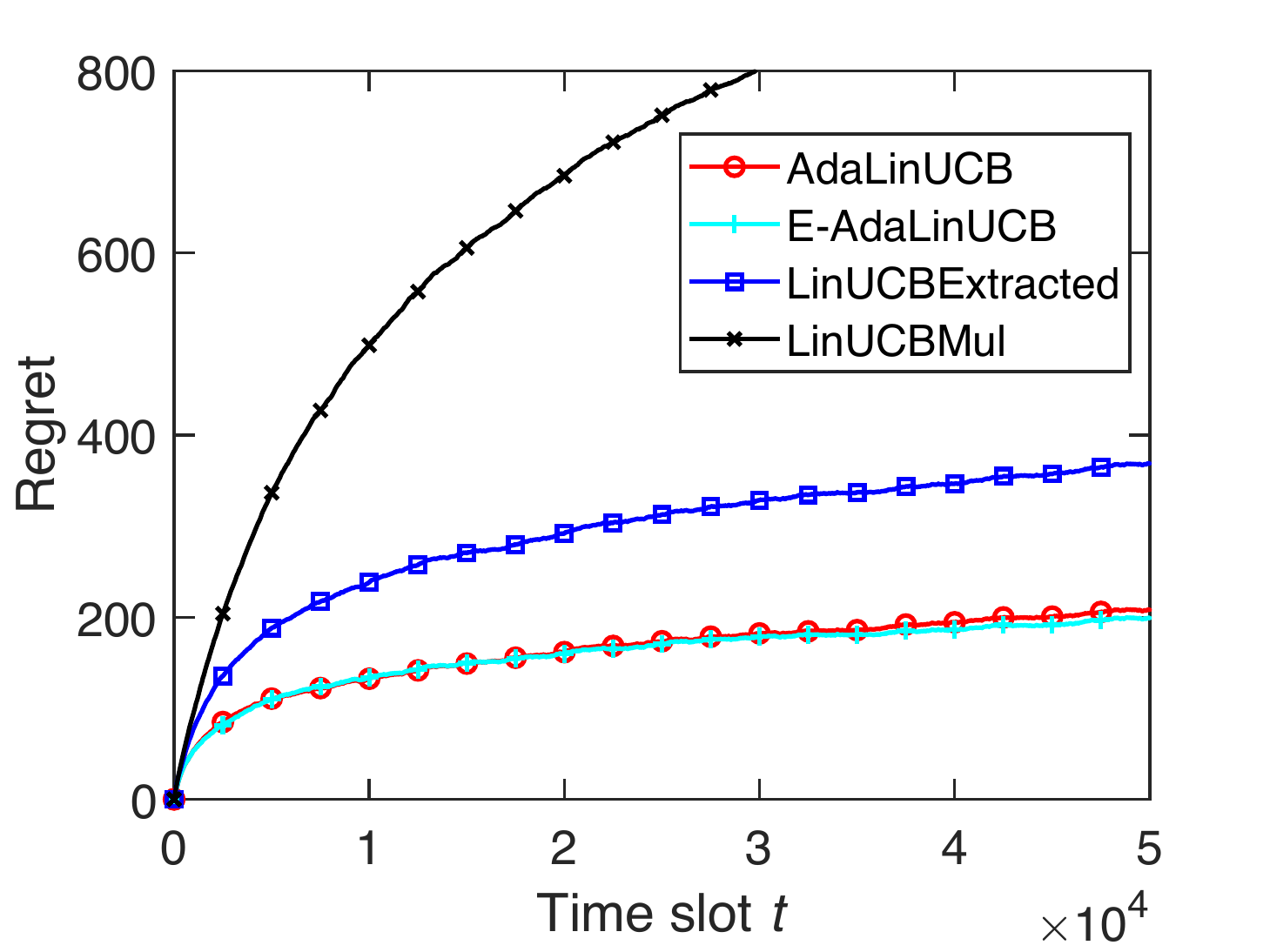}
\label{fig:beta_L00U05}}
\quad
\subfigure[$l^{(-)}=l_{0}^{(-)},l^{(+)}=l_{0.1}^{(+)}$]{\includegraphics[angle = 0,height = 0.23\linewidth,width = 0.31\linewidth]{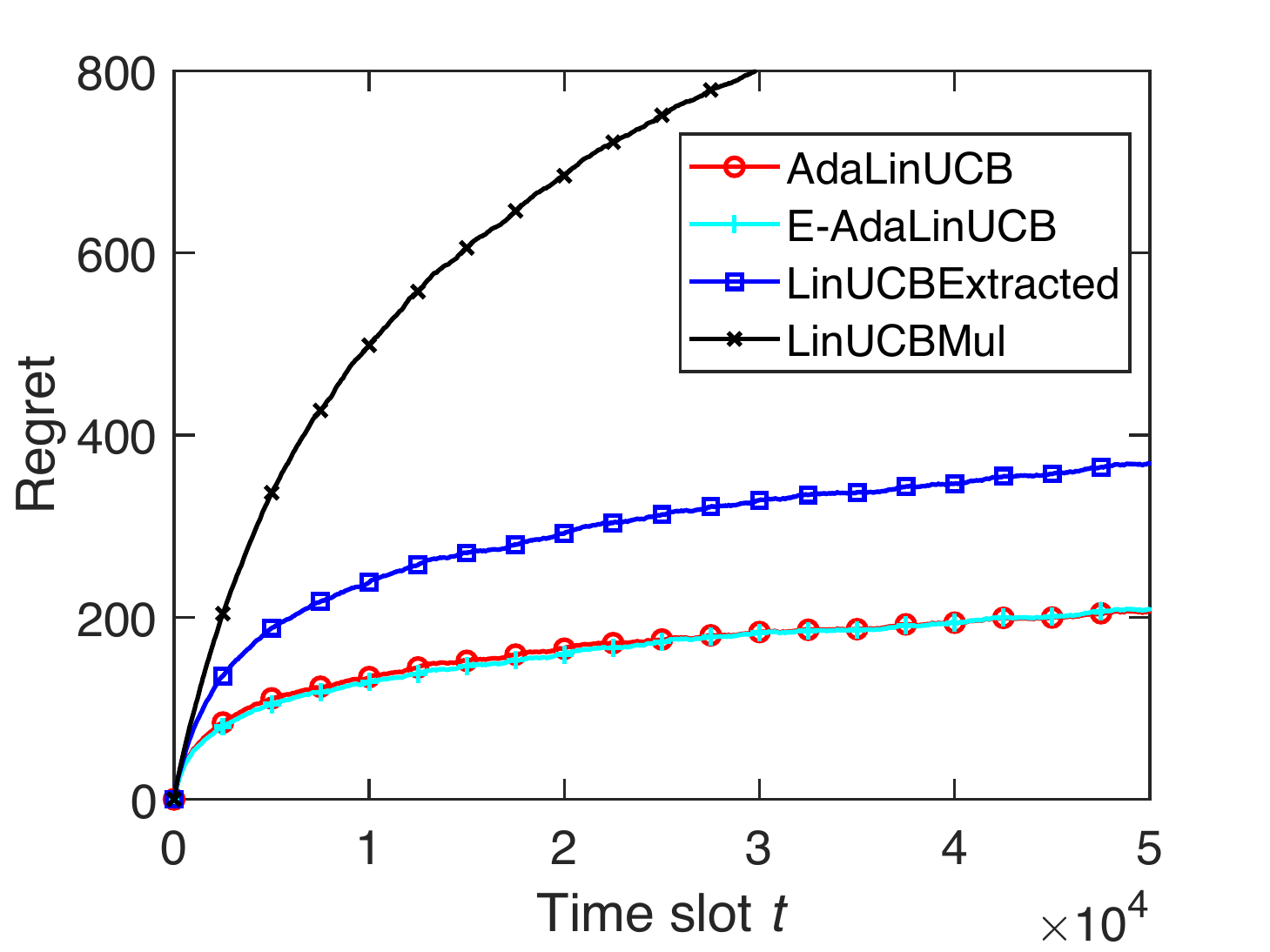}
\label{fig:beta_L00U10}}
\quad
\subfigure[$l^{(-)}=l_{0}^{(-)},l^{(+)}=l_{0.15}^{(+)}$]{\includegraphics[angle = 0,height = 0.23\linewidth,width = 0.31\linewidth]{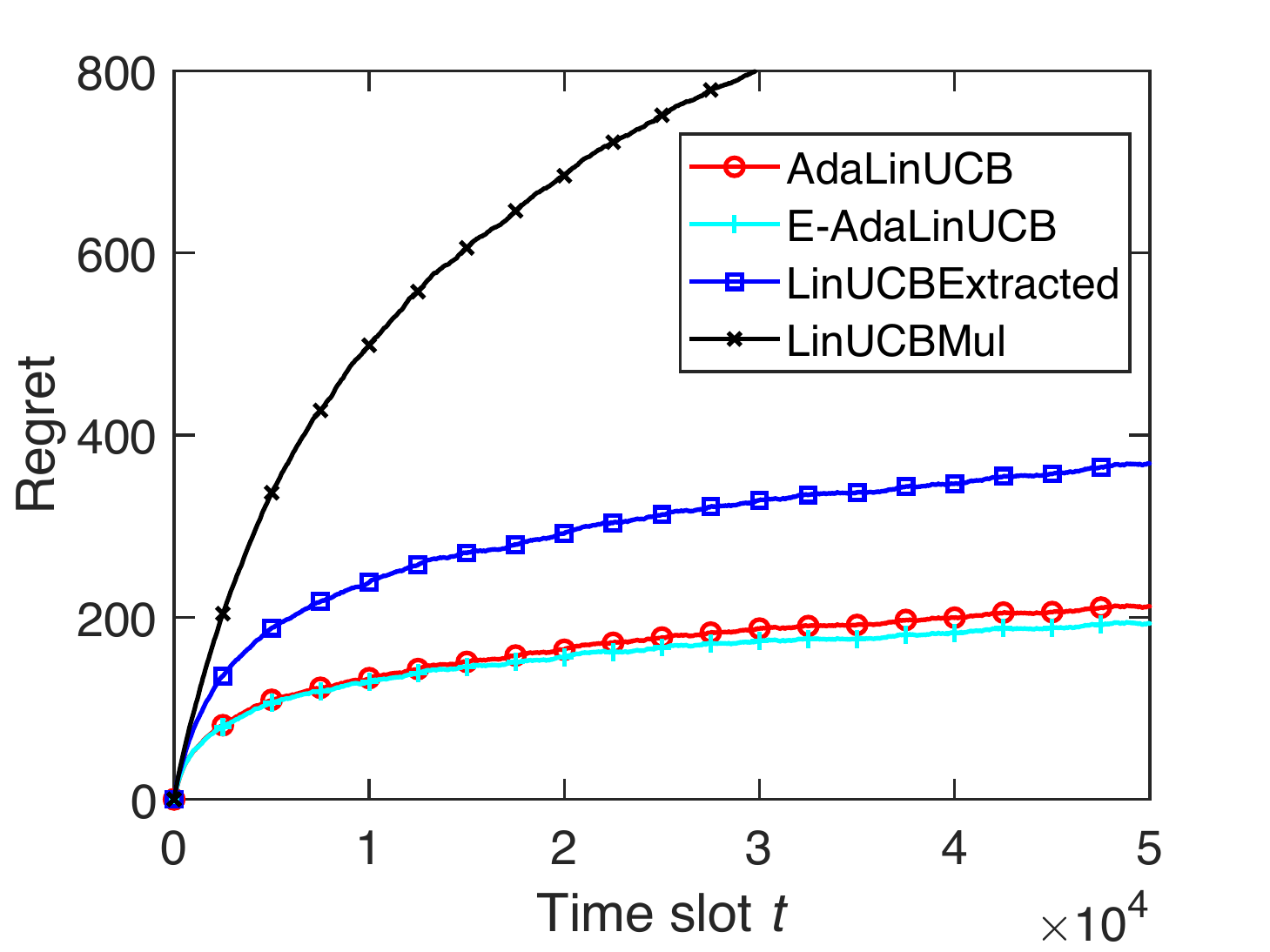}
\label{fig:beta_L00U15}}

\caption{Regret under beta distributed variation factor with different values of $l^{(+)}$.}
\label{fig:beta_load_+}
\end{center}
\end{minipage}

\end{center}
\end{figure*}

\subsection{Compare with KernelUCB}\label{se:ap_simu_kernel}

We have also implemented KernelUCB \cite{valko2013finite} which is a kernel-based upper confidence bound algorithm. It applies for general contextual bandits with non-linear payoffs. It can characterize general non-linear relationship between the context vector and reward based on the kernel that defines the similarity between two data points. There are many widely used kernels, such as Gaussian kernel, Laplacian kernel and polynomial kernel \cite{rasmussen2004gaussian}.

We demonstrate KernelUCB in 
Algo.~\ref{alg:KernelUCB1}.
The algorithm is based on paper \cite{valko2013finite}. Furthermore, in line $10$, we have actually used the technique of Schur complement \cite{zhang2006schur} to  update of kernel matrix $\bm{K}_t$ so as to boost the implementation of KernelUCB.


Fig.~\ref{fig:KernelUCB} demonstrates the performance of AdaLinUCB, LinUCBExtracted, and KernelUCB (with carefully selected hyper-parameters) for different scenarios. 
Note that the performance of KernelUCB highly depends on the choice of hyper-parameter. To make a fair comparison, we test the performance of KernelUCB for different hyper-parameter values, and chooses the hyper-parameters with the best performance (among the hyper-parameter values that we have experimented), i.e., $\Gamma_\text{kernel}=2$ for Gaussian kernel $k(z_1,z_2) = \exp(-\Gamma_\text{kernel}||z_1-z_2||^2)$,  $\lambda_\text{regularization}=0.5$ for kernel ridge regression.
As shown in Fig.~\ref{fig:KernelUCB}, AdaLinUCB outperforms KernelUCB under both binary-valued variation factor and continuous variation factor.

Besides the less competitive performance, as show in Fig.~\ref{fig:KernelUCB}, there are two other severe drawbacks that prevents the application of KernelUCB in many practical scenarios.
Firstly, its performance is highly sensitive to the choice of hyper-parameters.
As discussed above, we have  tested the performance of KernelUCB for different hyper-parameter values, and chooses the hyper-parameters with the best performance for a fair comparison.
However, even when the hyper-parameters just changes slightly (or environment such as variation factor fluctuation changes slightly),
the performance of KernelUCB can deteriorate severely such that it performs even worse then LinUCBExtracted.

Secondly, KernelUCB suffers from the high computational complexity problem.
Even if we have used the technique of Schur complement \cite{zhang2006schur} to  update of $\bm{K}_t$ so as to boost the implementation of KernelUCB as paper \cite{valko2013finite} , it still suffers from prohibitively high computational complexity even for moderately long time horizon. This is also the reason why Fig.~\ref{fig:KernelUCB} has a shorter time horizon than other figures.
Specifically, even to run a $10^4$-slot simulation, the time to run KernelUCB algorithm is at least $70$ times longer than the time to run AdaLinUCB algorithm.
In addition, when the time horizon is even larger, the time to run KernelUCB can be prohibitively long. 
This is because KernelUCB needs more computation with more existing data samples.
As a result, KernelUCB is not applicable for applications with large number of data samples in practice.

\begin{algorithm}[tb]
	\caption{KernelUCB}
	\label{alg:KernelUCB1} 
	\begin{algorithmic}[1]
		\STATE{Inputs:} {$\alpha \in \mathbb{R}_+$, $d \in \mathbb{N}$, $k(\cdot,\cdot)$, $\lambda=\lambda_\mathrm{regularization}$.}
	
		\FOR{$t = 1, 2, 3,\cdots, T$}
		\STATE{Observe possible arm set $\mathcal{D}_t$, and observe associated context vectors $x_{t,a}, \forall a \in \mathcal{D}_t$.}
		\STATE{
		Observe $L_t$, and get augment context $\tilde{x}_{t,a} = [L_t, x_{t,a}^\top]^\top, \forall a\in \mathcal{D}_t$.
		}
        \IF {$t=1$}
        \STATE{Choose the first actions $a_t \in \mathcal{D}_t$ (at start first action is pulled) }
        \ELSE 
        \FOR{$a \in \mathcal{D}_t$}
		\STATE{ 
$    k_{t,a} \leftarrow[ k(\tilde{x}_{t,a},\tilde{x}_{1,a_1}), k(\tilde{x}_{t,a},\tilde{x}_{2,a_2}),$  \\ 
    \quad \quad \quad \quad $ \cdots,k(\tilde{x}_{t,a},\tilde{x}_{{t-1},a_{t-1}})]^\top$		
    }
		\STATE { $\bm{K}_{t} \leftarrow $
		kernel matrix of   $(\tilde{x}_{1,a_1},\cdots,\tilde{x}_{t-1,a_{t-1}})$}
		\STATE{ $p_{t,a} \leftarrow k_{t,a}^\top [\bm{K}_{t}+ \lambda \bm{I}]^{-1}y_{t-1} $  \\
		~~~~$+ \alpha  \sqrt{k(\tilde{x}_{a,t},\tilde{x}_{a,t})-k_{t,a}^T [\bm{K}_{t} \! +\! \lambda \bm{I}]^{-1}k_{t,a}}$ 
		}
		\ENDFOR
		\STATE{Choose action $a_t = \arg\max_{a\in\mathcal{D}_t
		} p_{t,a}$ with ties broken arbitrarily.}
        \ENDIF
        \STATE{Observe nominal reward $r_{t,a_t}$.}
		\STATE $y_t \leftarrow [r_{1,a_1},r_{2,a_2},\cdots,r_{t,a_t}]^\top$
        \ENDFOR
	\end{algorithmic}
\end{algorithm}

\begin{figure*}[thbp]
\begin{center}
\begin{minipage}[t]{\textwidth}
\begin{center}
\subfigure[Binary-valued $L_t$ with $\rho = 0.9$. ($\epsilon_0=\epsilon_1=0$.)]{\includegraphics[angle = 0,height = 0.23\linewidth,width = 0.31\linewidth]{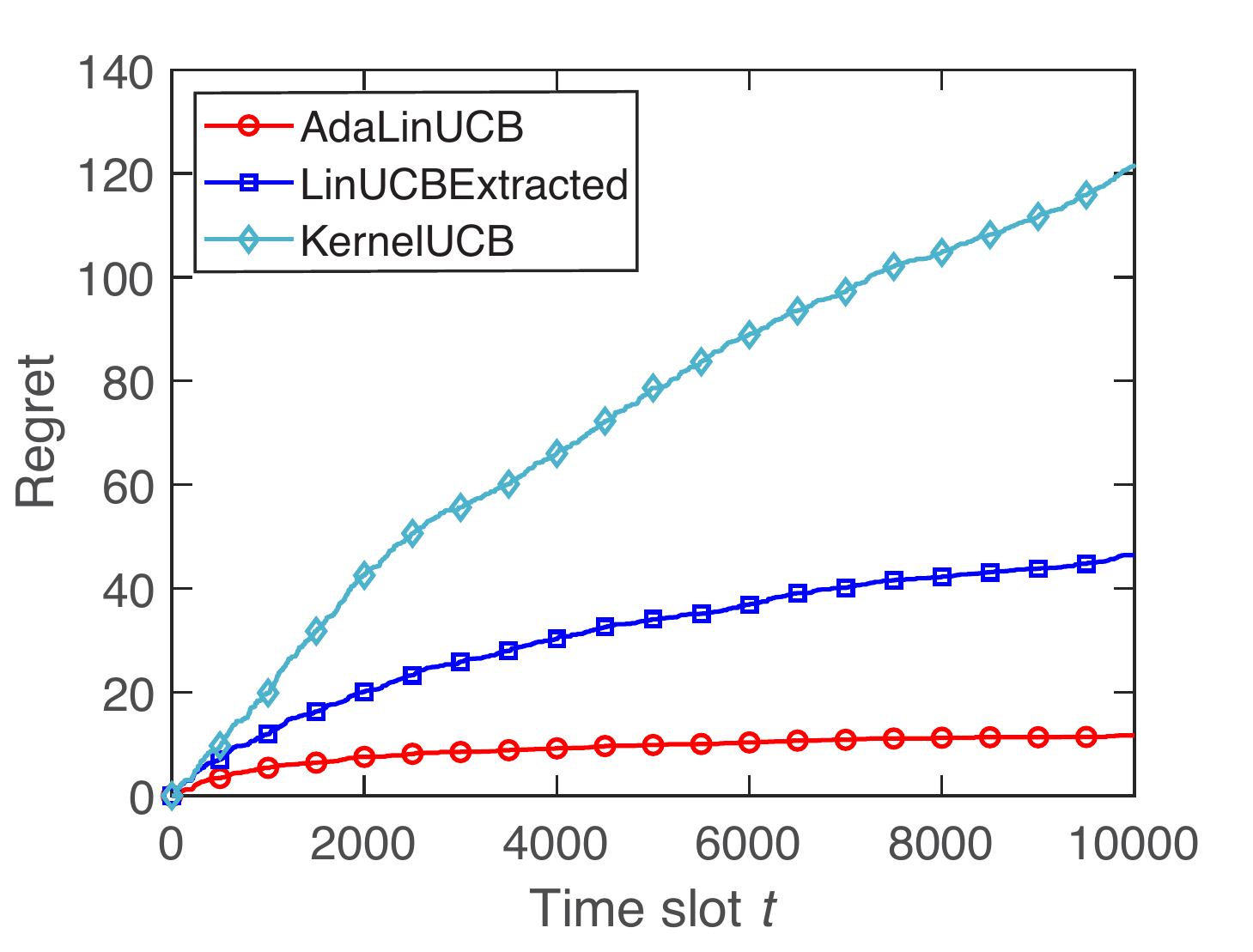}
\label{fig:app_k_binary_rho_0p9}}
\quad
\subfigure[Binary-valued $L_t$ with $\rho = 0.5$. ($\epsilon_0=\epsilon_1=0$.)]{\includegraphics[angle = 0,height = 0.23\linewidth,width = 0.31\linewidth]{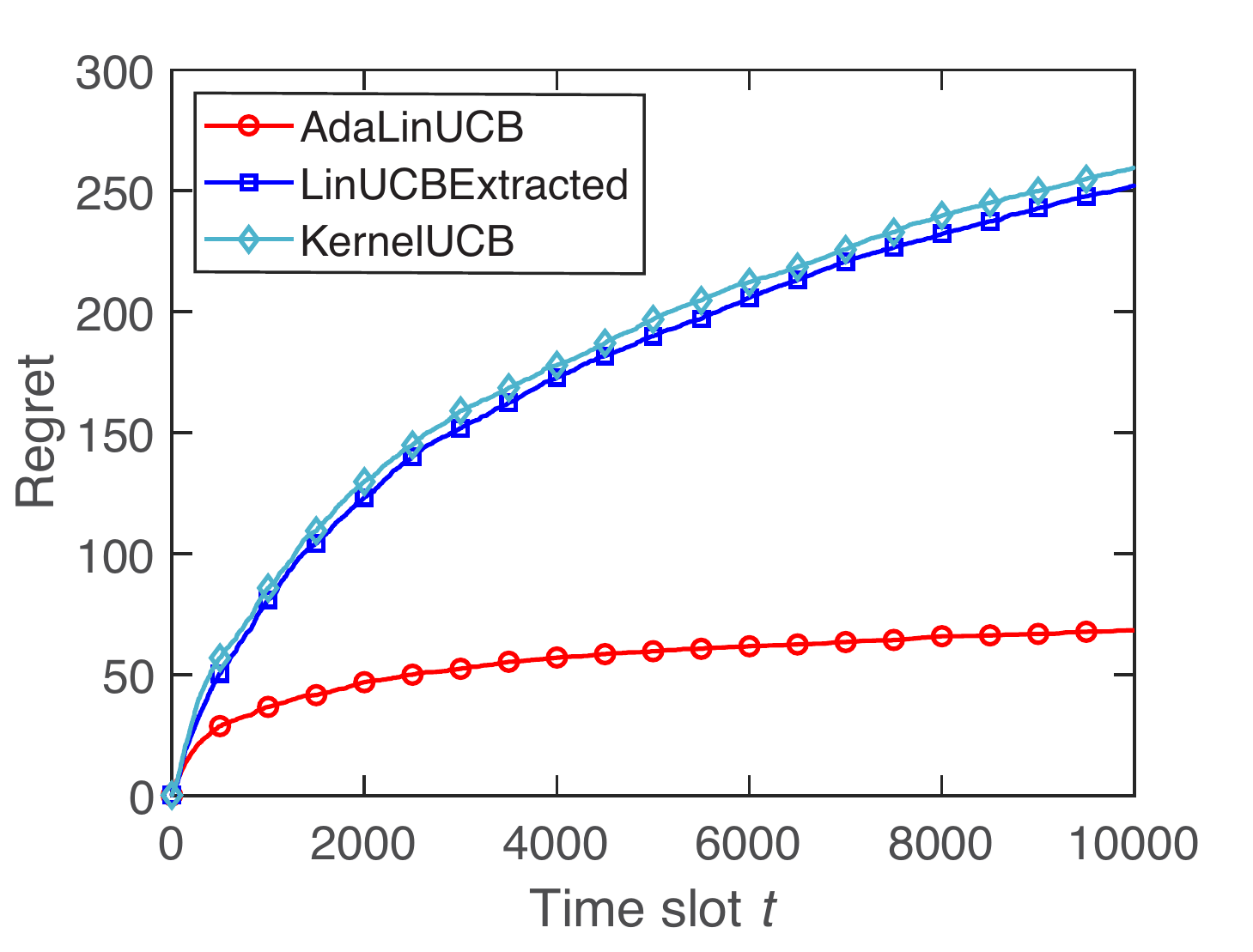}
\label{fig:app_k_binary_rho_0p5}}
\quad
\subfigure[Binary-valued $L_t$ with $\rho = 0.1$. ($\epsilon_0=\epsilon_1=0$.)]{\includegraphics[angle = 0,height = 0.23\linewidth,width = 0.31\linewidth]{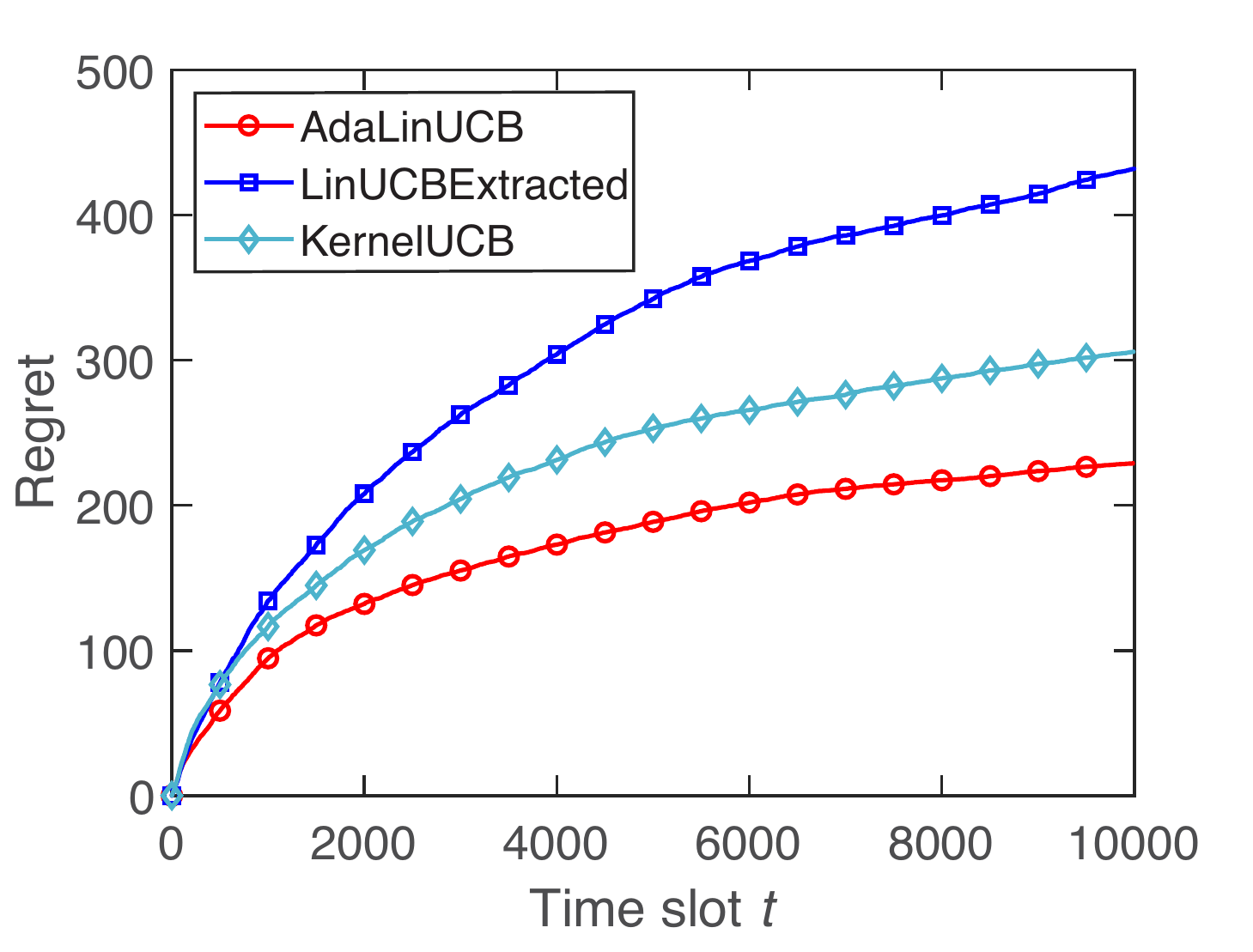}
\label{fig:app_k_binary_rho_0p1}}
\quad
\subfigure[Beta distributed variation factor; AdaLinUCB with $l^{(-)}=0, l^{(+)}=l_{0}^{(+)}$.]{\includegraphics[angle = 0,height = 0.23\linewidth,width = 0.31\linewidth]{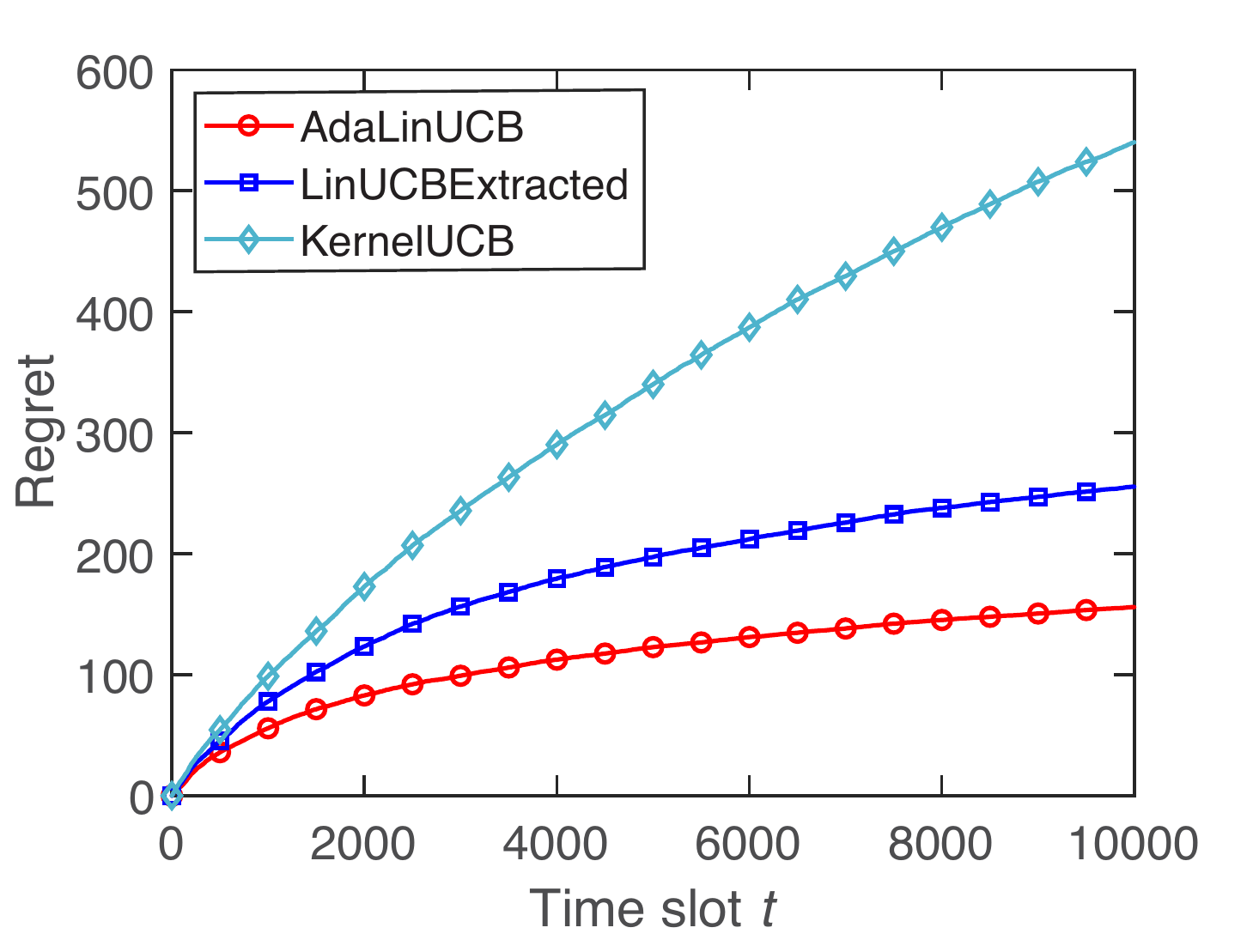}
\label{fig:app_k_beta}}
\vspace{-.25cm}
\caption{Performance Comparison with KernelUCB.}
\label{fig:KernelUCB}
\end{center}
\end{minipage}
\end{center}
\end{figure*}

\subsection{More for experiments on Yahoo! Today Module}\label{se:ap_simu_real}
We also test the performance of the algorithms using the data from Yahoo! Today Module. This dataset contains over $4$ million user visits to the Today module in a ten-day period in May 2009
\cite{Li2010_LinUCB}.
To evaluate contextual bandits using offline data, the experiment uses the unbiased offline evaluation protocol proposed in \cite{Li2011_unbiased_LinUCB_simu}.

In Yahoo! Today Module, for each user visit, there are $10$ candidate articles to be selected. The candidate articles are updated in a timely manner and are different for different time slots. Further, both the user and each of the candidate articles are associated with a $6$-dimensional feature vector, which are generated by a conjoint analysis with a bilinear model \cite{Chu2009_LinUCB_feature}.

\begin{figure*}[thbp]
\begin{center}
\begin{minipage}[t]{\textwidth}
\begin{center}
\subfigure[$l^{(-)}=l_0^{(-)},l^{(+)}=l_{0.1}^{(+)}$]{\includegraphics[angle = 0,height = 0.23\linewidth,width = 0.31\linewidth]{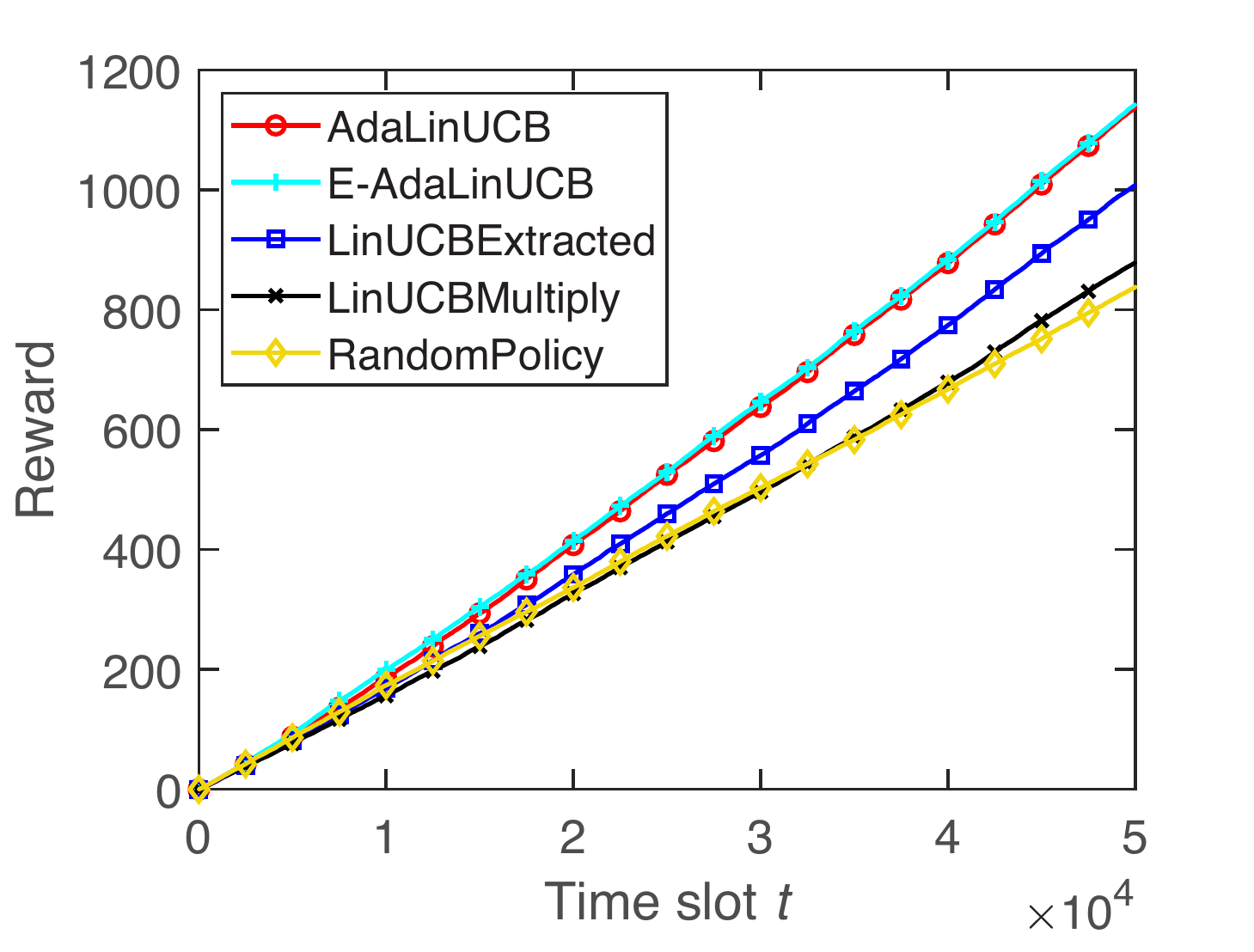}
\label{fig:real_L0U1}}
\quad
\subfigure[$l^{(-)}=l_0^{(-)},l^{(+)}=l_{0.2}^{(+)}$]{\includegraphics[angle = 0,height = 0.23\linewidth,width = 0.31\linewidth]{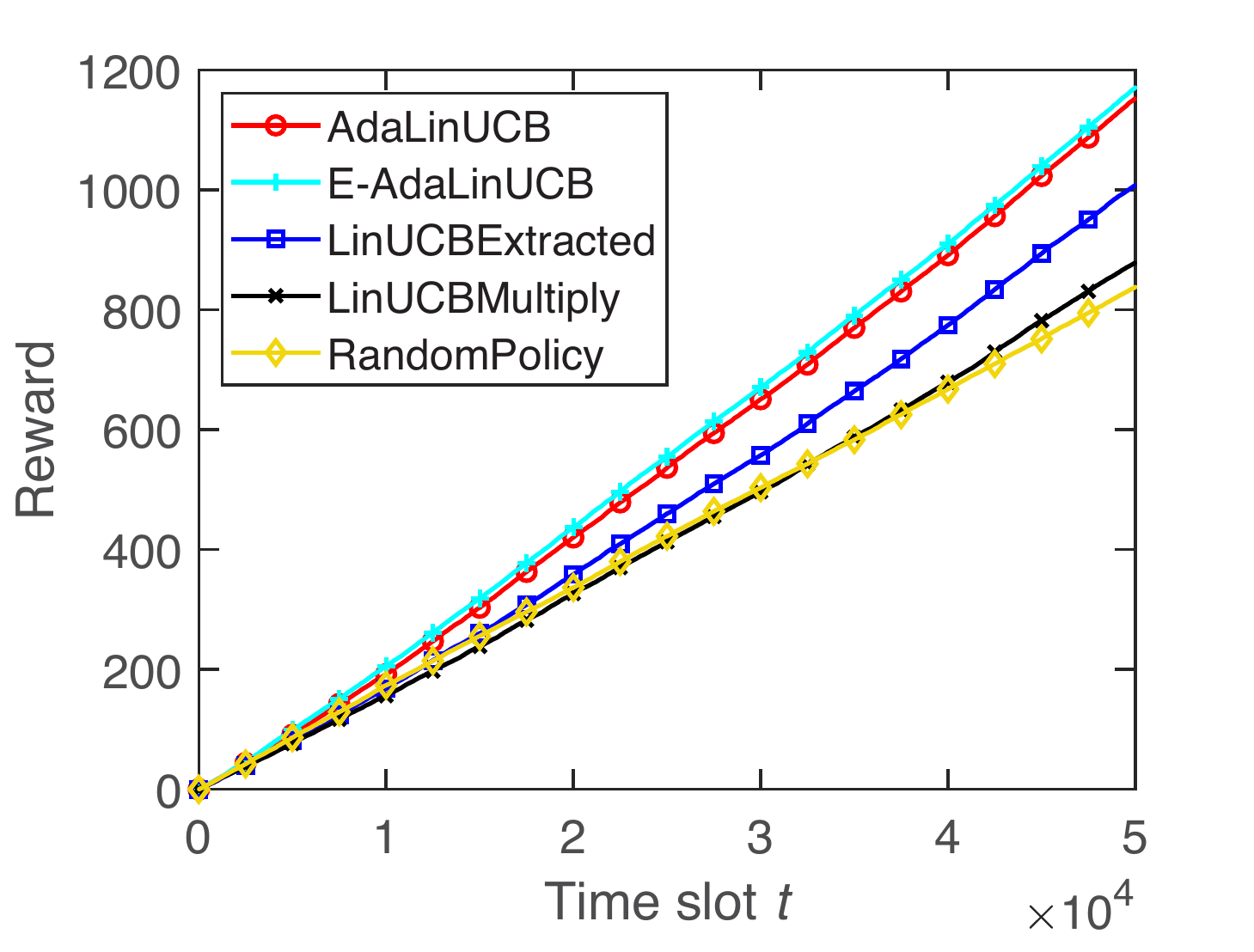}
\label{fig:real_L0U2}}
\quad
\subfigure[$l^{(-)}=l_0^{(-)},l^{(+)}=l_{0.3}^{(+)}$]{\includegraphics[angle = 0,height = 0.23\linewidth,width = 0.31\linewidth]{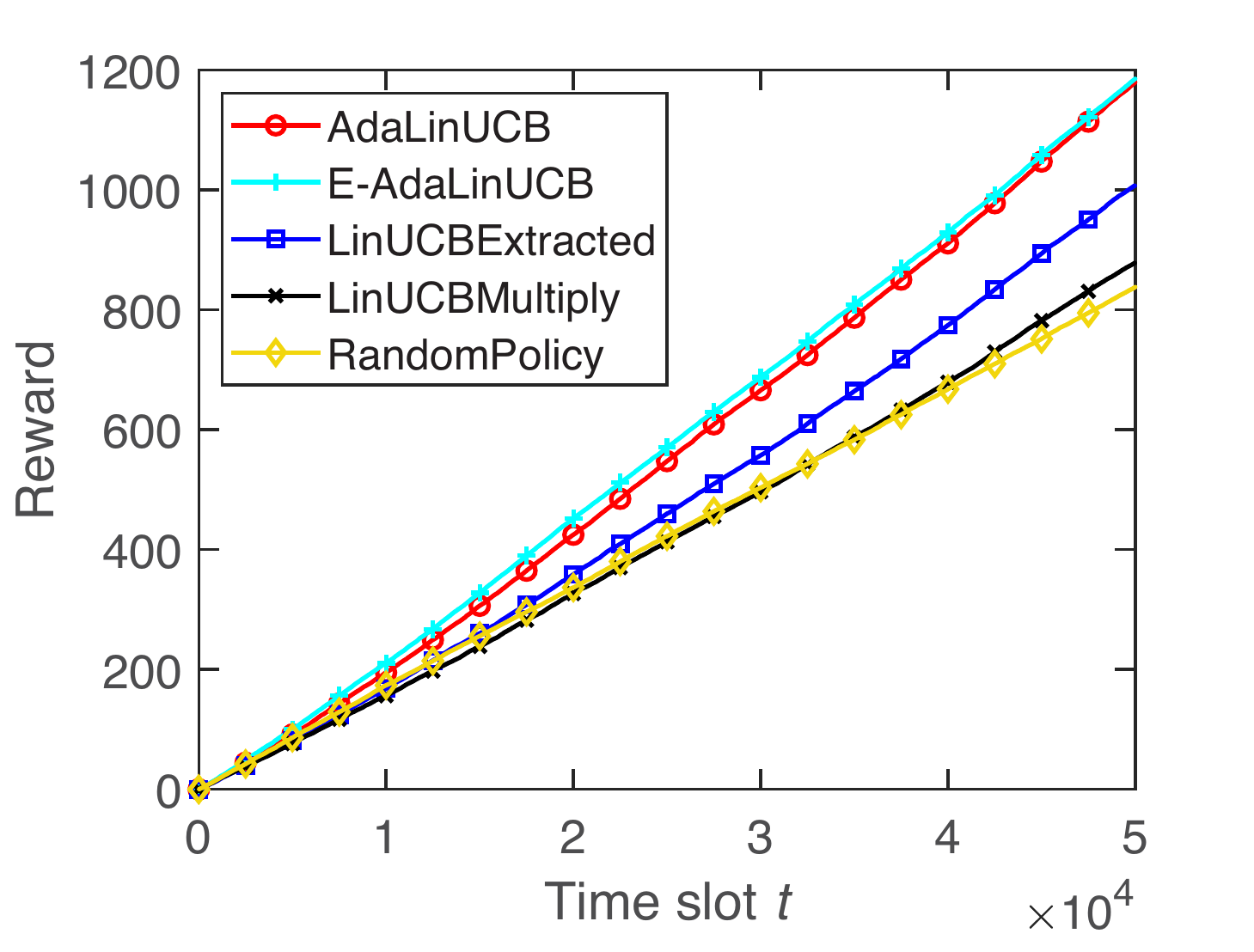}
\label{fig:real_L0U3}}
\quad
\vspace{-.25cm}
\subfigure[$l^{(-)}=l_{0.1}^{(-)},l^{(+)}=l_0^{(+)}$]{\includegraphics[angle = 0,height = 0.23\linewidth,width = 0.31\linewidth]{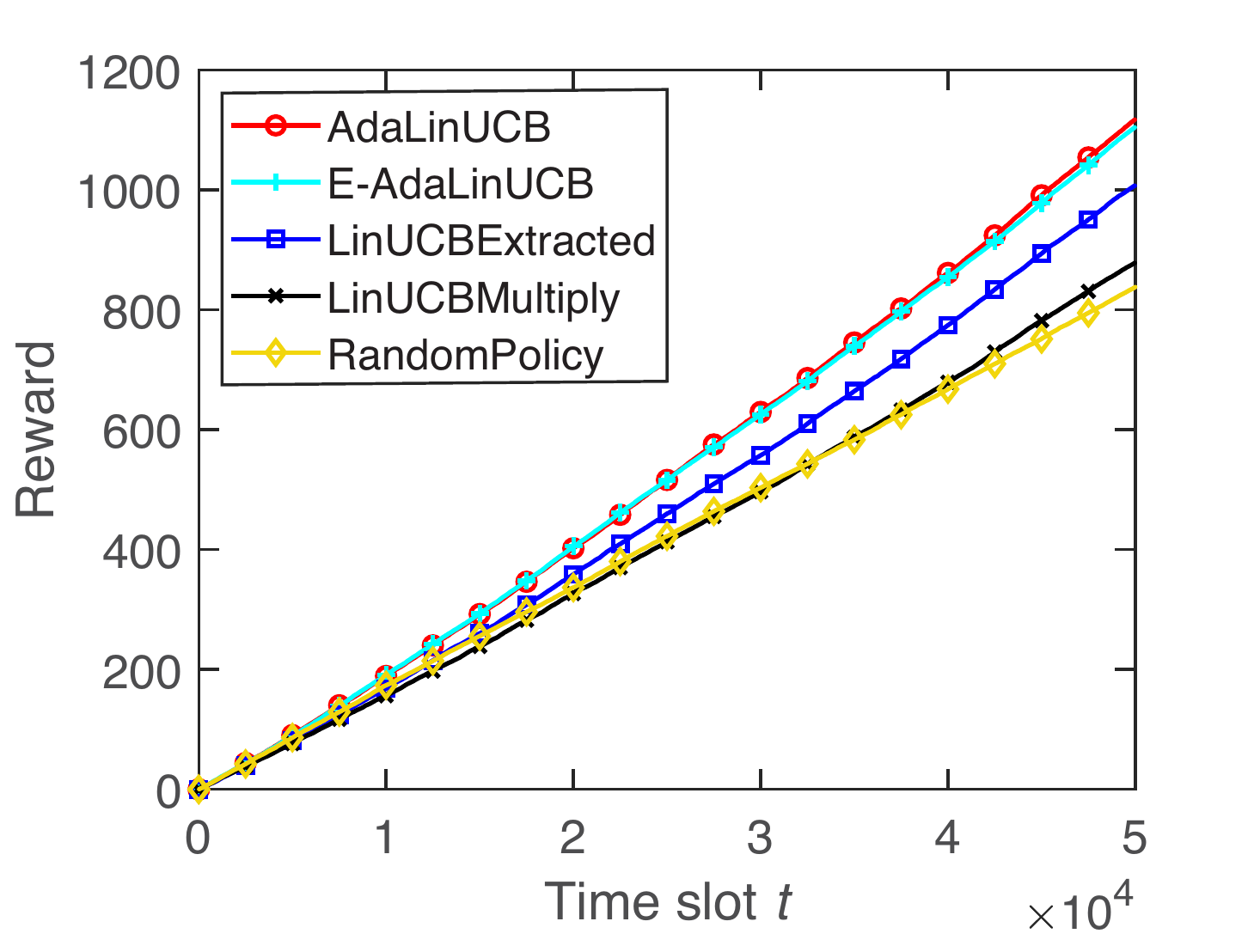}
\label{fig:real_L1U0}}
\quad
\subfigure[$l^{(-)}=l_{0.2}^{(-)},l^{(+)}=l_{0}^{(+)}$]{\includegraphics[angle = 0,height = 0.23\linewidth,width = 0.31\linewidth]{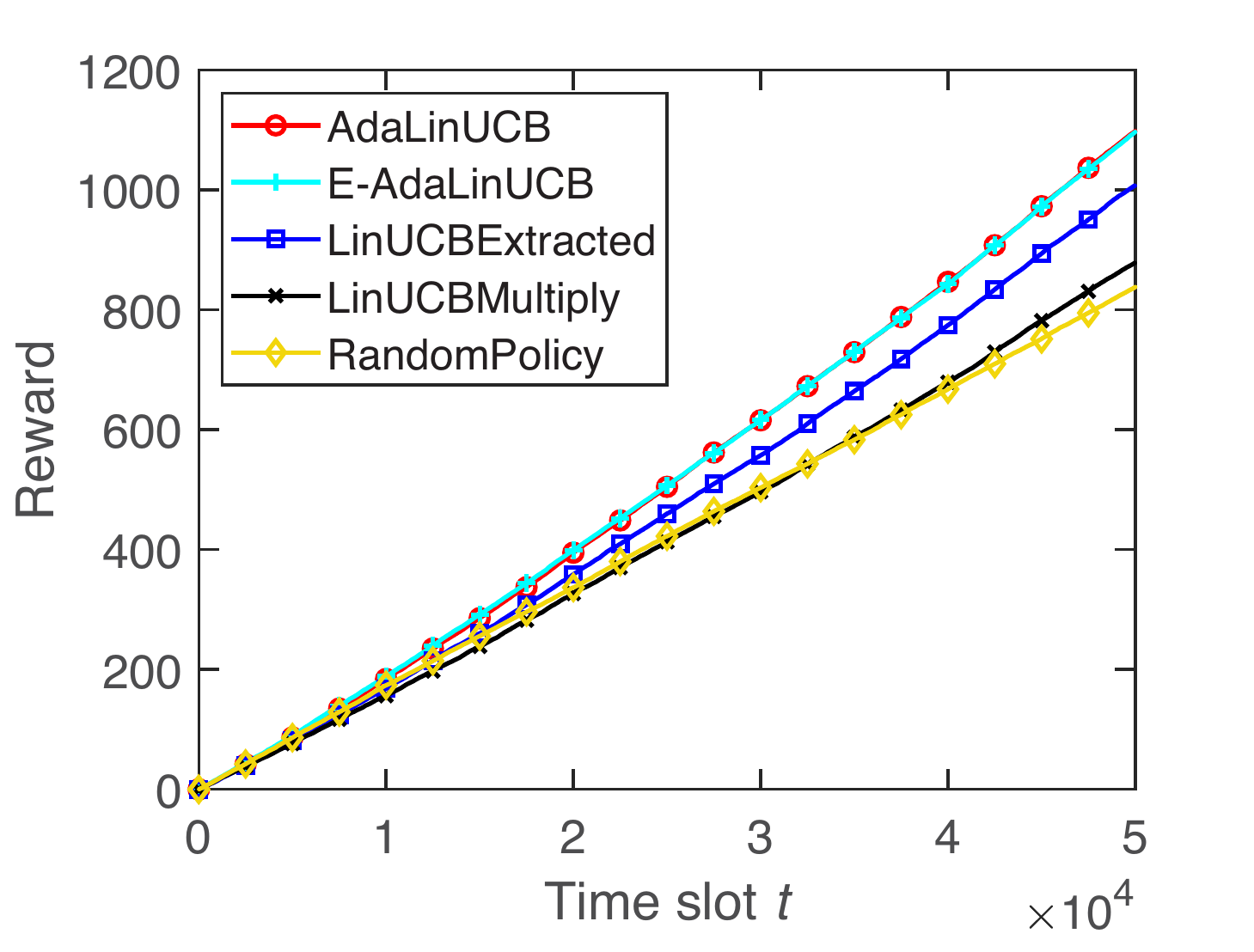}
\label{fig:real_L2U0}}
\quad
\subfigure[$l^{(-)}=l_{0.3}^{(-)},l^{(+)}=l_{0}^{(+)}$]{\includegraphics[angle = 0,height = 0.23\linewidth,width = 0.31\linewidth]{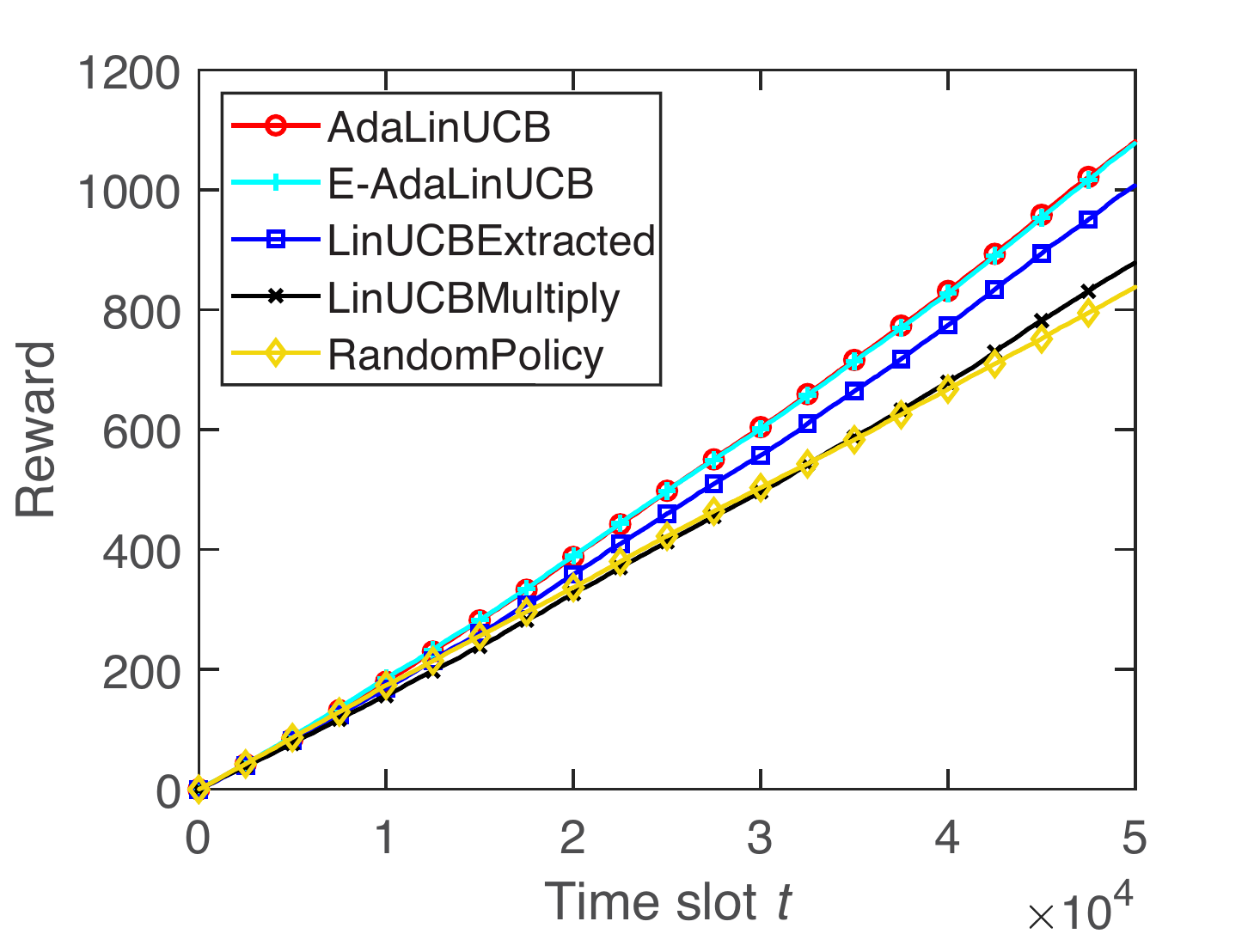}
\label{fig:real_L3U0}}
\quad
\subfigure[AdaLinUCB: $l^{(-)}=l_{0.05}^{(-)},l^{(+)}=l_{0.3}^{(+)}$; AdaLinUCB $(l^{(-)}=l^{(+)})$:$l^{(-)}=l^{(+)}=0.4$ ]{\includegraphics[angle = 0,height = 0.23\linewidth,width =0.31\linewidth]{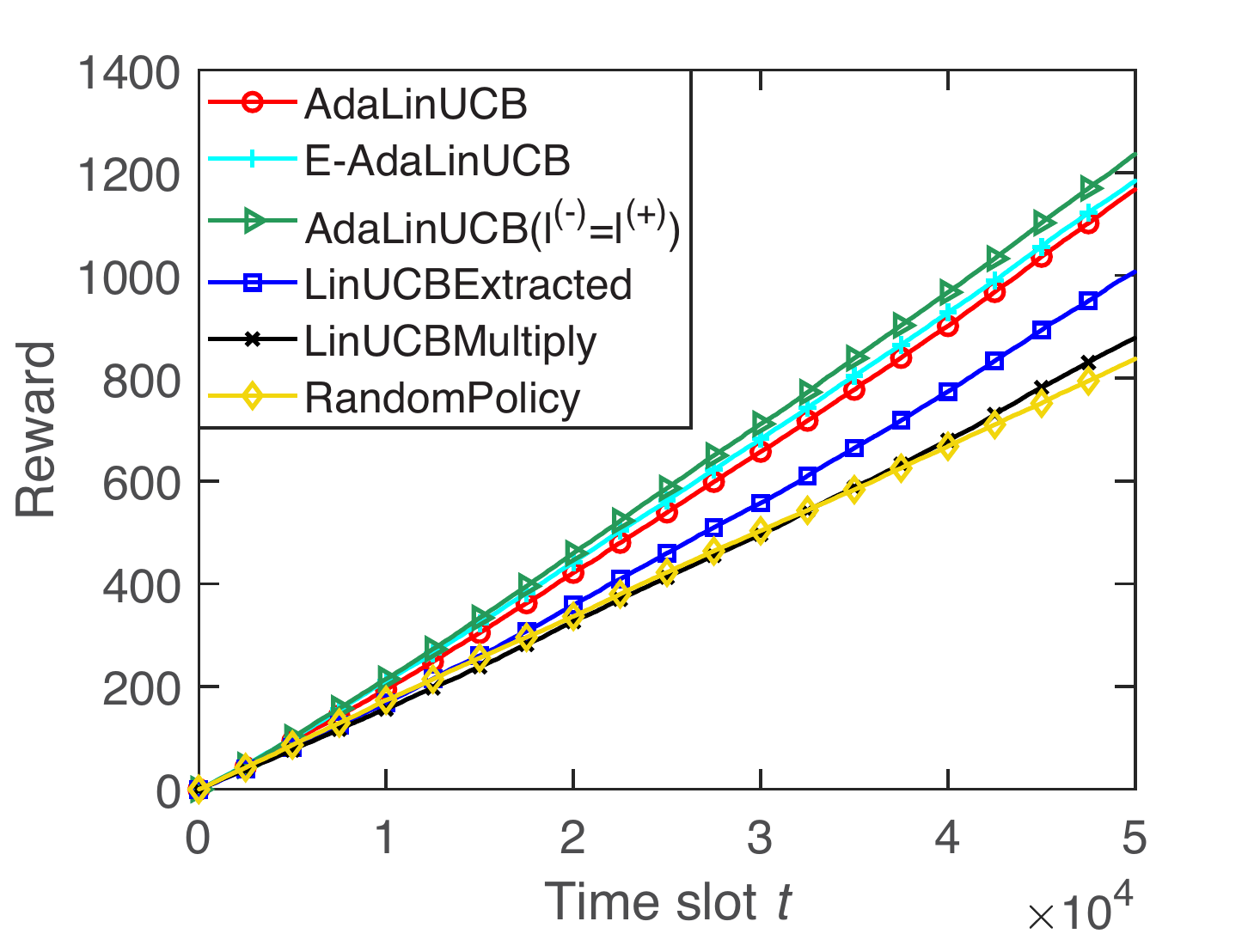}
\label{fig:real_L05U3J4}}
\quad
\subfigure[AdaLinUCB: $l^{(-)}=l_{0}^{(-)},l^{(+)}=l_{0.3}^{(+)}$; AdaLinUCB $(l^{(-)}=l^{(+)})$:$l^{(-)}=l^{(+)}=0.5$ ]{\includegraphics[angle = 0,height = 0.23\linewidth,width = 0.31\linewidth]{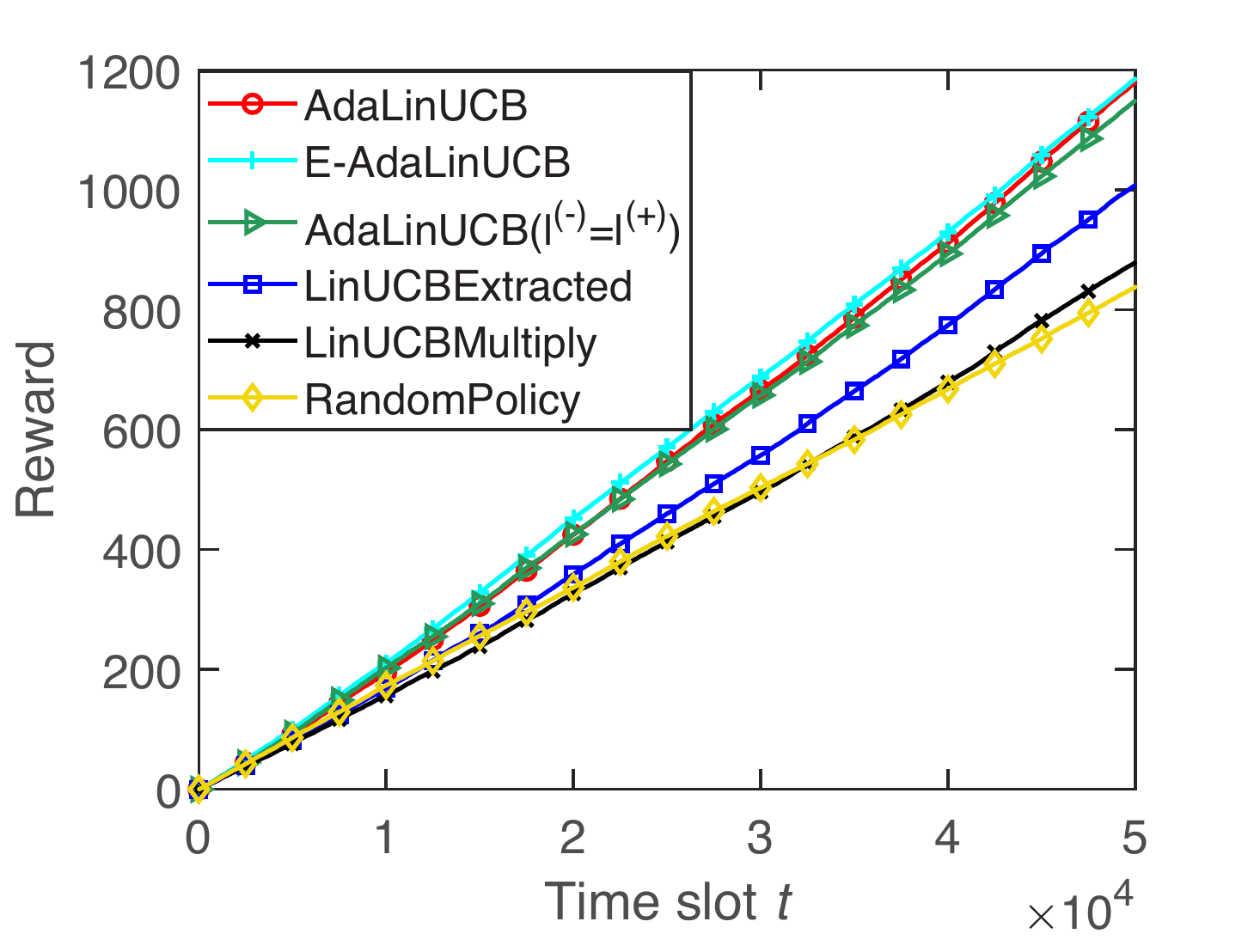}
\label{fig:real_L0U3J5}}
\quad
\subfigure[AdaLinUCB: $l^{(-)}=l_{0.1}^{(-)},l^{(+)}=l_{0.3}^{(+)}$; AdaLinUCB $(l^{(-)}=l^{(+)})$: $l^{(-)}=l^{(+)}=0.6$]{\includegraphics[angle = 0,height = 0.23\linewidth,width = 0.31\linewidth]{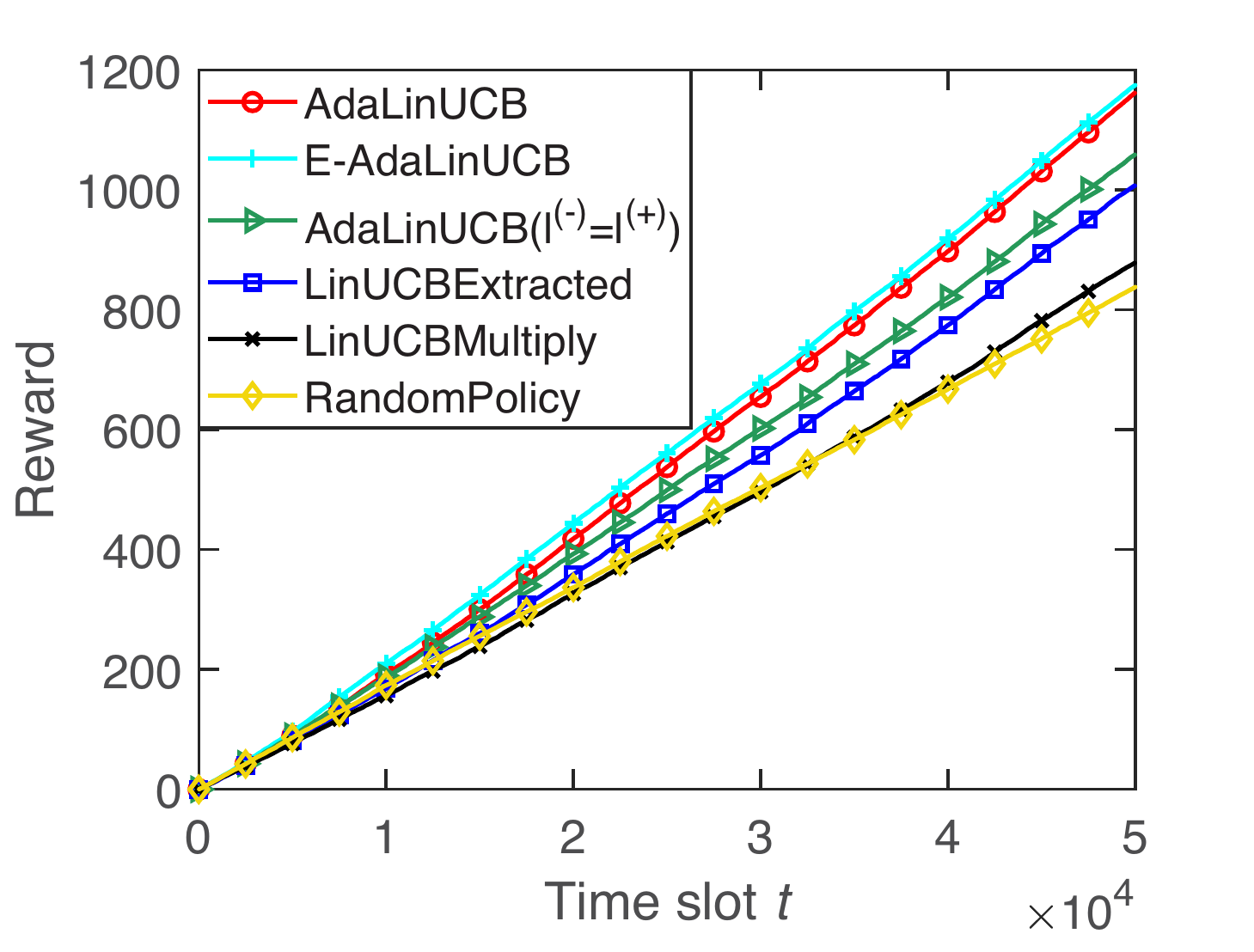}
\label{fig:real_L1U3J6}}



\caption{Performance comparison with different $l^{(-)}$ and $l^{(+)}$ values on Yahoo! Today Module.}
\label{fig:real_load_l}
\end{center}
\end{minipage}
\end{center}
\end{figure*}

For the variation factor, we use a real trace - the sales of a popular store . It includes everyday turnover in two years \cite{Ross}. The normalized variation factor variation is demonstrated in Fig.~\ref{fig:real_load_trace}.

Similarly to the experiments in Fig. \ref{fig:beta_load}, Fig. \ref{fig:beta_load_-} and Fig. \ref{fig:beta_load_+}, we have evaluated the impact of  of $l^{(-)}$ and $l^{(+)}$  in this data set in Fig. \ref{fig:real_load_trace}. We can see that the impact of threshold values for experiments on this real-world dataset is insignificant (when they are changing in a relatively large appropriate range) and the rewards of AdaLinUCB and E-AdaLinUCB are always higher than that of LinUCBExtracted and LinUCBMultiple.

\begin{figure*}[thbp]
\begin{center}
\includegraphics[angle = 0,height = 0.23\linewidth,width = 0.31\linewidth]{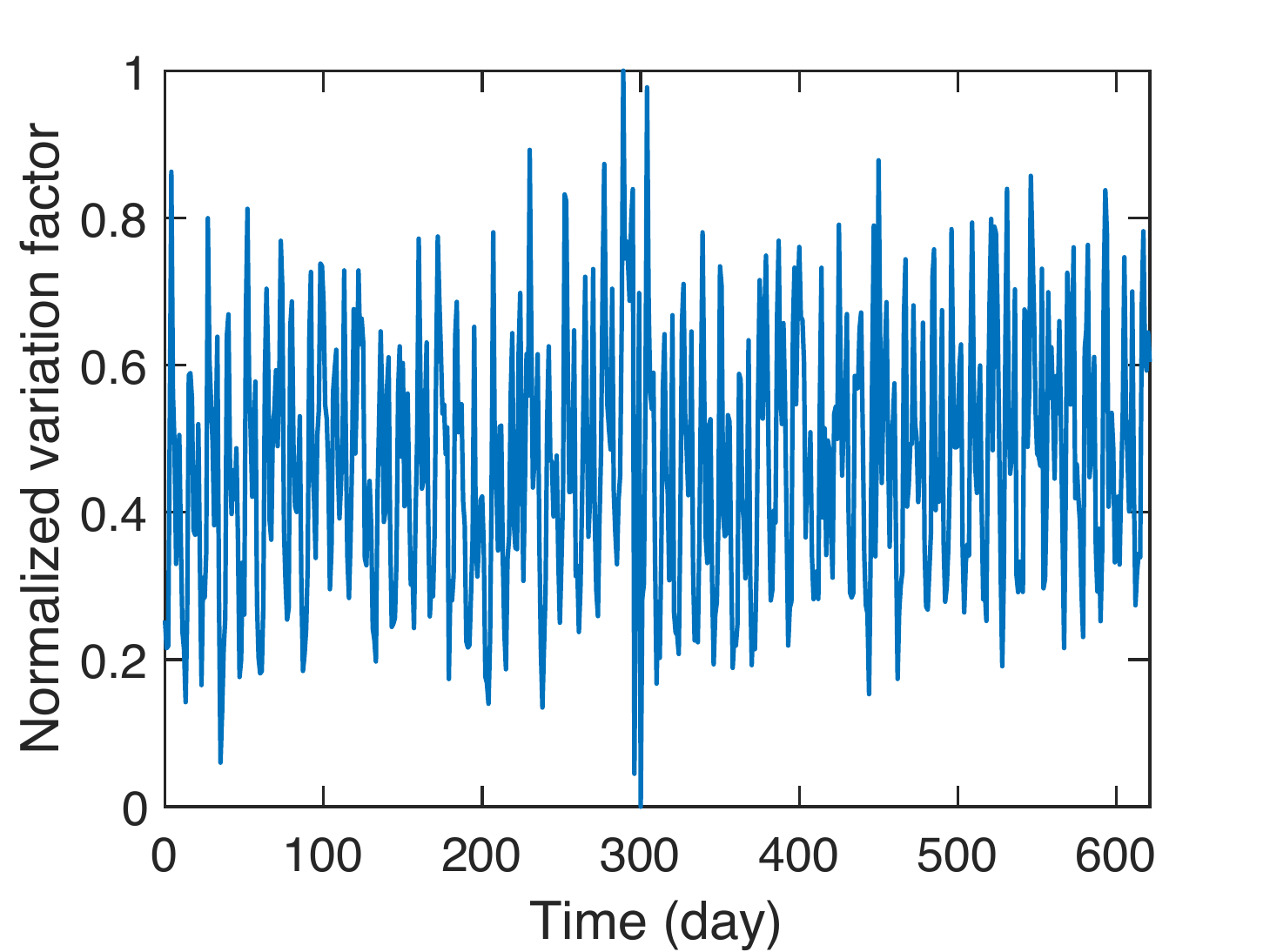}
\vspace{-.25cm}
\caption{Normalized variation factor demonstration.}
\label{fig:real_load_trace}
\end{center}
\end{figure*}

\end{document}